\documentclass{article} 
\usepackage{iclr2023_conference,times}

\usepackage{amsmath,amsfonts,bm}

\def\eqref#1{equation~\ref{#1}}

\def\ceil#1{\lceil #1 \rceil}

\def\1{\bm{1}}

\def\vzero{{\bm{0}}}

\def\va{{\bm{a}}}

\def\vg{{\bm{g}}}
\def\vh{{\bm{h}}}

\def\vv{{\bm{v}}}
\def\vw{{\bm{w}}}
\def\vx{{\bm{x}}}

\def\mE{{\bm{E}}}

\def\mI{{\bm{I}}}

\def\mX{{\bm{X}}}

\DeclareMathAlphabet{\mathsfit}{\encodingdefault}{\sfdefault}{m}{sl}
\SetMathAlphabet{\mathsfit}{bold}{\encodingdefault}{\sfdefault}{bx}{n}

\newcommand{\E}{\mathbb{E}}

\newcommand{\polyln}{\text{polyln}}
\newcommand{\poly}{\text{poly}}
\newcommand{\logit}{\text{logit}}

\usepackage{hyperref}
\usepackage{url}

\usepackage{times}
\usepackage{epsfig}
\usepackage{graphicx,subcaption}
\usepackage{amsmath}
\usepackage{amssymb}
\usepackage{Notations}
\usepackage{mathtools}
\usepackage{amsthm}
\usepackage{booktabs}
\usepackage{bbm}
\usepackage{array}
\usepackage{etoc}
\usepackage[titles]{tocloft}
\usepackage{enumitem}
\setlist[itemize]{leftmargin=5.5mm}
\setlist[enumerate]{leftmargin=5.5mm}

\usepackage{multirow}
\usepackage{wrapfig}

\newtheorem{theorem}{Theorem}[section]
\newtheorem{corollary}{Corollary}[theorem]
\newtheorem{lemma}[theorem]{Lemma}
\newtheorem{prop}{Proposition}

\theoremstyle{remark}
\newtheorem*{remark}{Remark}

\theoremstyle{definition}
\newtheorem{definition}{Definition}[section]

\newcommand{\boxedthm}[1]{%
\vspace{-1ex}
\[\colorbox{gray!10}{\fbox{%
\addtolength{\linewidth}{-2\fboxsep}%
\addtolength{\linewidth}{-2\fboxrule}%
\begin{minipage}{\linewidth}%
#1
\end{minipage}}}%
\]}

\title{Towards Understanding the Effect of Pretraining Label Granularity}

\author{
Guanzhe Hong\thanks{Part of this work was done as a student researcher at Google Research.}~~$^{1}$ \quad Yin Cui\thanks{Work done while at Google.}~~$^{2}$ \quad Ariel Fuxman$^{3}$ \quad Stanley H. Chan$^{1}$ \quad Enming Luo$^{3}$\\
$^1$Purdue University \quad $^2$ NVIDIA \quad $^3$ Google Research \\
{\tt\small \{hong288,stanchan\}@purdue.edu, yinc@nvidia.com,} \\
{\tt\small \{afuxman,enming\}@google.com}
}

\iclrfinalcopy
\begin{document}

\maketitle

\etocdepthtag.toc{mtchapter}
\etocsettagdepth{mtchapter}{subsection}
\etocsettagdepth{mtappendix}{none}

\begin{abstract}
In this paper, we study how the granularity of pretraining labels affects the generalization of deep neural networks in image classification tasks. 
We focus on the ``fine-to-coarse'' transfer learning setting, where the pretraining label space is more fine-grained than that of the target problem. 
Empirically, we show that pretraining on the \textit{leaf} labels of ImageNet21k produces better transfer results on ImageNet1k than pretraining on other coarser granularity levels, \textit{which supports the common practice used in the community}.
Theoretically, we explain the benefit of fine-grained pretraining by proving that, for a data distribution satisfying certain hierarchy conditions, 1) coarse-grained pretraining \textit{only} allows a neural network to learn the ``common'' or ``easy-to-learn'' features well, while 2) fine-grained pretraining helps the network learn the ``rarer'' or ``fine-grained'' features in addition to the common ones, thus improving its accuracy on \textit{hard} downstream test samples in which common features are missing or weak in strength.
Furthermore, we perform comprehensive experiments using the label hierarchies of iNaturalist 2021 and observe that the following conditions, in addition to proper choice of label granularity, enable the transfer to work well \textit{in practice}: 1) the pretraining dataset needs to have a \textit{meaningful label hierarchy}, and 2) the pretraining and target label functions need to \textit{align} well.

\end{abstract}

\section{Introduction}
\label{sec: intro}

Modern deep neural networks (DNNs) are highly effective at image classification. In addition to architectures, regularization techniques, and training methods \cite{samira2018,khan2020,lecun1998,alex2012,ilya2013,kingma2015,kaiming2016,simonyan2015,szegedy2015,xie2017,ioffe2015,srivastava2014,zhang2019}, the availability of large datasets of labeled natural images significantly contributes to the training of powerful DNN-based feature extractors, which can then be used for downstream image classification tasks \cite{deng2009,alex2009,sun2017,zhou2018,horn2018}. 
DNN models, especially the state-of-the-art vision transformers, are well-known to require pre-training on large datasets for effective (downstream) generalization \cite{vit2021,kaiming2016,alex2012}. However, another important dataset property that is often overlooked is the high granularity of the label space: there is limited understanding of why such large pretraining label granularity is necessary, especially when it can be several orders of magnitude greater than the granularity of the target dataset.

Studying the effect of pretraining label granularity in general is a challenging task, due to the wide variety of applications in which transfer learning is used. To make the problem more tractable, we focus on the following setting. First, we adopt the \textit{simplest} possible transfer methodology: pretraining a DNN on an image classification task and then finetuning the model for a target problem using the pretrained backbone. Second, we focus on the ``\textit{fine-to-coarse}'' transfer direction, where the source task has more classes than the target task. \textit{This transfer setting is common in the empirical transfer learning works}, especially for large models \cite{vit2021,radford2021,steiner2022how,Zhang_2021_ICCV,ridnik2021}.

In this setting, we empirically observed an interesting relationship between the pretraining label granularity and DNN generalization, which can be roughly summarized as follows:
\boxedthm{
Under certain basic conditions on the pretraining and target label functions, DNNs pretrained at reasonably high label granularities tend to generalize \textit{better} in downstream classification tasks than those pretrained at low label granularities.
}

\begin{wrapfigure}{rt}{0.41\textwidth}
\vspace{-5mm}
  \begin{center}
    \includegraphics[width=0.4\textwidth]{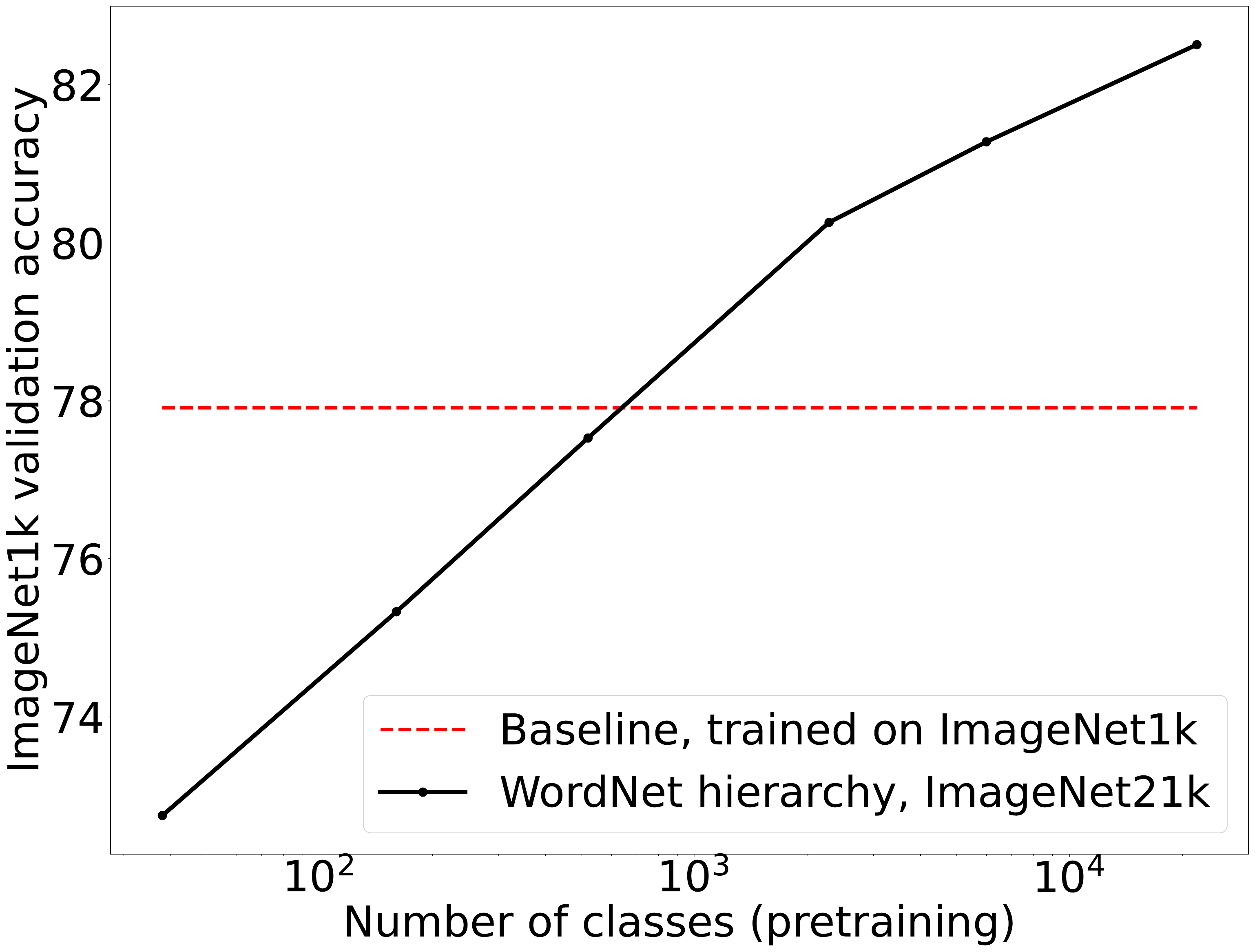}
  \end{center}
  \caption{ImageNet21k$\to$ImageNet1k transfer using a ViT-B/16 model. \textbf{Black line}: pretrained on the WordNet hierarchy of IN21k, \textit{finetuned} on IN1k. \textcolor{red}{\textbf{Red dotted line}}: baseline, trained and evaluated on IN1k. When the pretraining number of classes is $\sim 10^3$, even though IN21k has about 10 times the samples of IN1k, the finetuned network's accuracy is hardly better than the baseline's.}
  \label{fig: im21k->im1k, linear probe}
\vspace{-2ex}
\end{wrapfigure}

Figure \ref{fig: im21k->im1k, linear probe} shows an example of this relationship in the transfer setting of ImageNet21k$\to$ImageNet1k. Further details of this experiment are presented in Section \ref{subsect: im21k->im1k main text details} and the appendix. In this figure, we observe an increase in the network's validation accuracy on the target dataset when the pretraining label granularity increases. 
This indicates that the pretraining label granularity influences the features learnt by the neural network, because after pretraining on ImageNet21k, only the feature extractor is used for finetuning on ImageNet1k.

Closer examination of Figure \ref{fig: im21k->im1k, linear probe} suggests that the features learned by a DNN pretrained at lower granularities already allow it to classify a \textit{significant portion} of the dataset, and finer-grained pretraining brings \textit{incremental} improvements to the network's accuracy. This suggests two things: first, finer-grained pretraining helps the DNN learn \textit{more} features or certain features \textit{better} than pretraining at lower granularities; second, these ``fine-grained features'' potentially help the DNN classify a small set of samples that are too hard for a DNN pretrained at lower granularities. These observations suggest a correspondence between \textit{label hierarchy} and \textit{feature learnability}, similar in spirit to the \textit{multi-view} data property pioneered by \cite{zhu2020_kd}.

Our major findings and contributions are as follows. 
\begin{itemize}[noitemsep]
\item (\textit{Theoretical Analysis}) We provide a \textit{theoretical explanation} for the empirical observation above. Specifically, we prove mathematically that, for a data distribution satisfying certain hierarchy conditions, coarse-grained pretraining \textit{only} allows a neural network to learn the ``common'' or ``easy-to-learn'' features well, while fine-grained pretraining helps the network learn the ``rarer'' or ``fine-grained'' features in addition to the common ones, thus improving its accuracy on \textit{hard} samples where common features are missing or weak in strength.

\item (\textit{Empirical Analysis}) We show that pretraining on the \textit{leaf} labels of ImageNet21k produced better transfer results on ImageNet1k than pretraining on other coarser granularity levels, \textit{which supports the common practice used in the community}. Furthermore, we conducted comprehensive experiments using the label hierarchies of iNaturalist 2021, and observed that the following conditions, in addition to proper choice of label granularity, enable the transfer to work well \textit{in practice}: 1) the pretraining dataset needs to have a \textit{meaningful label hierarchy}, 2) the pretraining and target label functions need to \textit{align} well.
\end{itemize}

\section{Related Work}
\subsection{Theoretical work}
At the theory front, there has been a growing interest in explaining the success of DNNs through the lens of implicit regularization and bias towards ``simpler'' solutions, which can prevent overfitting even when DNNs are highly overparameterized \cite{lyu2021,nakkiran2019,ji2019,palma2019,huh2021}. However, there is also a competing view that DNNs can learn overly simple solutions, known as ``shortcut learning'', which can achieve high training and testing accuracy on in-distribution data but generalize poorly to challenging downstream tasks \cite{geirhos2020, shah2020, pezeshki2021}.
To our knowledge, \cite{shah2020, pezeshki2021} are the closest works to ours, as they both demonstrate that DNNs tend to perform shortcut learning and respond weakly to features that have a weak ``presence'' in the training data. Our work differs from \cite{shah2020, pezeshki2021} in several key ways. 

On the \textit{practical} side, we focus on how the pretraining label space affects classification generalization, while \cite{shah2020, pezeshki2021} primarily focus on demonstrating that simplicity bias can be harmful to generalization. Even though \cite{pezeshki2021} proposed a regularization technique to mitigate this, their experiments are too small-scale. 
On the \textit{theory} side, \cite{pezeshki2021} use the neural tangent kernel (NTK) model, which is unsuitable for analyzing our transfer-learning type problem because the feature extractor of an NTK model barely changes after pretraining. The theoretical setting in \cite{shah2020} is more limited than ours because they use the hinge loss while we use a more standard exponential-tailed cross-entropy loss. Additionally, our data distribution assumptions are more realistic, as they capture the hierarchy in natural images, which has direct impact on the (downstream) generalization power of the pretrained model, according to our results.

Our theoretical analysis is inspired by a recent line of work that analyzes the feature learning dynamics of neural networks. This line of work tracks how the hidden neurons of shallow nonlinear neural networks evolve to solve dictionary-learning-like problems \cite{zhu2022_adv, zhu2020_kd, shen2022}. In particular, our work adopts a \textit{multi-view} approach to the data distribution, which was first proposed in \cite{zhu2020_kd}, while we initialize our network in a similar way to \cite{zhu2022_adv}. However, the learning problems we analyze and the results we aim to show are significantly different from the existing literature. Therefore, we need to derive the gradient descent dynamics of the neural network from scratch.

\vspace{-1.3ex}
\subsection{Experimental work}
There is a growing body of empirical research on how to improve classification accuracy by manipulating the (pre-)training label space. One line of research focuses on using fine-grained labels to improve DNNs' semantic understanding of natural images and their robustness in downstream tasks \cite{mahajan2018, singh2022, yan2020, shnarch2022, juan2020, yang2021, chen2018, ridnik2021, son2023, ngiam2018, cui2018}. For example, \cite{mahajan2018, singh2022} use noisy hashtags from Instagram as pretraining labels, \cite{yan2020, shnarch2022} apply clustering on the data first and then treat the cluster IDs as pretraining labels, \cite{juan2020} use the queries from image search results, \cite{yang2021} apply image transformations such as rotation to augment the label space, and \cite{chen2018,ridnik2021} include fine-grained manual hierarchies in their pretraining processes. 
Our experimental results corroborate the utility of pretraining on fine-grained label space. However, we focus on analyzing the operating regime of this transfer method, specifically how pretraining label granularity, label function alignment, and training set size influence the quality of the model. We show that this transfer method only works well within a specific operating regime.

Another line of research focuses on exploiting the hierarchical structures present in (human-generated) label space to improve classification accuracy at the most fine-grained level \cite{yan2015, zhu2017, goyal2020, sun2017, zelikman2022, silla2011, shkodrani2021, bilal2017, goo2016}. For example, \cite{yan2015} adapt the network architecture to learn super-classes at each hierarchical level, \cite{zhu2017} add hierarchical losses in the hierarchical classification task, \cite{goyal2020} propose a hierarchical curriculum loss for curriculum learning. In contrast, our work focuses on the influence of label granularity on the model's generalization to target tasks with a coarser label space than the pretraining one.

\section{Problem Formulation and Intuition}
\label{sec:problem_formulation}
In this section, we introduce the relevant notations and training methodologies, and discuss our intuition on how label granularity influences DNN feature learning.

\subsection{Notations and methodology}
For a DNN-based classifier, given input image $\mX$, we can write its (pre-logit) output for class $c$ as
\begin{equation}
    F_{c}(\mX) = \left\langle \va_c, \vh(\mTheta; \mX)\right\rangle
\end{equation}
where $\va_c$ is the linear classifier for class $c$, $\vh(\mTheta; \cdot)$ is the network backbone with parameter $\mTheta$. 

In the transfer learning setting, we denote the sets of input samples for the target and source datasets as $\calX^{\text{tgt}}$ and $\calX^{\text{src}}$, and the corresponding sets of labels as $\calY^{\text{tgt}}$ and $\calY^{\text{src}}$, respectively. A dataset can be represented as $\calD = \left(\calX, \calY \right)$. For instance, the source training dataset is $\calD_{\text{train}}^{\text{src}} = \left(\calX^{\text{src}}_{\text{train}}, \calY^{\text{src}}_{\text{train}} \right)$. The relevant training and testing datasets are denoted as $\calD_{\text{train}}^{\text{src}}, \calD_{\text{train}}^{\text{tgt}}, \calD_{\text{test}}^{\text{tgt}}$. Finally, the granularity of a label set is denoted as $\calG(\calY)$, which represents the total number of classes.

Our transfer learning methodology is as follows. We first pretrain a neural network $F_c(\cdot) = \left\langle \va_c, \vh(\mTheta; \cdot)\right\rangle$ from random initialization using $\calD_{\text{train}}^{\text{src}}$ (typically with early stopping). This gives us the pretrained feature extractor $\vh(\mTheta_{\text{train}}^{\text{src}}; \cdot) $. We then either linear probe or finetune it using $\calD_{\text{train}}^{\text{tgt}}$, and evaluate it on $\calD_{\text{test}}^{\text{tgt}}$. In contrast, the \textit{baseline} is simply trained on $\calD_{\text{train}}^{\text{tgt}}$ and evaluated on $\calD_{\text{test}}^{\text{tgt}}$.

\subsection{How does the pretraining label granularity influence feature learning?}
To make our discussion in this subsection more concrete, let us consider an example. Suppose we have a set of pretraining images that consist of cats and dogs, $\calD^{\text{src}}_{\text{train}} = \calD^{\text{cat}} \cup \calD^{\text{dog}}$. The target problem $(\calD^{\text{tgt}}_{\text{train}}, \calD^{\text{tgt}}_{\text{test}})$ requires the DNN to identify whether the animal in the image is a cat or a dog.

Intuitively, we can say that a group of images belongs to a class because they \textit{share} certain visual ``features'' that are absent, or weak in all other classes. At the coarsest cat-versus-dog hierarchy level, \textit{common} cat features distinguish cat images from dog ones; these features are also the \textit{most noticeable} because they appear most frequently in the dataset. However, humans can also define fine-grained classes of cats and dogs based on their breeds. This means that \textit{each subclass has its own unique visual features that are only dominant within that subclass}. This leads to an interesting observation: fine-grained features may be rarer in the dataset, making them more difficult to notice. We illustrate this observation in Figure \ref{fig:cat vs dog illustration}.

\begin{figure}[h]
    \centering
    \includegraphics[width=0.6\linewidth]{ 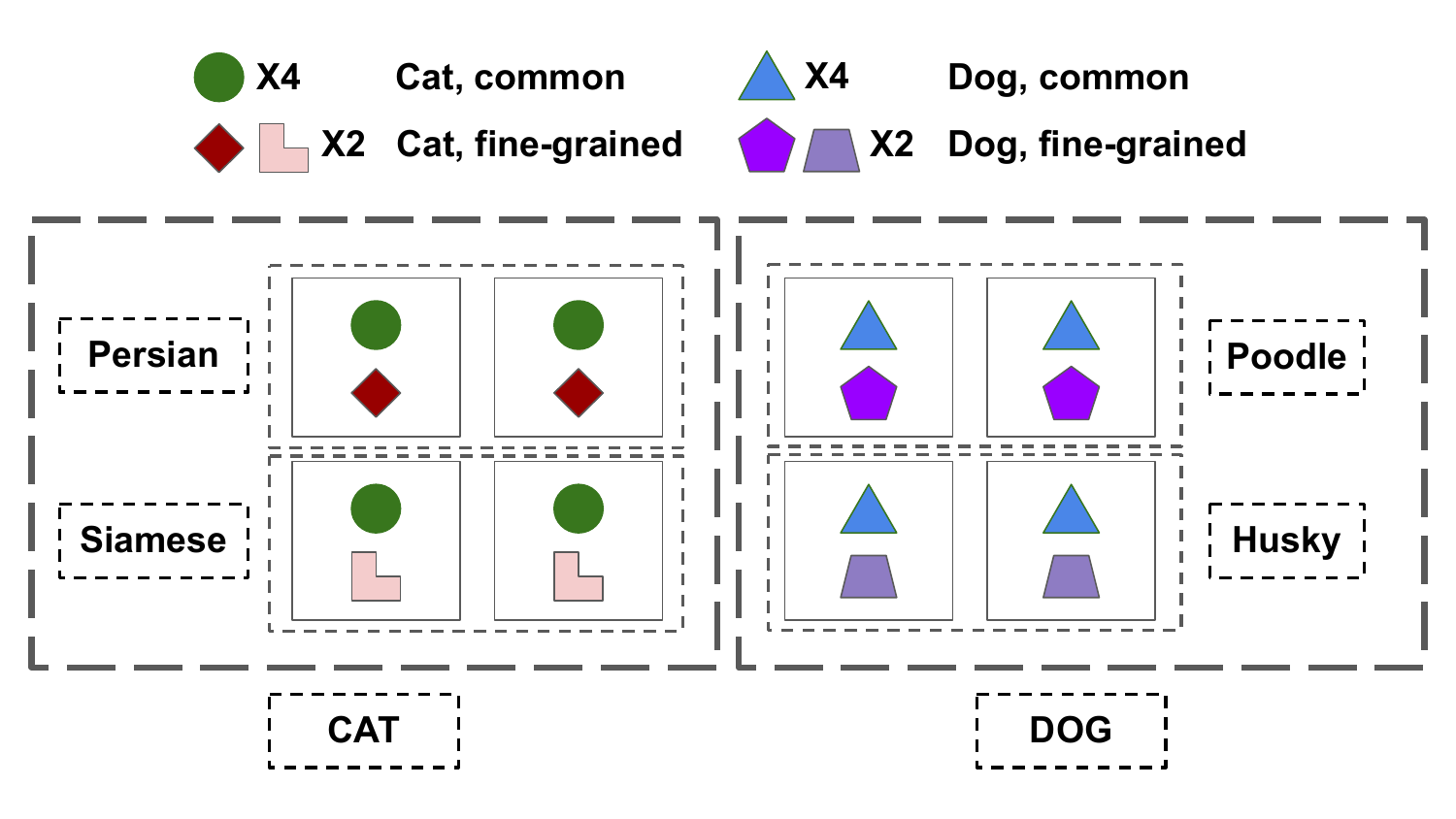}
    \caption{A simplified symbolic representation of the cat versus dog problem. The common features (green disk and blue triangle) appear more frequently, thus are more noticeable than the fine-grained features in the dataset. Furthermore, the learner can approximate the solution to the ``cat or dog'' problem by simply learning the common features.}
    \label{fig:cat vs dog illustration}
\end{figure}

Now suppose the pretraining label assignment is the binary ``cat versus dog'' task. An intelligent learner can take the arduous route of learning both the common and hard-to-notice fine-grained features in the two classes.
However, the learner can also take ``shortcuts'' by learning only the common features in each class to achieve low training loss. This latter form of learning can harm the network's generalization in downstream classification tasks because the network can be easily misled by the distracting irrelevant patterns in the image when the common features are weak in signal strength. 
One strategy to force the learner to learn the rarer features well is to explicitly label the fine-grained classes. This means that within each fine-grained class, the fine-grained features become as easy to notice as the common features. This forces the network to learn the fine-grained features to solve the classification problem.

\section{Theory of Label Granularity}
\label{sec:theory}

\begin{wrapfigure}{r}{0.35\textwidth}
    \vspace{-12mm}
  \begin{center}
    \includegraphics[width=0.33\textwidth]{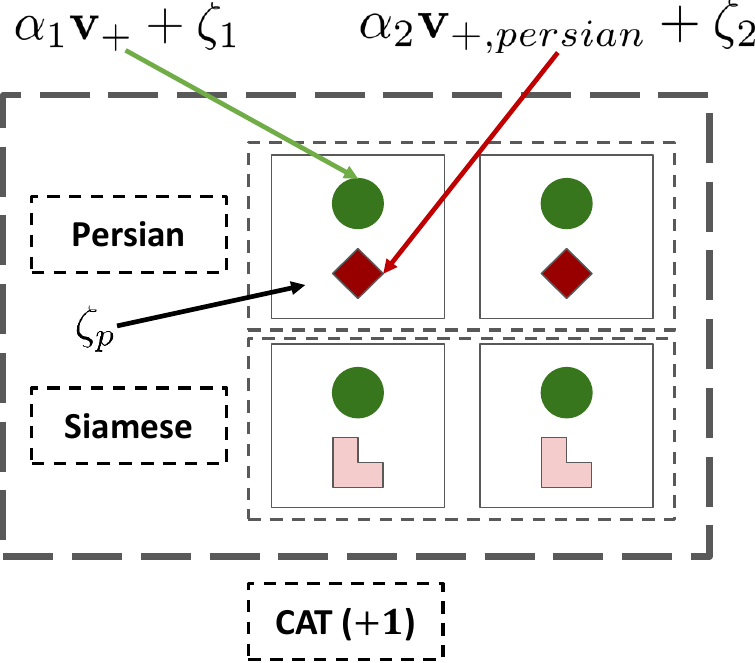}
  \end{center}
  \caption{A simplified illustration of how our intuition in Section \ref{sec:problem_formulation} translates to the definition of an easy sample.}
  \label{fig: sample illustration}
  \vspace{-5mm}
\end{wrapfigure}

\subsection{Source training data distribution}

We assume that an input sample in $\calD^{\text{src}}_{\text{train}}$ consists of $P$ disjoint patches of dimension $d$, in symbols, an input sample $\mX = (\vx_1, \vx_2, ..., \vx_P)$ with $\vx_p\in\mathbb{R}^d$. We consider the setting where $d$ is sufficiently large, and all our asymptotic statements are made with respect to $d$.

Building on the discussion in Section \ref{sec:problem_formulation}, we base our data distribution assumptions for the source dataset on our intuition about the hierarchy of common-versus-fine-grained features in natural images. For simplicity, we consider only two levels of label hierarchy. The root of this hierarchy has two superclasses $+1$ and $-1$. The superclass $+1$ has $k_+$ subclasses, with notation $(+,c)$ for $c\in[k_+]$. The same definition holds for the ``$-$'' classes. As for the input features, since the common and fine-grained features need to be sufficiently ``different'', we push this intuition to an extreme and assume all the ``features'' (in their purest form) have zero correlation and equal magnitude. This leads to the following definition.

\begin{definition}[Features and hierarchy]
\label{def: main text, feat hierarchy}
    We define \textbf{features} as elements of a fixed orthonormal dictionary $\calV = \{\vv_i\}_{i=1}^d \subset\mathbb{R}^d$. Furthermore, we call $\vv_+\in\calV$ the \textit{common feature} unique to samples in class $+1$, and $\vv_{+,c}\in\calV$ to be the \textit{fine-grained feature} unique to samples of subclass $(+,c)$.\footnote{We only consider one feature for each (sub-)class for notational simplicity; our work extends to the multi-feature case easily.}
\end{definition}

\begin{definition}[Sample generation]
\label{def: main text, sample gen}
    For an \textbf{easy sample} $\mX$ belonging to the $(+,c)$ subclass (for $c\in[k_+]$), sample its patches as follows:
    \vspace{-2mm}
    \begin{enumerate}[noitemsep]
        \item (Common-feature patches) Approximately $s^*$ patches are \textit{common-feature patches}, defined as $\vx_{p} = \alpha_{p}\vv_{+} + \vzeta_{p}$, for some (random) $\alpha_p \approx 1$, $\vzeta_{p}\sim\calN(\vzero, \sigma_{\zeta}^2\mI_d)$;
        \item (Subclass-feature patches) Approximately $s^*$ patches are \textit{subclass-feature patch}, defined as $\vx_{p} = \alpha_{p}\vv_{+,c} + \vzeta_{p}$, for some (random) $\alpha_p \approx 1$, $\vzeta_{p}\sim\calN(\vzero, \sigma_{\zeta}^2\mI_d)$;
        \item (Noise patches) For the remaining non-feature patches, $\vx_p = \vzeta_p$, where $\vzeta_{p}\sim\calN(\vzero, \sigma_{\zeta}^2\mI_d)$.
    \vspace{-2mm}
    \end{enumerate}

    A \textbf{hard sample} is generated in the same way as easy samples, except the common-feature patches are replaced by noise patches, and we replace approximately $s^{\dagger}$ number of noise patches by ``feature-noise'' patches, which are of the form $\vx_{p} = \alpha_{p}^{\dagger}\vv_{-} + \vzeta_{p}$, where $\alpha_{p}^{\dagger}\in o(1)$, and set one of the noise patches to $\vzeta^* \sim \calN(\vzero, \sigma_{\zeta^*}^2\mI_d)$ with $\sigma_{\zeta^*}\gg \sigma_{\zeta}$; these patches serve the role of ``distracting patterns'' discussed in Section \ref{sec:problem_formulation}.

    Samples belonging to the superclass $+1$ are the \textit{union} of the samples of each subclass $(+,c)$. See Figure \ref{fig: sample illustration} for an intuitive illustration of the easy samples.
\end{definition}

With the feature-based input sample generation process in place, we can define the source dataset's label function, and finally the source training set.

\begin{definition}[Source dataset's label mapping]
\label{def: main test, label map}
    A sample $\mX$ belongs to the $+1$ superclass if any one of its common- or subclass-feature patches contains $\vv_+$ or $\vv_{+,c}$ for any $c\in[k_+]$. It belongs to the $(+,c)$ subclass if any one of its subclass-feature patches contains $\vv_{+,c}$.
\end{definition}

\begin{definition}[Source training set]

    We assume the input samples of the source training set as $\calX_{\text{train}}^{\text{src}}$ are generated as in Definition \ref{def: main text, sample gen}; the corresponding labels are generated following Definition \ref{def: main test, label map}. Overall, we denote the dataset $\calD_{\text{train}}^{\text{src}}$.

\end{definition}

All of the above definitions also hold for the ``$-$'' classes. To see the full problem setup and parameter choices, please refer to Appendix \ref{section: theory, problem setup}.

\subsection{Target data distribution assumptions}
For simplicity and to ensure that baseline and fine-grained training have no unfair advantage over each other, we make the following assumptions in our theoretical setting: first, the input samples in the source and target datasets are all generated according to Definition \ref{def: main text, sample gen}; second, the true label function remains the same across the two datasets; third, since we are studying the ``fine-to-coarse'' transfer direction, the target problem's label space is the \textit{root} of the hierarchy, meaning that any element of $\calY^{\text{tgt}}_{\text{train}}$ or $\calY^{\text{tgt}}_{\text{test}}$ must belong to the label space $\{+1, -1\}$. Therefore, in our setting, only $\calY^{\text{src}}$ and $\calY^{\text{tgt}}$ can differ (in distribution) due to different choices in the label hierarchy level.
This analysis-oriented setting mirrors the transfer settings in Section \ref{sec:experiments}.\footnote{In this idealized setting, we have essentially made \textit{baseline training} and \textit{coarse-grained pretraining} the \textit{same} procedure. Therefore, an equally valid way to view our theory's setting is to consider $\calD_{\text{train}}^{\text{tgt}}$ the same as $\calD_{\text{train}}^{\text{src}}$ except with coarse-grained labels. In other words, we pretrain the network on two versions of the source dataset $\calD_{\text{train}}^{\text{src,coarse}}$ and $\calD_{\text{train}}^{\text{src,fine}}$, and then compare the two models on $\calD_{\text{test}}^{\text{tgt}}$ (which has coarse-grained labels).}

\subsection{Learner assumptions}
We assume that the learner is a two-layer average-pooling convolutional ReLU network:
\begin{equation}
    F_{c}(\mX) = \sum_{r=1}^m a_{c,r}\sum_{p=1}^P \sigma(\langle \vw_{c,r}, \vx_p\rangle + b_{c,r}),
\end{equation}
where $m$ is a low-degree polynomial in $d$ and denotes the width of the network, $\sigma(\cdot) = \max(0, \cdot)$ is the ReLU nonlinearity, and $c$ denotes the class. We perform an initialization of $\vw_{c,r}^{(0)} \sim \calN(\vzero, \sigma_0^2\mI_d)$ with $\sigma_0^2 = 1/\poly(d)$; we set $b_{c,r}^{(0)} = -\Theta\left(\sigma_0\sqrt{\ln(d)}\right)$ and manually tune it, similar to \cite{zhu2022_adv}. Cross-entropy is the training loss for both baseline and transfer training. To simplify analysis and to focus solely on the learning of the feature extractor, we freeze $a_{c,r} = 1$ during all baseline and transfer training phases, and we use the fine-grained model for binary classification as follows: $\widehat{F}_+(\mX) = \max_{c\in[k_+]}F_{+,c}(\mX), \, \widehat{F}_-(\mX) = \max_{c\in[k_-]}F_{-,c}(\mX)$.\footnote{See Appendix \ref{section: theory, problem setup} and the beginning of Appendix \ref{section: appendix, finegrained} for details of learner assumptions.}

\subsection{Main result: stochastic gradient descent on easy training samples}

In this subsection, we study how well the neural network generalizes if it is trained via stochastic gradient descent on \textit{easy samples only}. In particular, we allow training to run for $\poly(d)$ time.

\begin{theorem} [Coarse-label training: baseline]
\label{thm: main text, coarse label, baseline}
(Summary). Suppose $\calD_{\text{train}}^{\text{tgt}}$ consists only of easy samples, and we perform stochastic gradient descent training: at every iteration $t$, there is a fresh set of $N$ iid samples $\left(\mX_n^{(t)}, y_n^{(t)}\right)_n$. Moreover, assume the number of fine-grained classes $k_+ = k_- \in [\polyln(d), f(\sigma_{\zeta})]$, for some function $f$ of the noise standard deviation $\sigma_{\zeta}$.\footnote{$ f(\sigma_{\zeta}) = d^{0.4}$ is an example choice.}

With high probability, with proper choice of step size $\eta$, there exists $T^* \in \poly(d)$ such that for any $T \in [T^*, \poly(d)]$, the training loss satisfies $\calL(F^{(T)}) \le o(1)$, and for an easy test sample $(\mX_{\text{easy}},y)$, $\mathbb{P}\left[F_y^{(T)}(\mX_{\text{easy}}) \le F_{y'}^{(T)}(\mX_{\text{easy}})\right] \le o(1)$ for $y'\in \{+1, -1\} - \{y\}$. However, for all $t\in[0,\poly(d)]$, given a hard test sample $(\mX_{\text{hard}},y)$, $\mathbb{P}\left[F_y^{(t)}(\mX_{\text{hard}}) \le F_{y'}^{(t)}(\mX_{\text{hard}})\right] \ge \Omega(1)$.
\end{theorem}

To see the full version of this theorem, please see Appendix \ref{section: appendix, phase II coarse}; its proofs spans Appendix \ref{section: appendix init geometry} to \ref{section: appendix, phase II coarse}. This theorem essentially says that, with a mild lower bound on the number of fine-grained classes, if we only train on the \textit{easy} samples with \textit{coarse} labels, it is virtually impossible for the network to learn the fine-grained features even if we give it as much practically reachable amount of time and training samples as possible. Consequently, the network would perform poorly on any challenging downstream test image: if the image is missing the \textit{common} features, then the network can be easily misled by irrelevant features or other potential flaws in the image.

\begin{theorem}[Fine-grained-label training]
(Summary). Assume the same setting in Theorem \ref{thm: main text, coarse label, baseline}, except that we let the labels be fine-grained: train the network on a total of $k_+ + k_-$ subclasses instead of 2. With high probability, within $\poly(d)$ time, the trained network satisfies $\mathbb{P}\left[\widehat{F}_y^{(T)}(\mX) \le \widehat{F}_{y'}^{(T)}(\mX)\right] \le o(1)$ for $y'\in\{+1,-1\}-\{y\}$ on the target binary problem on  easy and hard test samples.\footnote{Finetuning $\widehat{F}$ can further boost the feature extractor's response to the true features. See Appendix \ref{sec: appendix, finegrained trainining, end error}.}
\end{theorem}

The full version of this result is presented in Appendix \ref{sec: appendix, finegrained trainining, end error}, and its proof in Appendix \ref{section: appendix, finegrained}. After fine-grained pretraining, the network's feature extractor gains a strong response to the fine-grained features, therefore its accuracy on the downstream hard test samples increases significantly.

One concern about the above theorems is that the neural networks are trained only on easy samples. As noted in Section \ref{sec: intro}, samples that can be classified correctly by training only with coarse-grained labels, or \textit{easy} samples, should make up the \textit{majority} of the training and testing samples, and pretraining at higher label granularities only possibly improves network performance on \textit{rare} hard examples. Our theoretical result is intended to present the ``feature-learning bias'' of a neural network in an exaggerated fashion. Therefore, it is natural to start with the case of ``no hard training examples at all''. In reality, even if a small portion of hard training samples is present, finite-sized training datasets can have many flaws that can cause the network to overfit severely before learning the fine-grained features, especially since rarer features are learnt more slowly and corrupted by greater amount of noise. We leave these deeper considerations for future theoretical work.\footnote{Allowing a small portion of hard training samples essentially yields a ``perturbed'' version of our results in this paper, which we leave for future work.}


\section{Empirical Results}
\label{sec:experiments}
Building on our theoretical analysis in an idealized setting, this section discusses conditions on the source and target label functions that we observed to be important for fine-grained pretraining to work \textit{in practice}. We present the core experimental results obtained on ImageNet21k and iNaturalist 2021 in the main text, and leave the experimental details and ablation studies to Appendix \ref{section: appendix A}.

\subsection{ImageNet21k$\to$ImageNet1k transfer experiment}
\label{subsect: im21k->im1k main text details}

\begin{wraptable}{r}{7cm}
\vspace{-2ex}
\centering
\scalebox{0.95}{
\begin{tabular}{l c c c}
\toprule
Pretrain on & Hier. lv & $\calG(\mathcal{Y}^{\text{src}})$ & Valid. acc.\\
\hline
IM1k  & - & 1000 & \textcolor{red}{\textbf{77.91}} \\
\cmidrule{1-4}
IM21k & 0 (leaf) & 21843 & \textbf{82.51} \\
\cmidrule{2-4}
& 1 & 5995 & 81.28\\
\cmidrule{2-4}
& 2 & 2281 & 80.26 \\
\cmidrule{2-4}
& 4 & 519 & 77.53 \\
\cmidrule{2-4}
& 6 & 160 & 75.53 \\
\cmidrule{2-4}
& 9 & 38 & 72.75 \\
\bottomrule
\end{tabular}
}
\caption{\textbf{Cross-dataset transfer}. ViT-B/16 average \textit{finetuning} validation accuracy on ImageNet1k, pretrained on ImageNet21k. The baseline (in red) is taken directly from \cite{vit2021}. See Appendix \ref{section: appendix A} for details.}
\label{table: vitb16, im21k to 1k, finetune}
\vspace{-4ex}
\end{wraptable}

This subsection provides more details about the experiment shown in Figure \ref{fig: im21k->im1k, linear probe}. Specifically, we show that the common practice of pretraining on ImageNet21k using leaf labels is indeed better than pretraining at lower granularities in the manual hierarchy.

\textit{Hierarchy definition}. The label hierarchy in ImageNet21k is based on WordNet \cite{wordnet, deng2009}. To define fine-grained labels, we first define the leaf labels of the dataset as Hierarchy level 0. For each image, we trace the path from the leaf label to the root using the WordNet hierarchy. We then set the $k$-th synset (or the root synset, if it is higher in the hierarchy) as the level-$k$ label of this image. This procedure also applies to the multi-label samples. This is how we generate the hierarchies shown in Table \ref{table: vitb16, im21k to 1k, finetune}.

\textit{Network choice and training}. For this dataset, we use the more recent Vision Transformer ViT-B/16 \cite{vit2021}. Our pretraining pipeline is almost identical to the one in \cite{vit2021}. For fine-tuning, we experimented with several strategies and report only the best results in the main text. To ensure a fair comparison, we also used these strategies to find the best baseline result by using $\calD^{\text{tgt}}_{\text{train}}$ for pretraining.

\textit{Results}. Table \ref{table: vitb16, im21k to 1k, finetune} shows a clear trend: the best accuracy occurs at the leaf level, and the network's accuracy on ImageNet1k decreases as the pretraining label granularity on ImageNet21k decreases. 

Interestingly, as the pretraining granularity approaches 1,000, the finetuned accuracies become relatively close to the baselines. This suggests that, even if there is little source-target distribution shift in the inputs, the label functions align well across the datasets and we pretrain with much more data than the baseline, we cannot see as much improvement with a poorly chosen pretraining granularity.

\subsection{Transfer experiment on iNaturalist 2021}
We conduct a systematic study of the transfer method \textit{within} the label hierarchies of iNaturalist 2021 \cite{inaturalist_2021}. This dataset is well-suited for our analysis because it has a manually defined label hierarchy that is based on the biological traits of the creatures in the images. Additionally, the large sample size of this dataset reduces the likelihood of sample-starved pretraining on reasonably fine-grained hierarchy levels.

\begin{figure}[t!]
    \centering
    \includegraphics[width=0.7\linewidth]{ 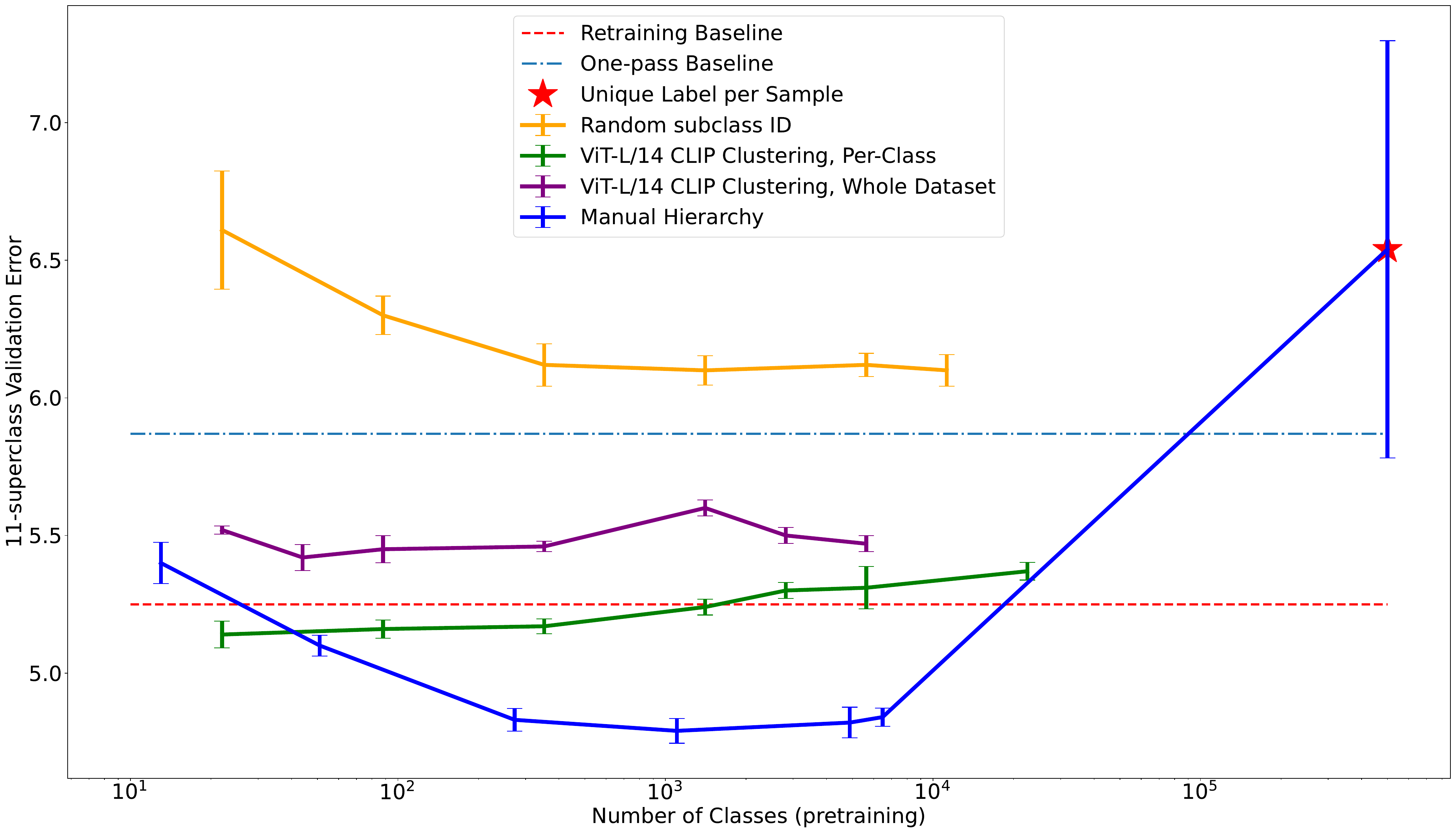}
    \caption{\textbf{In-dataset transfer}. ResNet34 validation error (with standard deviation) of finetuning on 11 superclasses of iNaturalist 2021, pretrained on various label hierarchies. The manual hierarchy outperforms the baseline and every other hierarchy, and exhibits a U-shaped curve.}
    \label{fig:resnet34 inat2021}
\end{figure}

Our experiments on this dataset demonstrate the importance of a meaningful label hierarchy, label function alignment and appropriate choice of pretraining label granularity.

\textit{Relevant datasets}. We perform transfer experiments within iNaturalist2021. More specifically, we set $\calX^{\text{src}}_{\text{train}}$ and $\calX^{\text{tgt}}_{\text{train}}$ both equal to the training split of the input samples in iNaturalist2021, and set $\calX^{\text{tgt}}_{\text{train}}$ to the testing split of the input samples in iNaturalist2021. To focus on the ``fine-to-coarse'' transfer setting, the \textit{target problem} is to classify the root level of the manual hierarchy, which contains 11 superclasses. To generate a greater gap between the performance of different hierarchies and to shorten training time, we use the mini version of the training set in all our experiments.

The \textit{motivation} behind this experimental setting is similar to how we defined our theory's data setting:  by using exactly the same input samples for the baseline and fine-grained training schemes, we can be more confident that the choice of (pre-)training label hierarchy is the \textit{cause} of any changes in the network's accuracy on the target problem. 

\textit{Alternative hierarchies generation}. To better understand the transfer method's operating regime, we experiment with different ways of generating the fine-grained labels for pretraining: we perform kMeans clustering on the ViT-L/14-based CLIP embedding \cite{radford2021, dehghani2021scenic} of every sample in the training set and use the cluster IDs as pretraining class labels. We carry out this experiment in two ways. The green curve in Figure \ref{fig:resnet34 inat2021} comes from performing kMeans clustering on the embedding of each superclass \textit{separately}, while the purple one's cluster IDs are from performing kMeans on the \textit{whole dataset}. The former way preserves the implicit hierarchy of the superclasses in the cluster IDs: samples from superclass $k$ cannot possibly share a cluster ID with samples belonging to superclass $k' \neq k$. Therefore, its label function is forced to align better with that of the 11 superclasses than the purple curve's. We also assign random class IDs to samples.

\textit{Network choice and training}. We experiment with ResNet 34 and 50 on this dataset. For pretraining on $\calD_{\text{train}}^{\text{src}}$ with fine-grained labels, we adopt a standard 90-epoch large-batch-size training procedure commonly used on ImageNet \cite{kaiming2016,goyal2017}. Then we finetune the network for 90 epochs and test it on the 11-superclass $\calD^{\text{tgt}}_{\text{train}}$ and $\calD^{\text{tgt}}_{\text{test}}$, respectively, using the pretrained backbone $\vh(\mTheta_{\text{src}}; \cdot)$: we found that finetuning at a lower batch size and learning rate improved training stability. To ensure a fair comparison, we trained the baseline model using exactly the same training pipeline, except that the pretraining stage uses $\calD^{\text{tgt}}_{\text{train}}$. We observed that this ``retraining'' baseline consistently outperformed the naive one-pass 90-epoch training baseline on this dataset. Due to space limitations, we leave the results of ResNet50 to the appendix.

\textit{Interpretation of results}. Figure \ref{fig:resnet34 inat2021} shows the validation errors of the resulting models on the 11-superclass problem. We make the following observations.
\begin{itemize}
    \item (Random class ID). Random class ID pretraining (orange curve) performs the worst of all the alternatives. The label function of this type does not generate a \textit{meaningful hierarchy} because it has no consistency in the features it considers discriminative when decomposing the superclasses. This is in stark contrast to the manual hierarchies, which decompose the superclasses based on the finer biological traits of the creatures in the image.
    \item (Human labels). Even with high-quality (human) labels, \textit{the granularity of pretraining labels should not be too large or too small}. As shown by the blue curve in Figure \ref{fig:resnet34 inat2021},  models trained on the manual hierarchies outperform all other alternative as long as the pretraining label granularity is beyond the order of $10^2$. However, the error exhibits a U shape, meaning that as the label granularity becomes too large, the error starts to rise. This is intuitive. If the pretraining granularity is too close to the target one, we should not expect improvement. On the other extreme, if we assign a unique label to \textit{every} sample in the training data, it is highly likely that the only \textit{differences} a model can find between each class would be frivolous details of the images, which would not be considered discriminative by the label function of the target coarse-label problem. In this case, the pretraining stage is almost meaningless and can be misleading, as evidenced by the very high label-per-sample error (red star in Figure \ref{fig:resnet34 inat2021}).
    \item (Cluster ID). For fine-grained pretraining to be effective, the features that the pretraining label function considers discriminative must \textit{align} well with those valued by the label function of the 11-superclass hierarchy. To see this point, observe that for models trained on cluster IDs obtained by performing kMeans on the CLIP embedding samples in each superclass \textit{separately} (green curve in Figure \ref{fig:resnet34 inat2021}), their validation errors are much lower than those trained on cluster IDs obtained by performing kMeans on the whole dataset (purple curve in Figure \ref{fig:resnet34 inat2021}). As expected, the manually defined fine-grained label functions align best with that of the 11 superclasses, and the results corroborate this view.
\end{itemize}

\section{Conclusion}

In this paper, we studied the influence of pretraining label granularity on the generalization of DNNs in downstream image classification tasks. 
Empirically, we confirmed that pretraining on the \textit{leaf} labels of ImageNet21k produces better transfer results on ImageNet1k than pretraining on coarser granularity levels; we further showed the importance of meaningful label hierarchy and label function alignment between the source and target tasks for fine-grained training in practice.
Theoretically, we explained why fine-grained pretraining can outperform the vanilla coarse-grained training by establishing a correspondence between label granularity and solution complexity. In the future, we plan to investigate the transfer scenario in which there is a nontrivial distribution shift between the source and target datasets. We also plan to study more scalable ways of obtaining fine-grained labels for training image classifiers. For example, we could first use large language models (LLMs) \cite{chatgpt2023} to decompose the coarse-grained labels, then use visual question answering (VQA) models \cite{instructblip, pali2023} to automatically classify the input samples in a fine-grained manner.

\section{Reproducibility Statement}
To ensure reproducibility and completeness of the experimental results of this paper, we discuss in detail the experimental procedures, relevant hyperparameter choices and ablation studies in Appendix \ref{section: appendix A}. Appendix \ref{section: theory, problem setup} to \ref{section: appendix, prob lemmas} are devoted to showing the complete theoretical results.

\bibliography{iclr2023_conference}
\bibliographystyle{iclr2023_conference}

\appendix
\etocdepthtag.toc{mtchapter}
\etocsettagdepth{mtchapter}{subsection}
\etocsettagdepth{mtappendix}{none}

\onecolumn
\begin{center}
    \huge Appendix
\end{center}

\etocdepthtag.toc{mtappendix}
\etocsettagdepth{mtchapter}{none}
\etocsettagdepth{mtappendix}{subsubsection}

\begingroup
\hypersetup{linkcolor=blue}
\tableofcontents
\endgroup

\newpage
\section{Additional Experimental Results}
\label{section: appendix A}
In this section, we present the full details of our experiments and relevant ablation studies. All of our experiments were performed using tools in the Scenic library \cite{dehghani2021scenic}.
\subsection{In-dataset transfer results}
To clarify, in this transfer setting, we are essentially transferring \textit{within} a dataset. More specifically, we set $\calX^{\text{src}} = \calX^{\text{tgt}}$ and only the label spaces $\calY^{\text{src}}$ and $\calY^{\text{tgt}}$ may differ (in distribution). The baseline in this setting is clear: train on $\calD^{\text{tgt}}_{\text{train}}$ and test on $\calD^{\text{tgt}}_{\text{test}}$. In contrast, after pretraining the backbone network $\vh(\mTheta; \cdot)$ on $\mathcal{Y}^{\text{src}}$, we finetune or linear probe it on $\calD^{\text{tgt}}_{\text{train}}$ using the backbone and then test on $\calD^{\text{tgt}}_{\text{test}}$.

\subsubsection{iNaturalist 2021}
\label{appdx_sec:inat2021}
On iNaturalist 2021, we use the mini training dataset with size 500,000 instead of the full training dataset to show a greater gap between the results of different hierarchies and speed up training. We use the architectures ResNet 34 and 50 \cite{kaiming2016}.

\textit{Training details}. Our pretraining pipeline on iNaturalist is essentially the same as the standard large-batch-size ImageNet-type training for ResNets \cite{kaiming2016, goyal2017}. The following pipeline applies to model pretraining on any hierarchy.
\begin{itemize}
    \item Optimization: SGD with 0.9 momentum coefficient, 0.00005 weight decay, 4096 batch size, 90 epochs total training length. We perform 7 epochs of linear warmup in the beginning of training until the learning rate reaches $0.1\times 4096/256 = 1.6$, and then apply the cosine annealing schedule.
    \item Data augmentation: subtracting mean and dividing by standard deviation, image (original or its horizontal flip) resized such that its shorter side is $256$ pixels, then a $224 \times 224$ random crop is taken.
\end{itemize}

For finetuning, we keep everything in the pipeline the same except setting the batch size to $4096/4 = 1024$ and base learning rate $1.6/4 = 0.4$. We found that finetuning at higher batch size and learning rate resulted in training instabilities and severely affected the final finetuned model's validation accuracy, while finetuning at lower batch size and learning rate than the chosen one resulted in lower validation accuracy at the end even though their training dynamics was stabler.

For the baseline accuracy, as mentioned in the main text, to ensure fairness of comparison, in addition to only training the network on the target 11-superclass problem for 90 epochs (using the same pretraining pipeline), we also perform ``retraining'': follow the exact training process of the models trained on the various hierarchies, but use  $\calD_{\text{train}}^{\text{tgt}}$ as the training dataset in both the pretrianing and finetuning stage. We observed consistent increase in the final validation accuracy of the model, so we report this as the baseline accuracy. Without retraining (so naive one-pass 90-epoch training on 11 superclasses), the average accuracy with standard deviation is $94.13, 0.025$.

\textit{Clustering}. To obtain the cluster-ID-based labels, we perform the following procedure. 
\begin{enumerate}
    \item For every sample $\mX_n$ in the mini training dataset of iNaturalist 2021, obtain its ViT-L/14 CLIP embedding $\mE_n$.
    \item Per-superclass kMeans clustering. Let $C$ be the predefined number of clusters per class.
    \begin{enumerate}
        \item For every superclass $k$, for the set of embedding $\{(\mE_n, y_n = k)\}$ belonging to that superclass, perform kMeans clustering with cluster size set to $C$.
        \item Given a sample with superclass ID $k\in\{1, 2, ..., 11\}$ and cluster ID $c\in\{1, 2, ..., C\}$, define its fine-grained ID as $C\times k + c$. 
    \end{enumerate}
    \item Whole-dataset kMeans clustering. Let $C$ be the predefined number of clusters on the whole dataset.
    \begin{enumerate}
        \item Perform kMeans on the embedding of all the samples in the dataset, with the number of clusters set to $C$. Set the fine-grained class ID of a sample to its cluster ID.
    \end{enumerate}
\end{enumerate}
Some might have the concern that having the same number of kMeans clusters per superclass could cause certain classes to have too few samples, which could be a reason for why the cluster ID hierarchies perform worse than the manual hierarchies. Indeed, the number of samples per superclass on iNaturalist is different, so in addition to the above ``uniform-number-of-cluster-per-superclass'' hierarchy, we add an extra label hierarchy by performing the following procedure to balance the sample size of each cluster:
\begin{enumerate}
    \item Perform kMeans for each superclass with number of clusters set to 2, 8, 32, 64, 128, 256, 512, 1024 and save the corresponding image-ID-to-cluster-ID dictionaries (so we are basically reusing the clustering results of the CLIP+kMeans per superclass experiment)
    \item For each superclass, find the image-ID-to-cluster-ID dictionary with the highest granularity while still keeping the minimum number of samples for each cluster $>$ predefined threshold (e.g. 1000 samples per subclass)
    \item Now we have nonuniform granularity for each superclass while ensuring that the sample count per cluster is above some predefined threshold.
\end{enumerate}
This simple procedure somewhat improves the balance of sample count per cluster, for example, Figure \ref{fig:inat2021, kmean per superclass, rebalanced} shows the sample count per cluster for the cases of total number of clusters = 608 and 1984. Unfortunately, we do not observe any meaningful improvement on the model's validation accuracy trained on this more refined hierarchy.

\begin{table*}[t!]
\setlength{\tabcolsep}{5pt}
\centering
\scalebox{0.7}{
\begin{tabular}{l c | c c c c c c c c c}
\toprule
Manual Hierarchy & $\calG(\mathcal{Y}^{\text{src}})$ & 11 & 13 & 51 & 273 & 1103 & 4884 & 6485\\
& Validation error & \textbf{\textcolor{red}{5.25$\pm$0.051}} & 5.40$\pm$0.075 & 5.10$\pm$0.038 & 4.83$\pm$0.041 & \textbf{4.79$\pm$0.045}& 4.82$\pm$0.056 & 4.84$\pm$0.033 \\
\hline
Random class ID & $\calG(\mathcal{Y}^{\text{src}})$ & 22 & 88 & 352 & 1,408 & 5,632 & 11,264 & 500,000\\
& Validation error & 6.61$\pm$0.215 & 6.30$\pm$0.070 & 6.12$\pm$0.77 & 6.10$\pm$0.053 & 6.12$\pm$0.042 & 6.10$\pm$0.057 & 6.54$\pm$0.758 \\
\hline
CLIP+kMeans & $\calG(\mathcal{Y}^{\text{src}})$ & 22 & 88 & 352 & 1408 & 2816 & 5632 & 22528\\
per superclass & Validation error & 5.14$\pm$0.049 & 5.16$\pm$0.033 & 5.17$\pm$0.027 & 5.24$\pm$0.029 & 5.30$\pm$0.029 & 5.31$\pm$0.077 & 5.37$\pm$0.032 \\
\hline
C+k per supclass & $\calG(\mathcal{Y}^{\text{src}})$ & 88 & 218 & 320 & 608 & 1040 & 1984\\
Class rebalanced & Validation error & 5.18$\pm$0.054 & 5.17$\pm$0.038 & 5.23$\pm$0.052 & 5.28$\pm$0.045 & 5.26$\pm$0.035 & 5.21$\pm$0.040 &  \\
\hline
CLIP+kMeans & $\calG(\mathcal{Y}^{\text{src}})$ & 22 & 44 & 88 & 352 & 1408 & 2816 & 5632 \\
whole dataset& Validation error & 5.52$\pm$0.015 & 5.42$\pm$0.047 & 5.45$\pm$0.049 & 5.46$\pm$0.019 & 5.60$\pm$0.029 & 5.50$\pm$0.029 & 5.47$\pm$0.029 \\
\bottomrule
\end{tabular}
}
\caption{\textbf{In-dataset transfer, iNaturalist 2021}. ResNet34 average finetuning validation error and standard deviation on 11 superclasses in iNaturalist 2021, pretrained on various label hierarchies with different label granularity. Baseline (11-superclass) and best performance are highlighted.}
\label{table: resnet34, inat2021, appendix 1}
\end{table*}

\begin{table*}[t!]
\setlength{\tabcolsep}{5pt}
\centering
\scalebox{0.8}{
\begin{tabular}{c  c | c c c c c c c c}
\toprule
90-Epoch ckpt & $\calG(\mathcal{Y}^{\text{src}})$ & 13 & 51 & 273 & 1103 & 4884 & 6485\\
& Validation error & 5.40$\pm$0.075 & 5.10$\pm$0.038 & 4.83$\pm$0.041 & 4.79$\pm$0.045 & 4.82$\pm$0.056 & 4.84$\pm$0.033 \\
\hline
70-Epoch ckpt & $\calG(\mathcal{Y}^{\text{src}})$ & 13 & 51 & 273 & 1103 & 4884 & 6485\\
& Validation error & 5.43$\pm$0.055 & 5.08$\pm$0.029 & 4.86$\pm$0.037 & 4.82$\pm$0.034 & 4.83$\pm$0.064 & 4.85$\pm$0.018 \\
\hline
50-Epoch ckpt & $\calG(\mathcal{Y}^{\text{src}})$ & 13 & 51 & 273 & 1103 & 4884 & 6485\\
& Validation error & 5.53$\pm$0.036 & 5.2$\pm$0.031 & 4.90$\pm$0.038 & 4.9$\pm$0.042 & 4.91$\pm$0.020 & 4.95$\pm$0.026 \\
\bottomrule
\end{tabular}
}
\caption{\textbf{In-dataset transfer, iNaturalist 2021}. ResNet34 average finetuned validation error and standard deviation on 11 superclasses in iNaturalist 2021, pretrained on the manual hierarchies, with different backbone checkpoints.}
\label{table: resnet34 different backbone ckpts, inat2021, appendix 1}
\end{table*}

\begin{table*}[t!]
\setlength{\tabcolsep}{5pt}
\centering
\scalebox{0.7}{
\begin{tabular}{l c | c c c c c c c c}
\toprule
Manual Hierarchy & $\calG(\mathcal{Y}^{\text{src}})$ & 11 & 13 & 51 & 273 & 1103 & 4884 & 6485\\
& Validation error & \textbf{\textcolor{red}{4.43$\pm$0.029}} & 4.44$\pm$0.063 & 4.36$\pm$0.062 & 4.22$\pm$0.021 & \textbf{4.20$\pm$0.035 }& 4.23$\pm$0.054 & 4.33$\pm$0.037 \\
\hline
Random class ID & $\calG(\mathcal{Y}^{\text{src}})$ & 22 & 88 & 352 & 1,408 & 5,632 & 11,264 & 500,000\\
& Validation error & 5.36$\pm$0.111 & 5.31$\pm$0.079 & 5.24$\pm$0.093 & 5.38$\pm$0.052 & 5.37$\pm$0.033 & 5.40$\pm$0.033 & 5.13$\pm$0.072 \\
\bottomrule
\end{tabular}
}
\caption{\textbf{In-dataset transfer, iNaturalist 2021}. ResNet50 finetuned average validation error and standard deviation on 11 superclasses in iNaturalist 2021, pretrained on label hierarchies with different label granularity.}
\label{table: resnet50, inat2021, appendix 1}
\end{table*}

\textit{Experimental procedures}. All the validation accuracies we report on ResNet34 are the averaged results of experiments performed on at least 6 random seeds: 2 random seeds for backbone pretraining and 3 random seeds for finetuning. We report the average accuracies with their standard deviation on various hierarchies in Table \ref{table: resnet34, inat2021, appendix 1}.

An additional experiment we performed with ResNet34 is a small grid search over what checkpoint of a pretrained backbone we should use for finetuning on the 11-superclass method; we tried the 50-, 70- and 90-epoch checkpoints of the backbone on the manual hierarchies. We report these results in Table \ref{table: resnet34 different backbone ckpts, inat2021, appendix 1}. As we can see, 90-epoch checkpoints performs almost equally well as the 70-epoch checkpoints and better than the 50-epoch ones by a nontrivial margin. With this observation, we chose to use the end-of-pretraining 90-epoch checkpoints in all our other experiments without further ablation studies on those hierarchies.

Our ResNet50 results are not as extensive as those on ResNet34. We present the average accuracies and standard deviations in Table \ref{table: resnet50, inat2021, appendix 1}.

\begin{figure}[t!]
    \centering
    \begin{subfigure}{\textwidth}
        \centering
        \includegraphics[width=.45\linewidth]{ 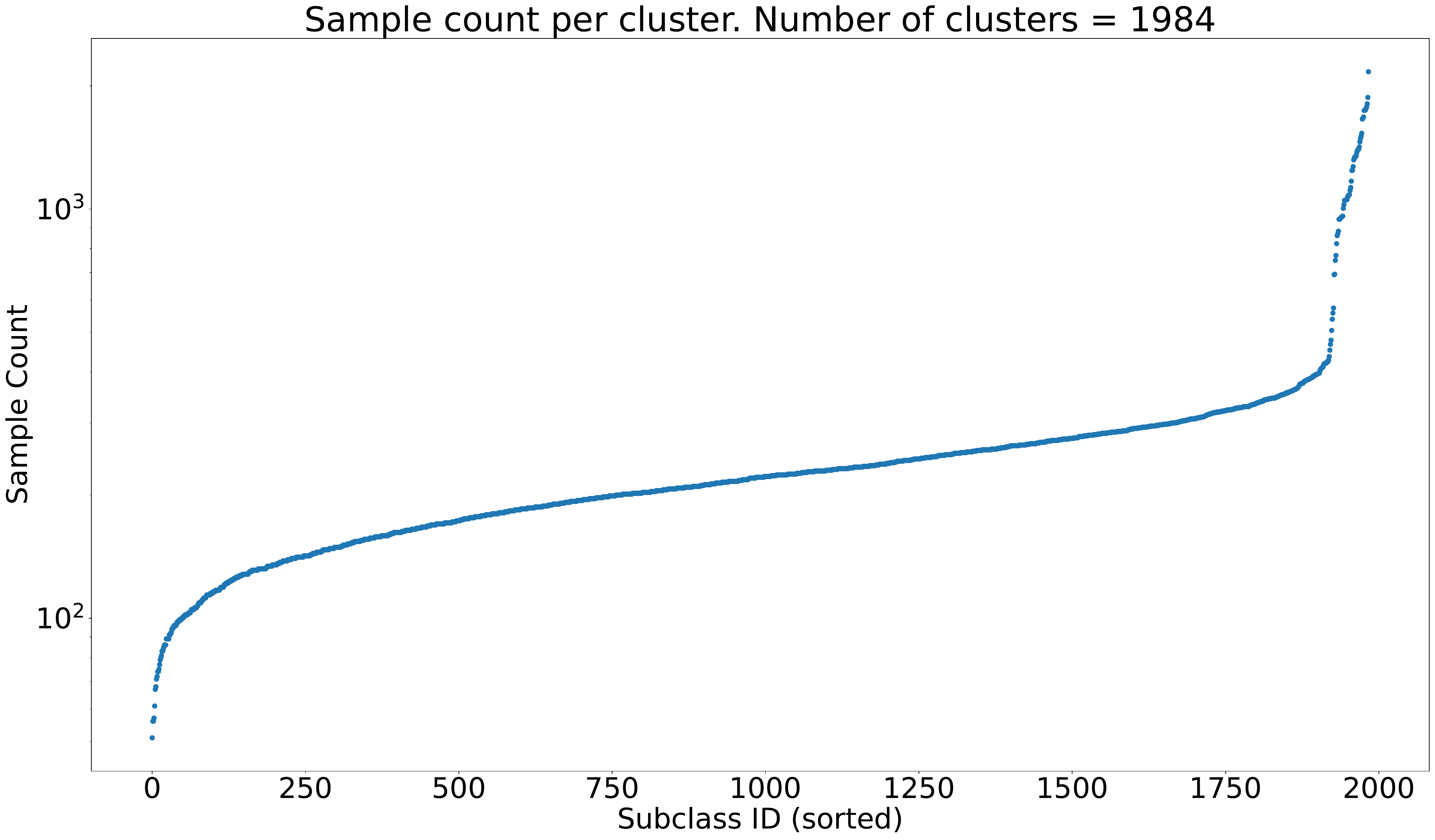}
        \includegraphics[width=.45\linewidth]{ 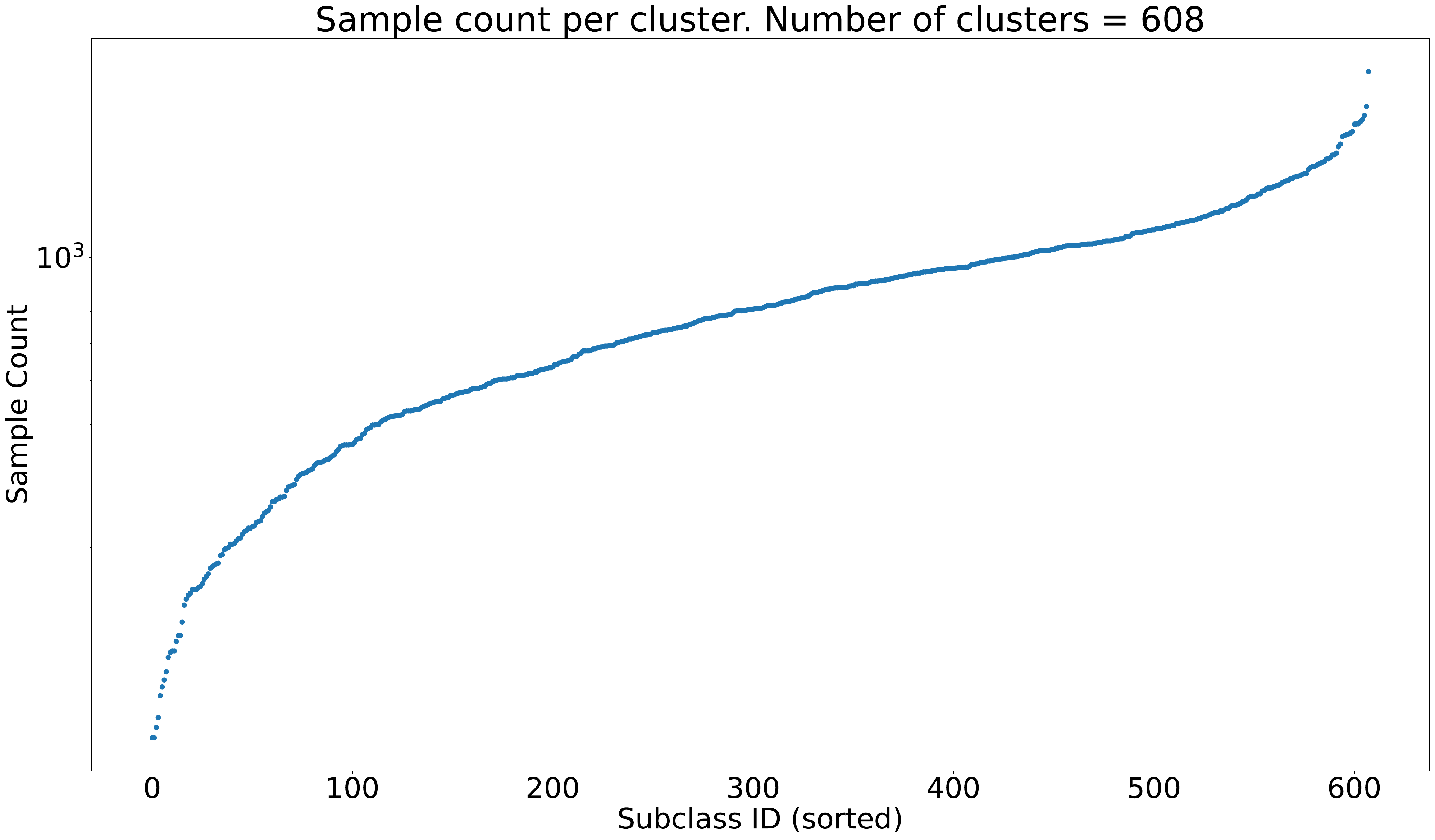}
    \end{subfigure}
    \caption{\textbf{In-dataset transfer, iNaturalist 2021}. Number of samples per cluster in the case of 608 and 1984 total clusters, after applying the sample size rebalancing procedure described in subsection \ref{appdx_sec:inat2021}. Observe that the sample sizes are reasonably balanced across almost all the subclasses.}
    \vspace{-2mm}
    \label{fig:inat2021, kmean per superclass, rebalanced}
\end{figure}

\subsubsection{ImageNet21k}

\begin{table}[t!]
\centering
\begin{tabular}{c c c}
\toprule
Hierarchy level & $\calG(\mathcal{Y}^{\text{src}})$ & Validation error\\
\hline
Baseline & 2 & \textbf{\textcolor{red}{7.90}} \\
0 (leaf) & 21843 & \textbf{6.56} \\
1 & 5995 & 6.76 \\
2 & 2281 & 6.70 \\
4 & 519 & 6.97 \\
6 & 160 & 7.31 \\
9 & 38 & 7.55 \\
\bottomrule
\end{tabular}
\caption{\textbf{In-dataset transfer}. ViT-B/16 validation error on the binary problem ``is this object a \textit{living thing}?'' of ImageNet21k. Pretrained on various hierarchy levels of ImageNet21k, finetuned on the binary problem. Observe that the maximal improvement appears at the leaf labels, and as $\calG(\mathcal{Y}^{\text{src}})$ approaches 2, the percentage improvement approaches 0.}
\label{table: vitb16, im21k to 21k}
\end{table}

The ImageNet21k dataset we experiment on contains a total of 12,743,321 training samples and 102,400 validation samples, with 21843 leaf labels. A small portion of samples have multiple labels.

Caution: due to the high demand on computational resources of training ViT models on ImageNet21k, all of our experiments that require (pre-)training or finetuning/linear probing on this dataset were performed with one random seed.

\textit{Hierarchy generation}. To define fine-grained labels, we start by defining the leaf labels of the dataset to be Hierarchy level 0. For every image, we trace from the leaf synset to the root synset relying on the WordNet hierarchy, and set the $k$-th synset (or the root synset, whichever is higher in level) as the level-$k$ label of this image; this procedure also applies to the multi-label samples. This is the way we generate the manual hierarchies shown in the main text.

Due to the lack of a predefined coarse-label problem, we manually define our target problem to be a binary one: given an image, if the synset ``Living Thing'' is present on the path tracing from the leaf label of the image to the root, assign label 1 to this image; otherwise, assign 0. This problem almost evenly splits the training and validation sets of ImageNet21k: 5,448,549:7,294,772 for training, 43,745:58,655 for validation. 

\textit{Network choice and pretraining pipeline}. We experiment with the ViT-B/16 model \cite{vit2021}. The pretraining pipeline of this model follows the one in \cite{vit2021} exactly: we train the model for 90 epochs using the Adam optimizer, with $\beta_1 = 0.9, \beta_2=0.999$, weight decay coefficient equal to 0.03 and a batch size of 4096; we let the dropout rate be 0.1; the output dense layer's bias is initialized to $-10.0$ to prevent huge loss value coming from the off-diagonal classes near the beginning of training \cite{cui2019cvpr}; for learning rate, we perform linear warmup for 10,000 steps until the learning rate reaches $10^{-3}$, then it is linearly decayed to $10^{-5}$. The data augmentations are the common ones in ImageNet-type training \cite{vit2021,kaiming2016}: random cropping and horizontal flipping. Note that we use the sigmoid cross-entropy for training since the dataset has multi-label samples.

\textit{Evaluation on the binary problem}. After the 90-epoch pretraining on the manual hierarchies, we evaluate the model on the binary problem. We report the best accuracies on each hierarchy level in Table \ref{table: vitb16, im21k to 21k}. To get a sense of how the relevant hyperparameters influence final accuracy of the model, we try out the following finetuning/linear probing strategies on the backbone trained on the \textit{leaf labels} and the target \textit{binary problem} of the dataset, and report the results in Table \ref{table: vitb16, in-dataset, ablation} (similar to our experiments on iNaturalist, we include the backbone trained on the binary problem in these ablation studies to ensure that our comparisons against the baseline are fair) :
\begin{enumerate}
    \item 90-epochs finetuning in the same fashion as the pretraining stage, but with a small grid search over 
    \begin{equation*}
    \begin{aligned}
        (\text{batch size}, \text{ base learning rate}) 
        = & \{(4096, 0.001), (4096/4=1024, 0.001/4 = 0.00025) , \\
        & (4096/8=512, 0.001/8 = 0.000125)\}.
    \end{aligned}
    \end{equation*}
    \item Linear probing with 20 epochs training length, using exactly the same training pipeline as in pretraining. We ran a small grid search over $(\text{batch size}, \text{ base learning rate}) = \{(4096, 0.001), (4096/8 = 512, 0.001/8 = 0.000125)\}$.
    \item 10-epochs finetuning, no linear warmup, 3 epochs of constant learning rate in the beginning followed by 7 epochs of linear decay, with a small grid search over $(\text{batch size}, \text{ base learning rate}) = \{(4096, 0.001), (4096/8 = 512, 0.001/8 = 0.000125)\}$.
\end{enumerate}
Table \ref{table: vitb16, in-dataset, ablation} helps us decide the best accuracies to report. First, as expected the linear probing results are much worse than the finetuning ones. Second, the ``retraining'' accuracy of 92.102 is the best baseline we can report (the same thing happened in the iNaturalist case) --- if we only train the model for 90 epochs (the naive one-pass training) on the binary problem, then the model's final validation accuracy is 91.746\%, which is lower than 92.102\% by a nontrivial margin. In contrast, the short 10-epoch finetuning strategy works best for the backbone trained on the leaf labels, therefore, we also use this strategy to evaluate the backbones trained on all the other manual hierarchies. A peculiar observation we made was that, finetuning the leaf-labels-pretrained backbone for extended period of time on the binary problem caused it to overfit severely: for batch size and base learning rate in the set $\{(4096, 0.001), (1024, 0.00025), (512, 0.000125)\}$, throughout the 90 epochs of finetuning, although its training loss exhibits the normal behavior of staying mostly monotonically decreasing, its validation accuracy actually reached its peak during the linear warmup period!

\begin{table*}[t!]
\setlength{\tabcolsep}{4pt}
\centering
\footnotesize
\scalebox{0.74}{
\begin{tabular}{l c | c c c | c c | c c}
\toprule
 & Eval strategy & \multicolumn{3}{c|}{90-epoch finetune} & \multicolumn{2}{c|}{Linear probe} & \multicolumn{2}{c}{10-epoch finetune} \\
\hline
\multirow{2}{*}{Leaf-pretrained} & (Batch size, base lr) & (4096,1e-3) & (1024,2.5e-4) & (512,1.25e-4) & (4096, 1e-3) & (512, 1.25e-4) & (4096,1e-3) & (512,1.25e-4) \\
& Validation error & 92.782 & 93.177 & 93.295 & 87.497 & 87.493 & 92.294 & \textbf{93.439}  \\
\hline
 \multirow{2}{*}{Baseline} & (Batch size, base lr) & (4096,1e-3) & (1024,2.5e-4) & (512,1.25e-4) & (4096, 1e-3) & (512, 1.25e-4) & (4096,1e-3) & (512,1.25e-4) \\
& Validation error & \textcolor{red}{\textbf{92.102}} & 91.971 & 91.939 & 91.703 & 91.719 & 92.002 & 91.856  \\
\bottomrule
\end{tabular}
}
\caption{\textbf{In-dataset transfer, ImageNet21k}. ViT-B/16 validation accuracy on the binary problem ``Is the object a Living Thing'' on ImageNet21k. Ablation study on the exact finetuning/linear probing strategy.}
\label{table: vitb16, in-dataset, ablation}
\end{table*}

\subsubsection{ImageNet1k}
\begin{table*}[t!]
\setlength{\tabcolsep}{10pt}
\centering
\scalebox{0.93}{
\begin{tabular}{l c | c c c }
\toprule
ResNet50 CLIP+kMeans & $\calG(\mathcal{Y}^{\text{src}})$ & 2000 & 4000 & 8000 \\
per-class & Validation error & 23.4$\pm$0.13 & 23.48$\pm$0.098 & 23.49$\pm$0.204 \\
\hline
ViT-L/14 CLIP+kMeans & $\calG(\mathcal{Y}^{\text{src}})$ & 2000 & 4000 & 8000 \\
per-class & Validation error & 23.4$\pm$0.127 & 23.47$\pm$0.074 & 23.78$\pm$0.048 \\
\hline
Random ID & $\calG(\mathcal{Y}^{\text{src}})$ & 2000 & 4000 & 8000 \\
per-class & Validation error & 23.4$\pm$0.068 & 23.4$\pm$0.070 & 23.65$\pm$0.071 \\
\bottomrule
\end{tabular}
}
\caption{\textbf{In-dataset transfer, ImageNet1k}. ResNet50 finetuned average validation error and standard deviation on the vanilla 1000 classes, pretrained on label hierarchies with different label granularity.}
\label{table: resnet50, in-dataset, imagenet1k, appendix 1}
\end{table*}

Our ImageNet1k in-dataset transfer experiments are done in a very similar fashion to the iNaturalist ones. In particular, the pretraining and finetuning pipeline for ResNet50 is exactly the same as the one in the iNaturalist case, so we do not repeat it here. 

Due to a lack of more fine-grained manual label on this dataset, we generate fine-grained labels by performing kMeans on the ViT-L/14 CLIP embedding of the dataset separately for each class; the exact procedure is also identical to the iNaturalist case. The CLIP backbones we use here are the ResNet50 version and the ViT-L/14 version. We report the average accuracies and their standard deviation in Table \ref{table: resnet50, in-dataset, imagenet1k, appendix 1}. All results are obtained from at least one random seed during pretraining and 3 random seeds during finetuning.

The best baseline we report is the one using retraining: if we adopt the pretrain-then-finetune procedure but with $\calD_{\text{train}}^{\text{tgt}}$ (i.e. the vanilla 1000-class labels) set as the pretraining dataset, then we obtain an average validation error of 23.28\% with standard deviation of 0.103, averaged over results of 3 random seeds. In comparison, if we only perform the naive one-pass 90-epoch training, we obtain average valiation error 24.04\%, with standard deviation 0.057.

From Table \ref{table: resnet50, in-dataset, imagenet1k, appendix 1}, we see that there is virtually no difference between the baseline and the best errors obtained by the models trained on the custom hierarchies: they are almost equally bad. Noting that the sample size of each class in ImageNet1k is only around $10^3$, and the fact that ImageNet1k classification is a ``hard problem'' --- it is a problem of high sample complexity --- further decomposing the classes causes each fine-grained class to have too few samples, leading to the above negative results. This reflects the intuition that higher label granularity does not necessarily mean better model generalization, since the sample size per class might become too small.

\subsection{Cross-dataset transfer, ImageNet21k$\to$ImageNet1k}
\label{appendix: im21k cross-dataset}

In this subsection, we report the average validation accuracy and standard deviation of the cross-dataset transfer experiment from ImageNet21k to ImageNet1k, as discussed in Figure \ref{fig: im21k->im1k, linear probe} and Section \ref{subsect: im21k->im1k main text details} in the main text.

\textit{Network choice}. We use the same architecture ViT-B/16 as the one in the in-dataset ImageNet21k transfer experiment and follow the same training procedure, which we repeat here for the reader's convenience. The pretraining pipeline of this model follows the one in \cite{vit2021}: we train the model for 90 epochs using the Adam optimizer, with $\beta_1 = 0.9, \beta_2=0.999$, weight decay coefficient equal to 0.03 and a batch size of 4096; we let the dropout rate be 0.1; the output dense layer's bias is initialized to $-10.0$ to prevent huge loss value coming from the off-diagonal classes near the beginning of training \cite{cui2019cvpr}; for learning rate, we perform linear warmup for 10,000 steps until the learning rate reaches $10^{-3}$, then it is linearly decayed to $10^{-5}$. The data augmentations are the common ones in ImageNet-type training \cite{vit2021,kaiming2016}: random cropping and horizontal flipping. Note that we use the sigmoid cross-entropy for training since the dataset has multi-label samples.

\textit{Finetuning}. For finetuning on ImageNet1k, our procedure is very similar to the one in the original ViT paper \cite{vit2021}, described in its Appendix B.1.1. We optimize the network for 8 epochs using SGD with momentum factor set to 0.9, zero weight decay, and batch size of 512. The dropout rate, unlike in pretraining, is set to 0. Gradient clipping at 1.0 is applied. Unlike \cite{vit2021}, we still finetune at the resolution of 224$\times$224. For learning rate, we apply linear warmup for 500 epochs until it reaches the base learning rate, then cosine annealing is applied; we perform a small grid search of $\text{base learning rate} = \{3\times10^{-3}, 3\times10^{-2}, 6\times10^{-2}, 3\times10^{-1}\}$. Every one of these grid search is repeated over 3 random seeds. We report the ImageNet1k validation accuracies and their standard deviations in Table \ref{table: vitb16, im21k to 1k, finetune, appendix}. In the main text, we report the best accuracy for each hierarchy level.

\textit{Linear probing}. For linear probing, we use the following procedure. We optimize the linear classifier for 40 epochs (similar to \cite{comp_rep2021}) using SGD with Nesterov momentum factor set to 0.9, a small weight decay coefficient $10^{-6}$, and batch size 512. We start with a base learning rate of 0.9, and multiply it by 0.97 per 0.5 epoch. In terms of data augmentation, we adopt the standard ones like before: horizontal flipping and random cropping of size 224$\times$224. We repeat this linear probing procedure over 3 random seeds given the pretrained backbone, and report the average validation accuracy and standard deviation in Table \ref{table: vitb16, im21k to 1k, linear probe, appendix}.

\textit{Baseline}. The baseline accuracy on ImageNet1k is directly taken from the ViT paper \cite{vit2021} (see Table 5 in it), in which the ViT-B/16 model is trained for 300 epochs on ImageNet1k.

\begin{table}[t]
\centering
\setlength{\tabcolsep}{10pt}
\scalebox{0.95}{
\begin{tabular}{l | c  c  c  c  c}
\toprule
Pretrained on / Base lr & $3\times10^{-3}$ & $3\times10^{-2}$ & $6\times10^{-2}$ & $3\times10^{-1}$\\
\hline
ImageNet21k, Hier. lv. 0 & 80.87$\pm$0.012 & 82.48$\pm$0.005 & \textbf{82.51$\pm$0.042} & 81.40$\pm$0.041 \\
ImageNet21k, Hier. lv. 1 & 77.38$\pm$0.037 & 81.03$\pm$0.054 & \textbf{81.28$\pm$0.045} & 80.40$\pm$0.087 \\
ImageNet21k, Hier. lv. 2 & 74.91$\pm$0.012 & 79.76$\pm$0.021 & \textbf{80.26$\pm$0.05} & 79.7$\pm$0.019 \\
ImageNet21k, Hier. lv. 4 & 63.65$\pm$0.052 & 76.43$\pm$0.033 & \textbf{77.32$\pm$0.088} & 77.53$\pm$0.078 \\
ImageNet21k, Hier. lv. 6 & 62.17$\pm$0.012 & 73.65$\pm$0.033 & 73.92$\pm$0.073 & \textbf{75.53$\pm$0.024} \\
ImageNet21k, Hier. lv. 9 & 53.68$\pm$0.034 & 69.33$\pm$0.045 & 71.08$\pm$0.068 & \textbf{72.75$\pm$0.071} \\
\bottomrule
\end{tabular}
}
\caption{\textbf{Cross-dataset transfer}. ViT-B/16 average \textit{finetuning} validation accuracy on ImageNet1k along with standard deviation, pretrained on various hierarchy levels of ImageNet21k, and a small grid search over the base learning rate. }
\label{table: vitb16, im21k to 1k, finetune, appendix}
\end{table}

\begin{table}[t!]
\centering
\begin{tabular}{l c c c}
\toprule
Pretrained on & Hier. lv & $\calG(\mathcal{Y}^{\text{src}})$ & Validation acc.\\
\hline
IM21k & 0 (leaf) & 21843 & \textbf{81.45$\pm$0.021} \\
\cmidrule{2-4}
& 1 & 5995 & 78.33$\pm$0.018 \\
\cmidrule{2-4}
& 2 & 2281 & 75.66$\pm$0.005 \\
\cmidrule{2-4}
& 4 & 519 & 68.95$\pm$0.051 \\
\cmidrule{2-4}
& 6 & 160 & 63.65$\pm$0.035 \\
\cmidrule{2-4}
& 9 & 38 & 57.35$\pm$0.016 \\
\bottomrule
\end{tabular}
\caption{\textbf{Cross-dataset transfer}. ViT-B/16 average \textit{linear-probing} validation accuracy on ImageNet1k along with standard deviation, pretrained on various hierarchy levels of ImageNet21k. }
\label{table: vitb16, im21k to 1k, linear probe, appendix}
\end{table}

\newpage
\section{Theory, Problem Setup}
\label{section: theory, problem setup}
\subsection{Data Properties}
\begin{enumerate}
    \item Coarse classification: a binary task, $+1$ vs. $-1$.
    \item An input sample $\mX\in\mathbb{R}^{d\times P}$ consists of $P$ patches, each with dimension $d$. In this work, always assume $d$ is sufficiently large\footnote{Consider each $d$-dimensional patch of the input as an embedding of the input image generated by, for instance, an intermediate layer of a DNN.};
    \item Assume there exists $k_+$ subclasses of the superclass ``$+$'', and $k_-$ subclasses of the superclass ``$-$''. Let $k_+ = k_-$.
    \item Assume orthonormal dictionary $\calV = \{\vv_1, ..., \vv_d\} \subset \mathbb{R}^d$, which forms an orthonormal basis of $\mathbb{R}^d$. Define $\vv_+\in\calV$ to be the common feature of class ``$+$''. For each subclass $(+,c)$ (where $c\in[k_+]$), denote the subclass feature of it as $\vv_{+,c} \in \calV$. Similar for the ``$-$'' class.
    \item For an easy sample $\mX$ belonging to the $(+,c)$ class (for $c\in[k_+]$), we sample its patches as follows:
    
    \textbf{Definition}: we define the function $\calP: \mathbb{R}^{d\times P} \times \calV \to [P]$ (so $(\mX; \vv) \mapsto I \subseteq [P]$) to extract, from sample $\mX$, the indices of the patches on which the dictionary word $\vv\in\calD$ dominates.

    \begin{enumerate}
        \item (Common-feature patches) With probability $\frac{s^*}{P}$, a patch $\vx_p$ in $\mX$ is a common-feature patch, on which $\vx_{p} = \alpha_{p}\vv_{+} + \vzeta_{p}$ for some (random) $\alpha_{p} \in \left[\sqrt{1 - \iota}, \sqrt{1 + \iota}\right]$;
        \item (Subclass-feature patches) With probability $\frac{s^*}{P-\vert \calP(\mX; \vv_{+})\vert}$, a patch with index $p \in \left([P] - \calP(\mX; \vv_{+})\right)$ is a subclass-feature patch, on which $\vx_{p} = \alpha_{p}\vv_{+,c} + \vzeta_{p}$, for random $\alpha_{p} \in \left[\sqrt{1 - \iota}, \sqrt{1 + \iota}\right]$;
        \item (Noise patches) For the remaining $P-|\calP(\mX; \vv_{+})|-|\calP(\mX; \vv_{+,c})|$ patches, $\vx_p = \vzeta_p$.
    \end{enumerate}
    
    \item A hard sample $\mX_{\text{hard}}$ for class $(+,c)$ is exactly the same as an easy one except:
    \begin{enumerate}
        \item Its common-feature patches are replaced by noise patches;
        \item (Feature noise patches) With probability $\frac{s^{\dagger}}{P - \vert \calP(\mX; \vv_{+,c})\vert}$, a patch with index $p \in \left([P] - \calP(\mX; \vv_{+,c})\right)$ is a feature-noise patch, on which $\vx_{p} = \alpha_{p}^{\dagger}\vv_{-} + \vzeta_{p}$ for some (random) $\alpha_{p} \in \left[\iota^{\dagger}_{lower}, \iota^{\dagger}_{upper}\right]$;
        \item Set one of the noise patches to $\vzeta^*\sim\calN(\vzero,\sigma_{\zeta^*}^2\mI_d)$.
    \end{enumerate}

    \item A sample $\mX$ belongs to the ``$+$'' superclass if $|\calP(\mX; \vv_{+})| > 0$ or $|\calP(\mX; \vv_{+,c})|>0$ for any $c$ (excluding feature-noise patches). 

    \item The above sample definitions also apply to the ``$-$'' classes by switching the class signs.
    \item A training batch of samples contains exactly $N/2k_+$ samples for each $(+,c)$ and $(-,c)$ subclass. This also means that each training batch contains exactly $N/2$ samples belonging to the $+1$ superclass, and $N/2$ samples for the $-1$ superclass.
    \item As discussed in the  main text, for both coarse-grained (baseline) and fine-grained training, we only train on \textit{easy} samples.
\end{enumerate}

\subsection{Learner Assumptions}
Assume the learner is a two-layer convolutional ReLU network:
\begin{equation}
    F_{c}(\mX) = \sum_{r=1}^m a_{c,r}\sum_{p=1}^P \sigma(\langle \vw_{c,r}, \vx_p\rangle + b_{c,r})
\end{equation}
To simplify analysis and only focus on the learning of the feature extractor, we freeze $a_{c,r} = 1$ throughout training. The nonlinear activation $\sigma(\cdot) = \max(0,\cdot)$ is ReLU. Note that the convolution kernels have dimension $d$ and stride $d$.

\begin{remark} 
One difference between this architecture and a CNN used in practice is that we do not allow feature sharing across classes: for each class $c$, we are assigning a disjoint group of neurons $\vw_{c,r}$ to it. Separating neurons for each class is a somewhat common trick to lower the complexity of analysis in DNN theory literature \cite{zhu2020_kd,stefani2021,cao2022benign}, as it reduces complex coupling between neurons \textit{across} classes which is not the central focus of our study in this paper.
\end{remark}

\subsection{Training Algorithm}
\textbf{Initialization}.

Sample $\vw_{c,r}^{(0)} \sim \calN(\vzero, \sigma_0^2 \mI_d)$, and set $b_{c,r}^{(0)} = - \sigma_0 c_b \sqrt{\ln(d)}$.

\textbf{Training}.

We adopt the standard cross-entropy training:
\begin{equation}
    \calL(F) = \sum_{n=1}^N L(F; \mX_n, y_n) = -\sum_{n=1}^N \ln\left(\frac{\exp(F_{y_n}(\mX_n))}{\sum_{c=1}^C \exp(F_{c}(\mX_n))} \right)
\end{equation}
This induces the stochastic gradient descent update for each hidden neuron ($c\in[k], r\in[m]$) per minibatch of $N$ iid samples:
\begin{equation}
\begin{aligned}
    \vw_{c,r}^{(t+1)} 
    = \vw_{c,r}^{(t)} + \eta \frac{1}{NP} \sum_{n=1}^N \Bigg( 
    & \mathbbm{1}\{y_n = c\}[1-\text{logit}_c^{(t)}(\mX_n^{(t)})]\sum_{p\in[P]}\sigma'(\langle \vw_{c,r}^{(t)}, \vx_{n,p}^{(t)} \rangle +b_{c,r}^{(t)}) \vx_{n,p}^{(t)} + \\
    & \mathbbm{1}\{y_n\neq c\} [-\text{logit}_c^{(t)}(\mX_n^{(t)}) ]\sum_{p\in[P]} \sigma'(\langle \vw_{c,r}^{(t)}, \vx_{n,p}^{(t)}  \rangle + b^{(t)}_{c,r}) \vx_{n,p}^{(t)}\Bigg)
\end{aligned}
\end{equation}

where 
\begin{equation}
    \text{logit}_c^{(t)}(\mX) = \frac{\exp(F_{c}(\mX))}{\sum_{y=1}^C \exp(F_{y}(\mX))} 
\end{equation}

As for the bias,
\begin{equation}
    b_{c,r}^{(t+1)} = b_{c,r}^{(t)} - \frac{\|\vw_{c,r}^{(t+1)} - \vw_{c,r}^{(t)}\|_2}{\ln^5(d)}
\end{equation}

Additionally, we train the models until the network's output margins are sufficiently large (or the network is sufficiently confident in its decision). More specifically, we train the models such that before $F_{y_n}(X_n) - \max_{y \neq y_n} F_{y}(X_n) \ge \Omega(1)$ for all $n$ (replace $X_n$ by $X_n^{(t)}$ if we are performing stochastic gradient descent), we do not early stop the model. For analysis purposes, we allow the models to train for longer in some of our theorems.

\begin{remark}
\begin{enumerate}
    \item The initialization strategy is similar to the one in \cite{zhu2022_adv}.
    \item Since the only difference between the training samples of coarse and fine-grained pretraining is the label space, the form of SGD update is identical. The only difference is the number of output nodes of the network: for coarse training, the output nodes are just $F_+$ and $F_-$ (binary classification), while for fine-grained training, the output nodes are $F_{+,1}, F_{+,2}, ..., F_{+,k_+}, F_{-,1}, F_{-,2}, ..., F_{-,k_-}$, a total of $k_+ + k_-$ nodes.
    \item The bias is for thresholding out the neuron's noisy activations that grow slower than $1/\ln^5(d)$ times the activations on the main features which the neuron detects. This way, the bias does not really influence updates to the neuron's response to the core features which it activate strongly on, since $1 - \frac{1}{\ln^5(d)} \approx 1$, while it removes useless low-magnitude noisy activations. This in fact creates a (generalization) gap between the nonlinear model that we are studying and linear models. Due to our parameter choices (as discussed below), if the model has no nonlinearity (remove the ReLU activations), then even if the model can be written as $F_+(\mX) = \sum_{p\in[P]} c_+\langle \vv_+, \vx_{p} \rangle + c_{+,1}\langle \vv_{+,1}, \vx_{p} \rangle + ... + c_{+,k_+}\langle \vv_{+,k_+}, \vx_{p} \rangle$ and $F_-(\mX) = \sum_{p\in[P]} c_-\langle \vv_-, \vx_{p} \rangle + c_{-,1}\langle \vv_{-,1}, \vx_{p} \rangle + ... + c_{-,k_-}\langle \vv_{-,k_-}, \vx_{p} \rangle$ for any sequence of nonnegative real numbers $c_+, c_-, \{c_{+,j}\}_{j=1}^{k_+}, \{c_{-,j}\}_{j=1}^{k_-}$ (which is the ideal situation since the true features are not corrupted by anything), it is impossible for the model to reach $o(1)$ error on the input samples, because the number of noise patches will accumulate to a variance of $\left(P - O(s^*)\right)\sigma_{\zeta} \gg O(s^*)$, which significantly overwhelms the signal from the true features. On the other hand, each noise patch is sufficiently small in magnitude with high probability (their strength is $o(1/\ln^5(d))$), so a slightly negative bias, as described above, can threshold out these noise-based signals and prevent them from accumulating across the patches.
\end{enumerate}

\end{remark}

\subsection{Parameter Choices}
The following are fixed choices of parameters for the sake of simplicity in our proofs. 
\begin{enumerate}
    \item Always assume $d$ is sufficiently large. All of our asymptotic results are presented with respect to $d$;
    \item $\poly(d)$ denotes the asymptotic order ``polynomial in $d$'';
    \item $\polyln(d)$ aymptotic order ``polylogarithmic in $d$'';
    \item $\polyln(d) \le k_+ = k_- \le d^{0.4}$ and $s^*\ln^5(d) \le k_+$ (i.e. $k_+$ lower bounded by polynomial of $\ln(d)$ of sufficiently high degree);
    \item Small positive constant $c_0\in(0,0.1)$;
    \item For coarse-grained (baseline) training, set $c_b = \sqrt{4 + 2c_0}$, and for fine-grained training, set $c_b = \sqrt{2 + 2c_0}$;
    \item $0 \le \iota \le \frac{1}{\polyln(d)}$;
    \item $\iota^{\dagger}_{lower} \ge \frac{1}{\ln^4(d)}$, and $s^{\dagger}\iota^{\dagger}_{upper} \le O\left(\frac{1}{\ln(d)}\right)$;
    \item $s^{\dagger} \ge 1$;
    \item $s^* \in \polyln(d)$ with a degree $> 15$;
    \item $\sigma_{\zeta} = \frac{1}{\ln^{10}(d)\sqrt{d}}$;
    \item $\sigma_{\zeta^*} \in \left[\omega\left(\frac{\polyln(d)}{\sqrt{d}}\right), O\left(\frac{1}{\polyln(d)}\right)\right]$;
    \item $P\sigma_{\zeta} \ge \omega(\polyln(d))$, and $P \le \poly(d)$;
    \item $\sigma_0 \le O\left(\frac{1}{d^3s^*\ln(d)} \right)$, and set $\eta = \Theta(\sigma_0)$ for simplicity;
    \item Batch of samples $\calB^{(t)}$ at every iteration has a deterministic size of $N \in (\Omega(\polyln(d) k_+d), \poly(d))$.
    \item Note: we sometimes abuse the notation  $x = a \pm b$ as an abbreviation for $x \in [a-b, a+b]$.
\end{enumerate}
\begin{remark}
We believe the range of parameter choice can be (asymptotically) wider than what is considered here, but for the purpose of illustrating the main messages of the paper, we do not consider a more general set of parameter choice necessary because having a wider range of it can significantly complicate and obscure the already lengthy proofs without adding to the core messages.

Additionally, the function $f(\sigma_{\zeta})$ in the main text is set to $\frac{\sigma_{\zeta}^{-1}}{\ln^{10}(d)d^{0.1}} = d^{0.4}$ in this appendix for derivation convenience.
\end{remark}

\subsection{Plan of presentation}
We shall devote the majority of our effort to proving results for the coarse-label learning dynamics, starting with appendix section \ref{section: appendix init geometry} and ending on \ref{section: appendix, phase II coarse}, and only devote section \ref{section: appendix, finegrained} to the fine-grained-label learning dynamics, since the analysis of fine-grained training overlaps significantly with the coarse-grained one.

\newpage
\section{Coarse-grained training, Initialization Geometry}
\textbf{For coarse-grained training, assume $m = \Theta(d^{2 + 2c_0})$}.

\label{section: appendix init geometry}
\begin{definition}
Define the following sets of interest of the hidden neurons:
\begin{enumerate}
    \item $\calU_{+,r}^{(0)} = \{\vv \in \calV: \langle \vw_{+,r}^{(0)}, \vv\rangle \ge \sigma_0 \sqrt{4 + 2c_0}\sqrt{\ln(d) - \frac{1}{\ln^5(d)}}\}$
    \item Given $\vv \in \calV$, $S^{*(0)}_+(\vv) \subseteq + \times [m]$ satisfies:
    \begin{enumerate}
        \item $\langle \vw_{+,r}^{(0)}, \vv \rangle \ge \sigma_0 \sqrt{4 + 2c_0} \sqrt{\ln(d) + \frac{1}{\ln^5(d)}}$
        \item $\forall \vv' \in \calV \text{ s.t. } \vv' \perp \vv, \, \langle \vw_{+,r}^{(0)}, \vv' \rangle < \sigma_0 \sqrt{4 + 2c_0} \sqrt{\ln(d) - \frac{1}{\ln^5(d)}}$ 
    \end{enumerate}
    \item Given $\vv \in \calD$, $S_{+}^{(0)}(\vv) \subseteq + \times [m]$ satisfies:
    \begin{enumerate}
        \item $\langle \vw_{+,r}^{(0)}, \vv \rangle \ge \sigma_0 \sqrt{4 + 2c_0} \sqrt{\ln(d) - \frac{1}{\ln^5(d)}}$
    \end{enumerate}
    \item For any $(+,r) \in S_{+,reg}^{*(0)} \subseteq + \times [m]$:
    \begin{enumerate}
        \item $\langle \vw_{+,r}^{(0)}, \vv \rangle \le \sigma_0 \sqrt{10} \sqrt{\ln(d)} \; \forall \vv\in\calV$
        \item $\left\vert \calU_{+,r}^{(0)} \right\vert \le O(1)$
    \end{enumerate}
\end{enumerate}
\end{definition}
\begin{prop}
\label{prop: init geometry, coarse}
Assume $m = \Theta(d^{2 + 2c_0})$, i.e. the number of neurons assigned to the $+$ and $-$ class are equal and set to $\Theta(d^{2 + 2c_0})$.

At $t=0$, for all $\vv \in \calV$, the following properties are true with probability at least $1- d^{-2}$ over the randomness of the initialized kernels:
\begin{enumerate}
    \item $|S_+^{*(0)}(\vv)|, |S_+^{(0)}(\vv)| = \Theta\left(\frac{1}{\sqrt{\ln(d)}}\right) d^{c_0}$
    \item In particular, for any $\vv,\vv' \in \calD$, $\left\vert \frac{|S_+^{*(0)}(\vv)|}{|S_+^{*(0)}(\vv')|} - 1 \right\vert, \left\vert \frac{|S_+^{*(0)}(\vv)|}{|S_+^{(0)}(\vv')|} - 1 \right\vert \le O\left( \frac{1}{\ln^5(d)}\right)$
    \item $S_{+,reg}^{(0)} = [m]$
\end{enumerate}
\end{prop}
\begin{proof}
Recall the tail bound of $g \sim \calN(0, 1)$ for every $\epsilon > 0$:
\begin{equation}
\begin{aligned}
\frac{1}{2}\frac{1}{\sqrt{2\pi}} \frac{\epsilon}{\epsilon^2 + 1} e^{-\epsilon^2/2} \le \mathbb{P}\left[ g \ge \epsilon \right] \le \frac{1}{2} \frac{1}{\sqrt{2\pi}} \frac{1}{\epsilon} e^{-\epsilon^2/2}
\end{aligned}
\end{equation}
First note that for any $r\in[m]$, $\{\langle \vw_{+,r}^{(0)}, \vv \rangle\}_{\vv\in\calV}$ is a sequence of iid random variables with distribution $\calN(0, \sigma_0^2)$. 

The proof of the first point proceeds in two steps.
\begin{enumerate}
    \item The following properties hold at $t=0$:
    \begin{equation}
    \begin{aligned}
        p_1 \coloneqq & \mathbb{P}\left[ \langle \vw_{+,r}^{(0)}, \vv \rangle \ge \sigma_0 \sqrt{4 + 2c_0} \sqrt{\ln(d) + \frac{1}{\ln^5(d)}} \right] \\
        \in & \frac{1}{\sqrt{8\pi}} d^{-2-c_0} e^{(-2-c_0)/\ln^5(d)}\left[ \frac{\sqrt{(4+2c_0)\left(\ln(d) + \frac{1}{\ln^5(d)}\right)}}{(4+2c_0)\left(\ln(d) + \frac{1}{\ln^5(d)}\right) + 1}, \frac{1}{\sqrt{(4+2c_0)\left(\ln(d) + \frac{1}{\ln^5(d)}\right)}}\right] \\
        = &\Theta\left(\frac{1}{\sqrt{\ln(d)}}\right)d^{-2-c_0}
    \end{aligned}
    \end{equation}
    and 
    \begin{equation}
    \begin{aligned}
        p_2 \coloneqq & \mathbb{P}\left[ \langle \vw_{+,r}^{(0)}, \vv \rangle \ge \sigma_0 \sqrt{4 + 2c_0} \sqrt{\ln(d) - \frac{1}{\ln^5(d)}} \right] \\
        \in & \frac{1}{\sqrt{8\pi}} d^{-2-c_0} e^{-(-2-c_0)/\ln^5(d)}\left[ \frac{\sqrt{(4+2c_0)\left(\ln(d) - \frac{1}{\ln^5(d)}\right)}}{(4+2c_0)\left(\ln(d) -  \frac{1}{\ln^5(d)}\right) + 1}, \frac{1}{\sqrt{(4+2c_0)\left(\ln(d) -  \frac{1}{\ln^5(d)}\right)}}\right] \\
        = & \Theta\left(\frac{1}{\sqrt{\ln(d)}}\right)d^{-2-c_0}
    \end{aligned}
    \end{equation}
    Therefore, for any $r\in[m]$, the random event described in $S_{+}^{*(0)}$ holds with probability
    \begin{equation}
        p_1 \times (1-p_2)^{d-1} = \Theta\left(\frac{1}{\sqrt{\ln(d)}}\right)d^{-2-c_0} \times \left( 1 -  \Theta\left(\frac{1}{\sqrt{\ln(d)}}\right)d^{-2-c_0}\right)^{d-1} =  \Theta\left(\frac{1}{\sqrt{\ln(d)}}\right)d^{-2-c_0}
    \end{equation}
    The last equality holds because defining $f(d) = d^{-2-c_0}$ and $d$ being sufficiently large,
    \begin{equation}
        g(d) \coloneqq |(d-1)\ln(1-f(d))| \le (d-1) \times (f(d) + O(f(d)^2)) \le O(d^{-1})
    \end{equation}
    which means 
    \begin{equation}
        (1-f(d))^{d-1} = e^{-g(d)} \in (1 - O(d^{-1}) , 1)
    \end{equation}
    
    \item Given $\vv\in\calV$, $|S_+^{*(0)}(\vv)|$ is a binomial random variable, with each Bernoulli trial (ranging over $r\in[m]$) having success probability $p_1 (1-p_2)^{d-1}$. Therefore, $\mathbb{E}\left[ |S_+^{*(0)}(\vv)| \right] = m p_1 (1-p_2)^{d-1} = \Theta\left(\frac{1}{\sqrt{\ln(d)}}\right) d^{c_0}$. 
    
    Now recall the Chernoff bound of binomial random variables. Let $\{X_n\}_{n=1}^m$ be an iid sequence of Bernoulli random variable with success rate $p$, and $S_n = \sum_{n=1}^m X_n$. Then for any $\delta \in (0,1)$,
    \begin{equation}
    \begin{aligned}
        & \mathbb{P}[S_n \ge (1+\delta) mp] \le \exp\left(- \frac{\delta^2 mp}{3} \right) \\
        & \mathbb{P}[S_n \le (1-\delta) mp] \le \exp\left(- \frac{\delta^2 mp}{2} \right)
    \end{aligned}
    \end{equation}
    It follows that, for each $\vv\in\calV$, $|S_+^{*(0)}(\vv)| = \Theta\left(\frac{1}{\sqrt{\ln(d)}}\right) d^{c_0}$ with probability at least $1- \exp(-\Omega(\ln^{-1/2}(d)) d^{c_0})$. Taking union bound over all possible  $\vv\in\calD$, the random event still holds with probability at least $1- \exp(-\Omega(\ln^{-1/2}(d)) d^{c_0} + \calO(\ln(d))) \ge 1- \exp(-\Omega(d^{0.5 c_0})) $ (in sufficiently high dimension).
    
\end{enumerate}

The proof for $S_+^{(0)}(\vv)$ proceeds in virtually the same way, so we omit the calculations here.

To show the second point, in particular $ \left\vert \frac{|S_+^{*(0)}(\vv)|}{|S_+^{(0)}(\vv')|} - 1 \right\vert \le O\left( \frac{1}{\ln^5(d)}\right)$, we need to be a bit more careful in our bounds of the relevant sets. In particular, we need to directly use the CDF of gaussian random variables:
\begin{equation}
\begin{aligned}
& \Bigg\vert \mathbb{P}\left[ \langle \vw_{+,r}^{(0)}, \vv \rangle \ge \sigma_0 \sqrt{4 + 2c_0} \sqrt{\ln(d) + \frac{1}{\ln^5(d)}} \right](1 \pm O(d^{-1})) \\
& - \mathbb{P}\left[ \langle \vw_{+,r}^{(0)}, \vv' \rangle \ge \sigma_0 \sqrt{4 + c_0} \sqrt{\ln(d) - \frac{1}{\ln^5(d)}} \right] \Bigg\vert \\
\le  & \frac{1}{2\sqrt{2\pi}} \int^{\sqrt{4 + 2c_0} \sqrt{\ln(d) + \frac{1}{\ln^5(d)}}}_{\sqrt{4 + 2c_0} \sqrt{\ln(d) - \frac{1}{\ln^5(d)}}} e^{-\epsilon^2/2} d\epsilon + O\left(\frac{1}{d^{3+c_0}\sqrt{\ln(d)}}\right)\\
\le & \frac{1}{2\sqrt{2\pi}} d^{-2-c_0}e^{(2+c_0)/\ln^5(d)} \sqrt{4 + 2c_0}\left( \sqrt{\ln(d) + \frac{1}{\ln^5(d)}} - \sqrt{\ln(d) - \frac{1}{\ln^5(d)}}\right) + O\left(\frac{1}{d^{3+c_0}\sqrt{\ln(d)}}\right)\\
= & \frac{1}{2\sqrt{2\pi}} d^{-2-c_0}e^{(2+c_0)/\ln^5(d)} \sqrt{4 + 2c_0} \frac{\frac{2}{\ln^5(d)}}{\sqrt{\ln(d) + \frac{1}{\ln^5(d)}} + \sqrt{\ln(d) - \frac{1}{\ln^5(d)}}} + O\left(\frac{1}{d^{3+c_0}\sqrt{\ln(d)}}\right)
\end{aligned}
\end{equation}
The expected difference in number between the two sets is just the above expression multiplied by $m = \Theta(d^{2 + 2c_0})$, and with probability at least $1 - \exp(-\Omega(d^{-c_0/4}))$, the difference term satisfies
\begin{equation}
\begin{aligned}
    & \frac{1}{2\sqrt{2\pi}} (1\pm d^{-c_0/2}) \Theta(d^{c_0}) e^{(2+c_0)/\ln^5(d)} \sqrt{4 + 2c_0} \frac{\frac{2}{\ln^5(d)}}{\sqrt{\ln(d) + \frac{1}{\ln^5(d)}} + \sqrt{\ln(d) - \frac{1}{\ln^5(d)}}} \pm  O\left(\frac{d^{2+2c_0}}{d^{3+c_0}\sqrt{\ln(d)}}\right) \\
    \in & \Theta\left(\frac{1}{\sqrt{\ln(d)}}\right) d^{c_0} \times \frac{1}{\ln^5(d)}
\end{aligned}
\end{equation}
By further noting from before that $|S_+^{(0)}(\vv)| = \Theta\left(\frac{1}{\sqrt{\ln(d)}}\right) d^{c_0}$, $ \left\vert \frac{|S_+^{*(0)}(\vv)|}{|S_+^{(0)}(\vv')|} - 1 \right\vert \le O\left( \frac{1}{\ln^5(d)}\right)$ follows. The proof of $\left\vert \frac{|S_+^{*(0)}(\vv)|}{|S_+^{*(0)}(\vv')|} - 1 \right\vert\le O\left( \frac{1}{\ln^5(d)}\right)$ follows a very similar argument, so we omit the calculations here.

Now, as for the set $S_{reg}^{(0)}$, we know for any $r\in[m]$ and $\vv_i\in\calD$, 
\begin{equation}
    \mathbb{P}\left[ \langle \vw_{+,r}^{(0)}, \vv_i \rangle \ge \sigma_0 \sqrt{10} \sqrt{\ln(d)} \right] \le O\left(\frac{1}{\sqrt{\ln(d)}}\right)d^{-5}.
\end{equation}
Taking the union bound over $r$ and $i$ yields
\begin{equation}
    \mathbb{P}\left[ \exists r \text{ and } i \text{ s.t.} \langle \vw_{+,r}^{(0)}, \vv_i \rangle \ge \sigma_0 \sqrt{10} \sqrt{\ln(d)} \right] \le md O\left(\frac{1}{\sqrt{\ln(d)}}\right)d^{-5} < d^{-2}.
\end{equation}

Finally, to show $\left\vert \calU_{+,r}^{(0)} \right\vert \le O(1)$ holds for every $(+,r)$, we just need to note that for any arbitrary $(+,r)$ neuron, the probability of $\left\vert \calU_{+,r}^{(0)} \right\vert > 4$ is no greater than
\begin{equation}
\begin{aligned}
p_2^{4}\binom{d}{4} \le O\left(\frac{1}{\ln^{2}d}\right) d^{-8-4c_0} \times d^4 \le  O\left(\frac{1}{\ln^{2}d}\right) d^{-4-4c_0}
\end{aligned}
\end{equation}
Taking union bound over all $m \le O\left(d^{2+2c_0}\right)$ neurons yields the desired result.

\end{proof}

\newpage
\section{Coarse-grained SGD Phase I: (Almost) Constant Loss, Neurons Diversify}
\label{section: appendix, phase I coarse}
\begin{definition}
We define $T_0$ to be the first time which there exists some sample $n$ such that 
\begin{equation}
    F_c^{(T_0)}(\mX_n^{(T_0)}) \ge d^{-1}
\end{equation}
Without loss of generality assume $c = +$. Define phase I to be the time $t \in [0, T_0)$.
\end{definition}

\subsection{Main results}
\begin{theorem}[Phase 1 SGD update properties]
\label{prop: phase 1 sgd induction}

The following properties hold with probability at least $ 1-O\left( \frac{mNPk_+t}{\poly(d)}\right) - O(e^{-\Omega(\ln^2(d))})$ for every $t \in [0, T_0)$.
\begin{enumerate}
    \item (On-diagonal common-feature neuron growth) For every $(+,r), (+,r') \in S_{+}^{*(0)}(\vv_+)$,
    \begin{equation}
        \vw_{+,r}^{(t)} - \vw_{+,r}^{(0)} = \vw_{+,r'}^{(t)} - \vw_{+,r'}^{(0)}
    \end{equation}

    Moreover,
    \begin{equation}
    \begin{aligned}
    \Delta \vw_{+,r}^{(t)} 
    = & \eta \Bigg(  \left(\frac{1}{2} \pm \psi_1\right) \sqrt{1 \pm \iota}\left(1 \pm s^{*-1/3}\right) \pm O\left(\frac{1}{\ln^{10}(d)}\right)\Bigg) \frac{s^*}{2P} \vv_{+}  + \Delta \vzeta^{(t)}_{+,r}
    \end{aligned}
    \end{equation}
    where $\Delta \vzeta^{(t)}_{+,r} \sim \calN(\vzero, \sigma_{\Delta \zeta_{+,r}}^{(t)2} \mI)$,  $\sigma_{\Delta\zeta_{+,r}}^{(t)} = \eta \sigma_{\zeta}\left(  \left(\frac{1}{2} \pm \psi_1\right) \sqrt{1 \pm s^{*-1/3}} \right) \frac{\sqrt{s^*}}{P\sqrt{2N}}$, and $\vert \psi_1 \vert \le O(d^{-1})$.

    Furthermore, every $(+,r) \in S_{+}^{*(0)}(\vv_+)$ activates on $\vv_+$-dominated patches at time $t$.
    
    \item (On-diagonal finegrained-feature neuron growth) For every possible choice of $c$ and every $(+,r), (+,r') \in S_{+}^{*(0)}(\vv_{+,c})$, 
    \begin{equation}
        \vw_{+,r}^{(t)} - \vw_{+,r}^{(0)} = \vw_{+,r'}^{(t)} - \vw_{+,r'}^{(0)}
    \end{equation}

    Moreover,
    \begin{equation}
    \begin{aligned}
    \Delta \vw_{+,r}^{(t)} 
    = & \eta \Bigg(  \left(\frac{1}{2} \pm \psi_1\right) \sqrt{1 \pm \iota}\left(1 \pm s^{*-1/3}\right) \pm O\left(\frac{1}{\ln^{10}(d)}\right)\Bigg)  \frac{s^*}{2 k_+ P} \vv_{+,c}  + \Delta\vzeta^{(t)}_{+,r}
    \end{aligned}
    \end{equation}
    where $\vzeta^{(t)}_{+,r} \sim \calN(\vzero, \sigma_{\Delta\zeta_{+,r}}^{(t)2} \mI)$, and $\sigma_{\Delta\zeta_{+,r}}^{(t)} = \eta\sigma_{\zeta}\left(  \left(\frac{1}{2} \pm \psi_1\right) \sqrt{1 \pm s^{*-1/3}} \right) \frac{\sqrt{s^*}}{P\sqrt{2Nk_+}}$.

    Furthermore, every $(+,r) \in S_{+}^{*(0)}(\vv_{+,c})$ activates on $\vv_+$-dominated patches at time $t$.
    
    \item The above results also hold with the ``$+$'' and ``$-$'' signs flipped.
\end{enumerate}
\end{theorem}
\begin{proof} 

The SGD update rule produces the following update:
\begin{align}
     \vw_{+,r}^{(t+1)}
    = & \vw_{+,r}^{(t)} + \eta \frac{1}{NP} \times\\
    & \sum_{n=1}^N \Bigg( 
    \mathbbm{1}\{y_n = +\}[1-\text{logit}_+^{(t)}(\mX_n^{(t)})]\sum_{p\in[P]} \sigma'(\langle \vw_{+,r}^{(t)}, \vx_{n,p}^{(t)} \rangle + b_{+,r}^{(t)}) \vx_{n,p}^{(t)} \label{expression: common feat, on-diag}\\
    & + \mathbbm{1}\{y_n = -\} [-\text{logit}_+^{(t)}(\mX_n^{(t)}) ]\sum_{p\in[P]} \sigma'(\langle \vw_{+,r}^{(t)}, \vx_{n,p}^{(t)}  \rangle + b^{(t)}_{+,r})  \vx_{n,p}^{(t)} \Bigg) \label{expression: common feat, off-diag}
\end{align}
In particular,
\begin{equation}
\begin{aligned}
\ref{expression: common feat, on-diag} 
= & \sum_{n=1}^N \mathbbm{1} \{ y_n = +\}\left(\frac{1}{2} \pm \psi_1\right) \times \\
& \Bigg\{\mathbbm{1}\{|\calP(\mX_n^{(t)}; \vv_{+})|>0\} \Bigg[\sum_{p\in\calP(\mX_n^{(t)}; \vv_{+})} \sigma'(\langle \vw_{+,r}^{(t)}, \alpha_{n,p}^{(t)} \vv_{+} + \vzeta^{(t)}_{n,p} \rangle + b_{+,r}^{(t)}) \left(\alpha_{n,p}^{(t)} \vv_{+} + \vzeta_{n,p}^{(t)} \right) \\
& + \sum_{p\notin\calP(\mX_n^{(t)}; \vv_{+})} \sigma'(\langle \vw_{+,r}^{(t)}, \vx_{n,p}^{(t)} \rangle + b_{+,r}^{(t)})  \vx_{n,p}^{(t)} \Bigg] \\
& + \mathbbm{1}\{|\calP(\mX_n^{(t)}; \vv_{+})|=0\} \sum_{p\in[P]} \sigma'(\langle \vw_{+,r}^{(t)}, \vx_{n,p}^{(t)} \rangle + b_{+,r}^{(t)}) \vx_{n,p}^{(t)} \Bigg\} \\
= & \sum_{n=1}^N \mathbbm{1} \{ y_n = +\}\left(\frac{1}{2} \pm \psi_1\right) \times \\
& \Bigg\{\mathbbm{1}\{|\calP(\mX_n^{(t)}; \vv_{+})|>0\} \Bigg[\sum_{p\in\calP(\mX_n^{(t)}; \vv_{+})} \mathbbm{1}\left\{ \langle \vw_{+,r}^{(t)}, \alpha_{n,p}^{(t)} \vv_{+} + \vzeta_{n,p}^{(t)} \rangle \ge b_{+,r}^{(t)} \right\}  \left(\alpha_{n,p}^{(t)}\vv_{+} + \vzeta_{n,p}^{(t)}\right) \\
& + \sum_{p\notin\calP(\mX_n^{(t)}; \vv_{+})} \mathbbm{1}\left\{\langle \vw_{+,r}^{(t)}, \vx_{n,p}^{(t)} \rangle \ge  b_{+,r}^{(t)}\right\} \vx_{n,p}^{(t)} \Bigg] \\
& + \mathbbm{1}\{|\calP(\mX_n^{(t)}; \vv_{+})|=0\} \sum_{p\in[P]} \mathbbm{1}\left\{\langle \vw_{+,r}^{(t)}, \vx_{n,p}^{(t)} \rangle \ge  b_{+,r}^{(t)}\right\} \vx_{n,p}^{(t)} \Bigg\}
\end{aligned}
\end{equation}

The rest of the proof proceeds by induction (in Phase 1).

First, recall that we set $b_{c,r}^{(0)} = - \sqrt{4 + 2c_0}\sqrt{\ln(d)}$, and $\Delta b_{c,r}^{(t)} = -\frac{\| \Delta \vw_{c,r}^{(t)} \|_2}{\ln^5(d)}$ for all $t$ in phase 1, and for any $+$-class sample $\mX_n$ with $p\in\calP(\mX_n^{(t)}; \vv_{+})$, $\alpha_{n,p}^{(t)} \in \sqrt{1\pm \iota}$ by our data assumption.

\textcolor{blue}{\textbf{Base case $t = 0$.}}

\textcolor{brown}{\textit{1. (On-diagonal common-feature neuron growth)}}

The base case for the neuron expression of point 1. is trivially true.

We show that the neurons $(+,r) \in S_{+}^{*(0)}(\vv_+)$ only activate on $\vv_+$-dominated patches at time $t=0$.

With probability at least $1-O\left(\frac{mNP}{\poly(d)}\right)$, by Lemma \ref{lemma: independent gaussian vector inner product concentration}, we have for all possible choices of $r,n,p$:
\begin{equation}
    \left\vert \langle \vw_{+,r}^{(0)}, \vzeta_{n,p}^{(0)} \rangle \right\vert \le O(\sigma_0\sigma_{\zeta} \sqrt{d\ln(d)}) \le O\left(\frac{\sigma_0}{\ln^{9}(d)}\right)
\end{equation}
It follows that
\begin{equation}
\begin{aligned}
    \langle \vw_{+,r}^{(0)}, \alpha_{n,p}^{(0)} \vv_{+} + \vzeta_{n,p}^{(0)} \rangle 
    = & \sigma_0\left\{\sqrt{1 \pm \iota}\times \left(\sqrt{4+2 c_0} \sqrt{\ln(d) + 1/\ln^5(d)}, \sqrt{10}\sqrt{\ln(d)} \right) \pm \frac{1}{\ln^{9}(d)} \right\} \\
    = & \sigma_0\left\{\left(\sqrt{1 - \iota}\sqrt{4 + 2c_0}\sqrt{\ln(d) + 1/\ln^5(d)}, \sqrt{1 + \iota}\sqrt{10}\sqrt{\ln(d)} \right) \pm \frac{1}{\ln^{9}(d)} \right\} 
\end{aligned}
\end{equation}

Employing the basic identity $a - b = \frac{a^2 - b^2}{a + b}$, we have the lower bound
\begin{equation}
\begin{aligned}
    &\sigma_0^{-1}\left(\langle \vw_{+,r}^{(0)}, \alpha_{n,p}^{(0)} \vv_{+} + \vzeta_{n,p}^{(0)} \rangle + b_{+,r}^{(0)} \right) \\
    & \ge \sqrt{(1 - \iota)(4 + 2c_0)(\ln(d) + 1/\ln^5(d))} -  \sqrt{(4 + 2c_0)\ln(d)} - O\left(\frac{1}{\ln^{9}(d)}\right) \\
    & = \frac{(1 - \iota)(4 + 2c_0)(\ln(d) + 1/\ln^5(d)) - (4 + 2c_0)\ln(d)}{\sqrt{(1 - \iota)(4 + 2c_0)(\ln(d) + 1/\ln^5(d))} +  \sqrt{(4 + 2c_0)\ln(d)}} - O\left(\frac{1}{\ln^{9}(d)}\right) \\
    & = \frac{(4+2c_0)(-\iota\ln(d) + (1-\iota)/\ln^5(d))}{\sqrt{(1 - \iota)(4 + 2c_0)(\ln(d) + 1/\ln^5(d))} +  \sqrt{(4 + 2c_0)\ln(d)}} - O\left(\frac{1}{\ln^{9}(d)}\right) \\
    & > 0
\end{aligned}
\end{equation}
The last inequality holds since $\iota \le \frac{1}{\polyln(d)}$ and $d$ is sufficiently large such that $\frac{1}{\ln^{9}(d)}$ does not drive the positive term down past $0$.

Therefore, the neurons in $S^{*(0)}_+(\vv_+)$ indeed activate on the $\vv_+$-dominated patches at $t=0$.

The rest of the patches $\vx_{n,p}^{(0)}$ is either a feature patch (not dominated by $\vv_+$) or a noise patch. 
By definition, $(+,r)\in S^{*(0)}_+(\vv_+)\implies (+,r)\in S^{(0)}_+(\vv_+)$. Therefore, by Theorem \ref{thm: sgd, universal nonact properties}, with probability at least $1 - O\left( \frac{mk_+NP}{\poly(d)} \right)$, at time $t=0$, the $(+,r)\in S^{*(0)}_+(\vv_+)$ neurons we are considering cannot activate on any feature patch dominated by $\vv\perp\vv_+$, nor on any noise patches.

It follows that the expression \ref{expression: common feat, on-diag} at time $t = 0$ is as follows:
\begin{equation}
\begin{aligned}
\ref{expression: common feat, on-diag} 
= & \sum_{n=1}^N \mathbbm{1} \{ y_n = +\}\left(\frac{1}{2} \pm \psi_1\right) \times \\
& \Bigg\{\mathbbm{1}\{|\calP(\mX_n^{(0)}; \vv_{+})|>0\} \Bigg[\sum_{p\in\calP(\mX_n^{(0)}; \vv_{+})}\left(\sqrt{1 \pm \iota} \vv_{+} + \vzeta_{n,p}^{(0)}\right) + \sum_{p\notin\calP(\mX_n^{(0)}; \vv_{+})} 0 \Bigg] \\
& + \mathbbm{1}\{|\calP(\mX_n^{(0)}; \vv_{+})|=0\} \sum_{p\in[P]} 0 \Bigg\} \\
= & \left(\frac{1}{2} \pm \psi_1\right) \sum_{n=1}^N \mathbbm{1}\{ y_n = +, |\calP(\mX_n^{(0)}; \vv_{+})| > 0\} \sum_{p\in\calP(\mX_n^{(0)}; \vv_{+})}\left(\sqrt{1 \pm \iota}\vv_{+} + \vzeta_{n,p}^{(0)}\right) \\
= & \left(\frac{1}{2} \pm \psi_1\right) \times \\
& \left\vert \left\{(n,p)\in[N]\times[P]: y_n = +, |\calP(\mX_n^{(0)}; \vv_{+})| > 0, p \in \calP(\mX_n^{(0)}; \vv_{+}) \right\} \right\vert \left(\sqrt{1 \pm \iota} \vv_{+} \right) \\
& +  \sum_{n=1}^N \sum_{p\in\calP(\mX_n^{(0)}; \vv_{+})}\{ y_n = +\}\left(\frac{1}{2} \pm \psi_1\right) \vzeta_{n,p}^{(0)}
\end{aligned}
\end{equation}

On average,
\begin{equation}
\begin{aligned}
    & \mathbb{E}\left[ \left\vert \left\{(n,p)\in[N]\times[P]: y_n = +, |\calP(\mX_n^{(0)}; \vv_{+})| > 0, p \in \calP(\mX_n^{(0)}; \vv_{+}) \right\} \right\vert \right] \\
    = & \frac{s^*}{P} \times P \times \frac{N}{2} = \frac{s^* N}{2}
\end{aligned}
\end{equation}
Furthermore, with our parameter choices, and by concentration of binomial random variables, with probability at least $1 - e^{-\Omega(\polyln(d))}$,
\begin{equation}
     \left\vert \left\{(n,p)\in[N]\times[P]: y_n = +, |\calP(\mX_n^{(0)}; \vv_{+})| > 0, p \in \calP(\mX_n^{(0)}; \vv_{+}) \right\} \right\vert =  \frac{s^* N}{2} \left(1 \pm s^{*-1/3}\right)
\end{equation}
must be true. 

It follows that
\begin{equation}
\begin{aligned}
    \ref{expression: common feat, on-diag} 
    = & \left(\frac{1}{2} \pm \psi_1\right) \times \frac{s^* N}{2} \left(1 \pm s^{*-1/2}\right) \times  \left(\sqrt{1 \pm \iota} \vv_{+} \right)  \\
    & + \sum_{n=1}^N \sum_{p\in\calP(\mX_n^{(0)}; \vv_{+})}\{ y_n = +\}\left(\frac{1}{2} \pm \psi_1\right) \vzeta_{n,p}^{(0)}
\end{aligned}
\end{equation}

The other component expression \ref{expression: common feat, off-diag} is zero with probability at least $1 - O\left( \frac{mk_+NP}{\poly(d)} \right)$ by Theorem \ref{thm: sgd, universal nonact properties}.

By noting that
\begin{equation}
\begin{aligned}
    \text{Var}\left(\Delta \vzeta_{+,r}^{(0)}\right) 
    = & \text{Var}\left(\frac{\eta}{NP}\sum_{n=1}^N \sum_{p\in\calP(\mX_n^{(0)}; \vv_{+})}\{ y_n = +\}\left(\frac{1}{2} \pm \psi_1\right) \vzeta_{n,p}^{(0)}\right) \\
    = & \eta^2 \left(\frac{1}{2} \pm \psi_1\right)^2 \frac{s^* }{2NP^2} \left(1 \pm s^{*-1/3}\right) \sigma_{\zeta}^2,
\end{aligned}
\end{equation}
and
\begin{equation}
    \E\left[\Delta \vzeta_{+,r}^{(0)}\right]
    = \E\left[\frac{\eta}{NP}\sum_{n=1}^N \sum_{p\in\calP(\mX_n^{(0)}; \vv_{+})}\{ y_n = +\}\left(\frac{1}{2} \pm \psi_1\right) \vzeta_{n,p}^{(0)}\right] = \vzero,
\end{equation}
we finish the proof of the base case for point 1.

\textcolor{brown}{\textit{2. (On-diagonal finegrained-feature neuron growth)}}

The proof of the base case of point 2. is virtually identical to point 1, so we omit the computations here.

\textcolor{blue}{\textbf{Inductive step}}: We condition on the high probability events of the induction hypothesis for $t\in [0, T]$ (with $T < T_0$ of course), and prove the statements for $t = T+1$.

\textcolor{brown}{\textit{1. (On-diagonal common-feature neuron growth)}}

By the induction hypothesis, up to time $t=T$, with probability at least $1 - O\left( \frac{mk_+NPT}{\poly(d)} \right)$, for all $(+,r)\in S_+^{*(T)}(\vv_+)$,
\begin{equation}
\begin{aligned}
\Delta \vw_{+,r}^{(t)} 
= & \eta \Bigg(  \left(\frac{1}{2} \pm \psi_1\right) \sqrt{1 \pm \iota}\left(1 \pm s^{*-1/3}\right) \Bigg) \frac{s^*}{2P} \vv_{+}  + \Delta \vzeta^{(t)}_{+,r}
\end{aligned}
\end{equation}
where $\Delta \vzeta^{(t)}_{+,r} \sim \calN(\vzero, \sigma_{\Delta \zeta}^{(t)2} \mI)$,  $\sigma_{\Delta\zeta}^{(t)} = \eta \sigma_{\zeta}\left(  \left(\frac{1}{2} \pm \psi_1\right) \sqrt{1 \pm s^{*-1/3}} \right) \frac{\sqrt{s^*}}{P\sqrt{2N}}$. 

\textbf{Expression of $\vw_{+,r}^{(T+1)}$}.

Conditioning on the high-probability event of the induction hypothesis, at time $t = T+1$,
\begin{equation}
\begin{aligned}
\vw_{+,r}^{(T+1)} 
= & \vw_{+,r}^{(0)} + \sum_{\tau = 0}^T \Delta \vw_{+,r}^{(\tau)}  \\
= & \eta T \Bigg(  \left(\frac{1}{2} \pm \psi_1\right) \sqrt{1 \pm \iota}\left(1 \pm s^{*-1/3}\right) \Bigg) \frac{s^*}{2P} \vv_{+}  +  \vzeta^{(t)}_{+,r}
\end{aligned}
\end{equation}
where $ \vzeta^{(t)}_{+,r} \sim \calN(\vzero, \sigma_{\zeta}^{(t)2} \mI)$,  $\sigma_{\zeta}^{(t)} = \eta \sigma_{\zeta}\sqrt{T}\left(  \left(\frac{1}{2} \pm \psi_1\right) \sqrt{1 \pm s^{*-1/3}} \right) \frac{\sqrt{s^*}}{P\sqrt{2N}}$. 

Let us compute $\Delta \vw_{+,r}^{(T+1)} $. 

We first want to show that $\vw_{+,r}^{(T+1)}$ activates on $\vv_+$-dominated patches $\vx_{n,p}^{(T+1)} = \sqrt{1\pm\iota}\vv_+ + \vzeta_{n,p}^{(T+1)}$. We need to show that the following expression is above 0:
\begin{equation}
\begin{aligned}
    & \langle \vw_{+,r}^{(T+1)}, \vx_{n,p}^{(T+1)} \rangle + b_{+,r}^{(T+1)} \\ 
    = &  \langle \vw_{+,r}^{(0)}, \sqrt{1 \pm \iota}\vv_+ + \vzeta_{n,p}^{(T+1)} \rangle + b_{+,r}^{(0)} \\
    & + \Bigg\langle \eta T \Bigg(  \left(\frac{1}{2} \pm \psi_1\right) \sqrt{1 \pm \iota}\left(1 \pm s^{*-1/3}\right) \pm O\left(\frac{1}{\ln^{10}(d)}\right)\Bigg) \frac{s^*}{2P} \vv_{+}  + \vzeta^{(T+1)}_{+,r}, \sqrt{1\pm\iota}\vv_+ + \vzeta_{n,p}^{(T+1)} \Bigg\rangle \\
    & + \sum_{\tau=0}^{T} \Delta b_{+,r}^{(\tau)}
\end{aligned}
\end{equation}

Let us treat the three terms (on three lines) separately.

First, following virtually the same argument as in the base case, the following lower bound holds with probability at least $1 - O\left( \frac{mNP}{\poly(d)} \right)$ for all $n,p$ and $(+,r) \in S_+^{*(T)}(\vv_+)$:
\begin{equation}
\begin{aligned}
    & \langle \vw_{+,r}^{(0)}, \sqrt{1\pm\iota} \vv_{+} + \vzeta_{n,p}^{(T+1)} \rangle + b_{+,r}^{(0)} \\
    \ge & \sigma_0\left\{\sqrt{(1 - \iota)(4 + 2c_0)(\ln(d) + 1/\ln^5(d))} -  \sqrt{(4 + 2c_0)\ln(d)} - O\left(\frac{1}{\ln^{9}(d)}\right) \right\}\\
    > & 0
\end{aligned}
\end{equation}

Now consider the second term.

We know, with probability at least $1 - e^{-\Omega(d)}$, for all $n$ and $p$,
\begin{equation}
\begin{aligned}
    \left\vert \langle \vzeta_{n,p}^{(T+1)}, \vv_{+} \rangle \right\vert \le O\left(\frac{1}{\ln^{10}(d)} \right),
\end{aligned}
\end{equation}
therefore,
\begin{equation}
\begin{aligned}
    & \left\vert \langle \eta T \Bigg(  \left(\frac{1}{2} \pm \psi_1\right) \sqrt{1 \pm \iota}\left(1 \pm s^{*-1/3}\right) \pm O\left(\frac{1}{\ln^{10}(d)}\right)\Bigg) \frac{s^*}{2P} \vv_{+} , \vzeta_{n,p}^{(T+1)} \rangle \right\vert \\
    \le & \eta T \frac{s^*}{2P}O\left(\frac{1}{\ln^{10}(d)}\right).
\end{aligned}
\end{equation}

Moreover, with probability at least $1-e^{-\Omega(d)}$,
\begin{equation}
\begin{aligned}
    \left\vert \langle \vzeta_{+,r}^{(T+1)}, \vv_+ \rangle \right\vert \le \eta\sqrt{T}\frac{\sqrt{s^*}}{P\sqrt{2N}} \times O\left(\frac{1}{\ln^{10}(d)}\right)
\end{aligned}
\end{equation}
and with probability at least $1 - e^{-\Omega(d)}$,
\begin{equation}
    \left\vert \langle \vzeta_{+,r}^{(T)}, \vzeta_{n,p}^{(T+1)} \rangle \right\vert \le O\left( \sigma_{\zeta} \sigma_{\zeta}^{(T)} d\right) \le O\left(\eta\sqrt{T} \frac{\sqrt{s^*}}{P\sqrt{2N}}\frac{1}{\ln^{20}(d) d} d \right) \le \eta\sqrt{T} \frac{\sqrt{s^*}}{P\sqrt{2N}} \frac{1}{\ln^{19}(d)}
\end{equation}
therefore
\begin{equation}
\begin{aligned}
\langle \eta T \vzeta^{(T+1)}_{+,r}, \sqrt{1\pm\iota}\vv_+ + \vzeta_{n,p}^{(T+1)} \rangle 
\le \eta \sqrt{T} \frac{\sqrt{s^*}}{P\sqrt{2N}} O \left(\frac{1}{\ln^{10}(d)} \right).
\end{aligned}
\end{equation}

It follows that with probability at least $1 - O(e^{-\Omega(d)})$,
\begin{equation}
\begin{aligned}
    &\Bigg\langle \eta T \Bigg(  \left(\frac{1}{2} \pm \psi_1\right) \sqrt{1 \pm \iota}\left(1 \pm s^{*-1/3}\right) \Bigg) \frac{s^*}{2P} \vv_{+}  + \vzeta^{(T+1)}_{+,r}, \sqrt{1\pm\iota}\vv_+ + \vzeta_{n,p}^{(T+1)} \Bigg\rangle \\
    = & \langle \eta T \Bigg(  \left(\frac{1}{2} \pm \psi_1\right) \sqrt{1 \pm \iota}\left(1 \pm s^{*-1/3}\right) \Bigg) \frac{s^*}{2P} \vv_{+} , \sqrt{1\pm\iota}\vv_+ \rangle \\
    & + \langle \eta T \Bigg(  \left(\frac{1}{2} \pm \psi_1\right) \sqrt{1 \pm \iota}\left(1 \pm s^{*-1/3}\right) \Bigg) \frac{s^*}{2P} \vv_{+} , \vzeta_{n,p}^{(T+1)} \rangle \\
    & + \langle \eta \vzeta^{(T+1)}_{+,r}, \sqrt{1\pm\iota}\vv_+ + \vzeta_{n,p}^{(T+1)} \rangle \\
    \ge & \eta T \left(\frac{1}{2} - \psi_1^{(T+1)}\right) (1 - \iota)\left(1 - s^{*-1/3}\right) \frac{s^*}{2P} - \eta \sqrt{T} \frac{\sqrt{s^*}}{P\sqrt{2N}} O \left(\frac{1}{\ln^{10}(d)} \right).
\end{aligned}
\end{equation}

Now we compute the third term. By the induction hypothesis,

\begin{equation}
\begin{aligned}
& \sum_{t=0}^{T} \Delta b_{+,r}^{(t)} \\
= & \sum_{t=0}^{T} \frac{\|\Delta \vw_{+,r}^{(t)}\|_2}{\ln^5(d)} \\
= & \sum_{t=0}^{T}\frac{1}{\ln^5(d)}\left\| \eta   \left(\frac{1}{2} \pm \psi_1\right) \sqrt{1 \pm \iota}\left(1 \pm s^{*-1/3}\right) \frac{s^*}{2P} \vv_{+}  + \Delta\vzeta^{(t)}_{+,r}\right\|_2 \\
\le & \sum_{t=0}^{T} \frac{1}{\ln^5(d)} \eta \left(\frac{1}{2} + \psi_1\right) \sqrt{1 + \iota}\left(1 + s^{*-1/3}\right)  \frac{s^*}{2P} \left\|  \vv_{+}\right\|_2 + \sum_{t=0}^{T} \frac{1}{\ln^5(d)}\left\| \Delta \vzeta^{(t)}_{+,r} \right\|_2 \\ 
= &  \frac{1}{\ln^5(d)} \eta T \left(\frac{1}{2} + \psi_1\right) \sqrt{1 + \iota}\left(1 + s^{*-1/3}\right)  \frac{s^*}{2P} + \sum_{t=0}^{T} \frac{1}{\ln^5(d)}\left\| \Delta \vzeta^{(t)}_{+,r} \right\|_2
\end{aligned}
\end{equation}

With probability at least $1 - O\left(\frac{mT}{\poly(d)}\right)$, for all $t\in[0,T]$ and $r$ in consideration,
\begin{equation}
    \left\| \Delta \vzeta^{(t)}_{+,r} \right\|_2 \le \eta \frac{\sqrt{s^*}}{P\sqrt{2N}} O \left(\frac{1}{\ln^{10}(d)} \right)
\end{equation}

Therefore, 
\begin{equation}
\begin{aligned}
& \sum_{t=0}^{T} \Delta b_{+,r}^{(t)} \\ 
\le &  \frac{1}{\ln^5(d)} \left( \eta T \left(\frac{1}{2} + \psi_1\right) \sqrt{1 + \iota}\left(1 + s^{*-1/3}\right)  \frac{s^*}{2P} + \eta T \frac{\sqrt{s^*}}{P\sqrt{2N}} O \left(\frac{1}{\ln^{10}(d)} \right) \right)
\end{aligned}
\end{equation}

Combining our calculations of the three terms from above, we find the following estimate:
\begin{equation}
\begin{aligned}
    & \langle \vw_{+,r}^{(T+1)}, \vx_{n,p}^{(T+1)} \rangle + b_{+,r}^{(T+1)} \\ 
    > & \; 0 \\
    & + \eta T \left(\frac{1}{2} - \psi_1\right) (1 - \iota)\left(1 - s^{*-1/3}\right) \frac{s^*}{2P} - \eta \sqrt{T} \frac{\sqrt{s^*}}{P\sqrt{2N}} O \left(\frac{1}{\ln^{10}(d)} \right) \\
    & -  \frac{1}{\ln^5(d)} \left( \eta T \left(\frac{1}{2} + \psi_1\right) \sqrt{1 + \iota}\left(1 + s^{*-1/3}\right)  \frac{s^*}{2P} + \eta T\frac{\sqrt{s^*}}{P\sqrt{2N}} O \left(\frac{1}{\ln^{10}(d)} \right) \right) \\
    > & \eta T \left( \left(\frac{1}{2} - \psi_1\right) (1 - \iota)\left(1 - s^{*-1/3}\right) - O \left(\frac{1}{\ln^{4}(d)} \right)  \right)\frac{s^*}{2P} \\
    > & \; 0
\end{aligned}
\end{equation}

On the other hand, by Theorem \ref{thm: sgd, universal nonact properties}, with probability at least $1 - O\left( \frac{mk_+NPT}{\poly(d)} \right)$, none of the $(+,r)\in S_+^{*(T)}(\vv_+)$ can activate on $\vx_{n,p}^{(T+1)}$ that are feature-patches dominated by $\vv\perp\vv_+$ or noise patches. 

Combining the above observations, with probability at least $1 - O\left( \frac{mk_+NP(T+1)}{\poly(d)} \right)$, the update expressions up to time $t=T+1$ can be written as follows:
\begin{equation}
\begin{aligned}
&\Delta \vw_{+,r}^{(t)} = \left(\frac{1}{2} \pm \psi_1\right) \\
& \times \Bigg\{ \left\vert \left\{(n,p)\in[N]\times[P]: y_n = +, |\calP(\mX_n^{(t)}; \vv_{+})| > 0, p \in \calP(\mX_n^{(t)}; \vv_{+}) \right\} \right\vert \left(\sqrt{1 \pm \iota} \vv_{+} \right) \\
& +  \sum_{n=1}^N \sum_{p\in\calP(\mX_n^{(0)}; \vv_{+})}\{ y_n = +\}\left(\frac{1}{2} \pm \psi_1\right) \vzeta_{n,p}^{(t)} \Bigg\}
\end{aligned}
\end{equation}

The rest of the derivations proceeds virtually the same as in the base case; we just need to rely on the concentration of binomial random variables to calculate
\begin{equation}
     \left\vert \left\{(n,p)\in[N]\times[P]: y_n = +, |\calP(\mX_n^{(0)}; \vv_{+})| > 0, p \in \calP(\mX_n^{(0)}; \vv_{+}) \right\} \right\vert =  \frac{s^* N}{2} \left(1 \pm s^{*-1/3}\right)
\end{equation}
which completes the proof of the expression of $\Delta \vw_{+,r}^{(t)}$.

Additionally, to show 
\begin{equation}
    \vw_{+,r}^{(T+1)} - \vw_{+,r}^{(0)} = \vw_{+,r'}^{(T+1)} - \vw_{+,r'}^{(0)}
\end{equation}
we just need to note that, by the above sequence of derivations, for every $(+,r) \in S_{+}^{*(0)}(\vv_+)$, these neurons receive exactly the same update at time $t=T+1$
\begin{equation}
    \sum_{n=1}^N \mathbbm{1} \{ y_n = +\} \mathbbm{1}\{|\calP(\mX_n^{(T+1)}; \vv_{+})|>0\} [1-\text{logit}_+^{(T+1)}(\mX_n^{(T+1)})] 
     \sum_{p\in\calP(\mX_n^{(T+1)}; \vv_{+})} \left(\alpha_{n,p}^{(T+1)} \vv_{+} + \vzeta_{n,p}^{(T+1)} \right). \\
\end{equation}

\textcolor{brown}{\textit{2. (On-diagonal finegrained-feature neuron growth)}}

For point 2, the proof strategy is almost identical, the only difference is that at every iteration, the expected number of patches in which subclass features appear in is
\begin{equation}
\begin{aligned}
     &\left\vert \left\{(n,p)\in[N]\times([P] - \calP(\mX_n^{(T)}); \vv_{+,c}): y_n = +, |\calP(\mX_n^{(T)}; \vv_{+,c})| > 0, p \in \calP(\mX_n^{(T)}; \vv_{+,c}) \right\} \right\vert \\
     = &  \frac{s^* N}{2k_+} \left(1 \pm s^{*-1/3}\right)
\end{aligned}
\end{equation}
which holds with probability at least $1 - e^{-\Omega(\ln^2(d))}$ for the relevant neurons.
\end{proof}

\begin{corollary}
$T_0 < O\left( \left(\eta \frac{s^*}{P} \right)^{-1}\right) \in \poly(d)$.
\end{corollary}
\begin{proof}
Follows from Theorem \ref{prop: phase 1 sgd induction}.
\end{proof}

\subsection{Lemmas}
\begin{lemma}
\label{lemma: phase 1, loss scale bound}
During the time $t \in [0, T_0)$, for any $\mX_n^{(t)}$,
\begin{equation}
    1 - \logit^{(t)}_+(\mX_n^{(t)}) = \frac{1}{2} \pm O(d^{-1})
\end{equation}
The same holds for $1 - \logit^{(t)}_-(\mX_n^{(t)})$.

Therefore, $\vert \psi_1 \vert \le O(d^{-1})$ for $t \in [0, T_0)$.
\end{lemma}
\begin{proof}
By definition of $T_0$, for any $t \in [0, T_0]$, we have $F_c^{(t)}(\mX_n^{(t)}) < d^{-1} + O\left(\eta \right)$ for all $n$, therefore, using Taylor approximation,
\begin{equation}
    1 - \logit^{(t)}_+(\mX_n^{(t)}) = \frac{\exp(F_{-}^{(t)}(\mX_n^{(t)}))}{\exp(F_{+}^{(t)}(\mX_n^{(t)})) + \exp(F_{-}^{(t)}(\mX_n^{(t)}))} < \frac{\exp(d^{-1})}{1 + 1} \le \frac{1}{2} + O(d^{-1})
\end{equation}
The lower bound can be proven due to convexity of the exponential:
\begin{equation}
    \frac{\exp(F_{-}^{(t)}(\mX_n^{(t)}))}{\exp(F_{+}^{(t)}(\mX_n^{(t)})) + \exp(F_{-}^{(t)}(\mX_n^{(t)}))} > \frac{1}{2}\exp(-d^{-1}) \ge \frac{1}{2} - \frac{1}{2d}
\end{equation}
\end{proof}

\newpage
\section{Coarse-grained SGD Phase II: Loss Convergence, Large Neuron Movement}
\label{section: appendix, phase II coarse}
Recall that the desired probability events in Phase I happens with probability at least $1 - o(1)$.

In phase II, common-feature neurons start gaining large movement and drive the training loss down to $o(1)$. We show that the desired probability events occur with probability at least $1-o(1)$. 

We study the case of $T_1 \le \poly(d)$, where $T_1$ denotes the time step at the end of training.

\subsection{Main results}

\begin{theorem}
With probability at least $1-O\left( \frac{mk_+NPT_1}{\poly(d)}\right)$, the following events take place:
\begin{enumerate}
    \item There exists time $T^* \in \poly(d)$ such that for any $t \in [T^*, \poly(d)]$, for any $n\in[N]$, the training loss $L(F; \mX_n^{(t)}, y_n) \in o(1)$.
    \item (Easy sample test accuracy is nearly perfect) Given an easy test sample $(\mX_{\text{easy}},y)$, for $y'\in \{+1, -1\} - \{y\}$, for $t \in [T^*, \poly(d)]$,
    \begin{equation}
        \mathbb{P}\left[F_y^{(t)}(\mX_{\text{easy}}) \le F_{y'}^{(t)}(\mX_{\text{easy}})\right] \le o(1).
    \end{equation}
    \item (Hard sample test accuracy is bad) However, for all $t\in[0,\poly(d)]$, given a hard test sample $(\mX_{\text{hard}},y)$, 
    \begin{equation}
        \mathbb{P}\left[F_y^{(t)}(\mX_{\text{hard}}) \le F_{y'}^{(t)}(\mX_{\text{hard}})\right] \ge \Omega(1).
    \end{equation}
\end{enumerate}

\end{theorem}
\begin{proof}
The training loss property follows from Lemma \ref{lemma: phase II, after T1,1} and Lemma \ref{lem: net response, final}. We can set $T^* = T_{1,1}$ or any time beyond it (and upper bounded by $\poly(d)$).

The test accuracy properties follow from Lemma \ref{lemma: coarse, hard sample mistake prob} and Lemma \ref{lemma: coarse, easy sample mistake prob}.

\end{proof}

\subsection{Lemmas}

\begin{lemma}[Phase II, Update Expressions]
\label{lemma: phase II, general}
For any $T_1 \in \poly(d)$, with probability at least $1 - O\left(\frac{mNPk_+t}{\poly(d)} \right)$, during $t \in [T_0, T_1]$, for any $(+,r) \in S_+^{*(0)}(\vv_+)$,
\begin{equation}
\begin{aligned}
& \Delta \vw_{+,r}^{(t)} \\
= & \eta \sum_{n=1}^N \mathbbm{1} \{ y_n = +\}\exp\left\{ - F_+^{(t)}(\mX_n^{(t)}) \right\}  \\
& \times   \frac{\exp(F_-^{(t)}(\mX_n^{(t)}))}{\exp\left(F_-^{(t)}(\mX_n^{(t)})- F_+^{(t)}(\mX_n^{(t)})\right) + 1 } (1\pm s^{*-1/3})\frac{s^*}{NP} \left(\sqrt{1\pm\iota}\vv_{+} + \vzeta_{n,p}^{(t)}\right),
\end{aligned}
\end{equation}
(where $c_n^t$ denotes the subclass index of sample $\mX_n^{(t)}$) and for any $(+,r) \in S_+^{*(0)}(\vv_{+,c})$,
\begin{equation}
\begin{aligned}
& \Delta \vw_{+,r}^{(t)} \\
= & \eta\exp\Bigg\{ - (1\pm s^{*-1/3})\sqrt{1\pm\iota}  \left(1 \pm O\left(\frac{1}{\ln^5(d)}\right)\right) s^*\left(A_{+,r^*}^{*(t)}\left\vert S^{*(0)}_+(\vv_+) \right\vert + A_{+,c,r^*}^{*(t)}\left\vert S^{*(0)}_+(\vv_{+,c}) \right\vert \right) \Bigg\}  \\
& \times  \sum_{n=1}^N \mathbbm{1} \{ y_n = (+,c)\} \frac{\exp(F_-^{(t)}(\mX_n^{(t)}))}{\exp\left(F_-^{(t)}(\mX_n^{(t)})- F_+^{(t)}(\mX_n^{(t)})\right) + 1 } (1\pm s^{*-1/3})\frac{s^*}{NP} \left(\sqrt{1\pm\iota}\vv_{+,c} + \vzeta_{n,p}^{(t)}\right),
\end{aligned}
\end{equation}

In fact, for any $\vv\in\{\vv_+\}\cup\{\vv_{+,c}\}_{c=1}^{k_+}$, every neuron in $S_+^{*(0)}(\vv)$ remain activated (on $\vv$-dominated patches) and receive exactly the same updates at every iteration as shown above.

For simpler exposition, for any $(+,r^*) \in S_+^{*(0)}(\vv_+)$, we write $A_{+,r^*}^{*(t)} \coloneqq \langle \vw_{+,r^*}^{(t)}, \vv_+ \rangle$; similarly for $A_{+,c,r^*}^{*(t)} \coloneqq \langle \vw_{+,r^*}, \vv_{+,c} \rangle$ for neurons $(+,r^*) \in S_+^{*(0)}(\vv_{+,c})$.

Moreover, on ``$+$''-class samples, the neural network response satisfies the estimate for every $(+,r^*) \in S_+^{*(0)}(\vv_+)$:
\begin{equation}
\begin{aligned}
    F_+^{(t)}(\mX_n^{(t)}) 
    = & (1\pm s^{*-1/3})\sqrt{1\pm\iota}  \left(1 \pm O\left(\frac{1}{\ln^5(d)}\right)\right)  \\
    & \times s^*\left(A_{+,r^*}^{*(t)}\left\vert S^{*(0)}_+(\vv_+) \right\vert + A_{+,c_n^t,r^*}^{*(t)}\left\vert S^{*(0)}_+(\vv_{+,c_n^t}) \right\vert \right),
\end{aligned}
\end{equation}

The same claims hold for the ``$-$'' class neurons (with the class signs flipped).
\end{lemma}

\begin{proof}
In this proof we focus on the neurons in $S_+^{*(0)}(\vv_{+})$; the proof for the update expressions for those in $S_+^{*(0)}(\vv_{+,c})$ are proven in virtually the same way.

\textbf{Base case}, $t = T_0$.

First define $A_{+,r^*}^{*(t)} \coloneqq \langle \vw_{+,r^*}, \vv_+ \rangle$, $(+,r^*) \in S_+^{*(0)}(\vv_+)$; similarly for $A_{+,c,r^*}^{*(t)} \coloneqq \langle \vw_{+,r^*}, \vv_{+,c} \rangle$. Note that the choice of $r^*$ does not really matter, since we know from phase I that every neuron in $S_+^{*(0)}(\vv_+)$ evolve at exactly the same rate, so by the end of phase I, $\|\vw_{+,r}^{(T_0)} - \vw_{+,r'}^{(T_0)}\|_2 \le O(\sigma_0\ln(d)) \ll \|\vw_{+,r}^{(T_0)}\|_2$ for any $(+,r), (+,r') \in S_+^{*(0)}(\vv_+)$.

Let $(+,r) \in S_+^{*(0)}(\vv_+)$.
Similar to phase I, consider the update equation
\begin{align}
    \vw_{+,r}^{(t+1)}
    = & \vw_{+,r}^{(t)} + \eta \frac{1}{NP} \times\\
    & \sum_{n=1}^N \Bigg( 
    \mathbbm{1}\{y_n = +\}[1-\text{logit}_+^{(t)}(\mX_n^{(t)})]\sum_{p\in[P]} \sigma'(\langle \vw_{+,r}^{(t)}, \vx_{n,p}^{(t)} \rangle + b_{+,r}^{(t)}) \vx_{n,p}^{(t)}\\
    & + \mathbbm{1}\{y_n = -\} [-\text{logit}_+^{(t)}(\mX_n^{(t)}) ]\sum_{p\in[P]} \sigma'(\langle \vw_{+,r}^{(t)}, \vx_{n,p}^{(t)}  \rangle + b^{(t)}_{+,r})  \vx_{n,p}^{(t)} \Bigg)
\end{align}

For the on-diagonal update expression, we have
\begin{equation}
\begin{aligned}
& \sum_{n=1}^N 
\mathbbm{1}\{y_n = +\}[1-\text{logit}_+^{(t)}(\mX_n^{(t)})]\sum_{p\in[P]} \sigma'(\langle \vw_{+,r}^{(t)}, \vx_{n,p}^{(t)} \rangle + b_{+,r}^{(t)}) \vx_{n,p}^{(t)} \\
= & \sum_{n=1}^N \mathbbm{1} \{ y_n = +\} [1-\text{logit}_+^{(t)}(\mX_n^{(t)})]\\
& \Bigg\{\mathbbm{1}\{|\calP(\mX_n^{(t)}; \vv_{+})|>0\} \Bigg[\sum_{p\in\calP(\mX_n^{(t)}; \vv_{+})} \mathbbm{1}\left\{ \langle \vw_{+,r}^{(t)}, \alpha_{n,p}^{(t)} \vv_{+} + \vzeta_{n,p}^{(t)} \rangle \ge b_{+,r}^{(t)} \right\}  \left(\alpha_{n,p}^{(t)}\vv_{+} + \vzeta_{n,p}^{(t)}\right) \\
& + \sum_{p\notin\calP(\mX_n^{(t)}; \vv_{+})} \mathbbm{1}\left\{\langle \vw_{+,r}^{(t)}, \vx_{n,p}^{(t)} \rangle \ge  b_{+,r}^{(t)}\right\} \vx_{n,p}^{(t)} \Bigg] \\
& + \mathbbm{1}\{|\calP(\mX_n^{(t)}; \vv_{+})|=0\} \sum_{p\in[P]} \mathbbm{1}\left\{\langle \vw_{+,r}^{(t)}, \vx_{n,p}^{(t)} \rangle \ge  b_{+,r}^{(t)}\right\} \vx_{n,p}^{(t)} \Bigg\}
\end{aligned}
\end{equation}
Following from Theorem \ref{prop: phase 1 sgd induction} and \ref{thm: sgd, universal nonact properties}, the neurons' non-activation on the patches that do not contain $\vv_{+}$, and activation on the $\vv_+$-dominated patches hold with probability at least $1 - O\left(\frac{mNPk_+T_0}{\poly(d)}\right)$ at time $T_0$. Therefore, the above update expression reduces to
\begin{equation}
\begin{aligned}
& \sum_{n=1}^N \mathbbm{1} \{ y_n = +, |\calP(\mX_n^{(t)}; \vv_{+})|>0\} [1-\text{logit}_+^{(t)}(\mX_n^{(t)})] \sum_{p\in\calP(\mX_n^{(t)}; \vv_{+})} \left(\alpha_{n,p}^{(t)}\vv_{+} + \vzeta_{n,p}^{(t)}\right)
\end{aligned}
\end{equation}
Note that for samples $\mX_n^{(t)}$ with $y_n = +$,

\begin{equation}
\begin{aligned}
    [1-\text{logit}_+^{(t)}(\mX_n^{(t)})] 
    = & \frac{\exp(F_-^{(t)}(\mX_n^{(t)}))}{\exp(F_-^{(t)}(\mX_n^{(t)})) + \exp(F_+^{(t)}(\mX_n^{(t)}))} \\
\end{aligned}
\end{equation}

Now we need to estimate the network response $F_+^{(t)}(\mX_n^{(t)})$. With probability at least $1 - \exp(-\Omega(s^{*1/3}))$, we have the upper bound (let $(+,c_n^t)$ denote the subclass which sample $\mX_n^{(t)}$ belongs to):
\begin{equation}
\begin{aligned}
    & F_+^{(t)}(\mX_n^{(t)}) \\
    \le & \sum_{p\in\calP(\mX^{(t)}; \vv_+)} \sum_{(+,r)\in S^{(0)}_+(\vv_+)} \langle \vw_{+,r}^{(t)}, \vv_+ + \vzeta_{n,p}^{(t)} \rangle + b_{+,r}^{(t)} \\
    & +  \sum_{p\in\calP(\mX^{(t)}; \vv_{+,c_n^t})} \sum_{(+,r)\in S^{(0)}_+(\vv_{+,c_n^t})} \langle \vw_{+,r}^{(t)}, \vv_{+,c_n^t} + \vzeta_{n,p}^{(t)} \rangle + b_{+,r}^{(t)} \\
    \le & (1+s^{*-1/3})\sqrt{1+\iota} s^* \left(1 + O\left(\frac{1}{\ln^{9}(d)}\right)\right) \left(A_{+,r^*}^{*(t)}\left\vert S^{(0)}_+(\vv_+) \right\vert + A_{+,c_n^t,r^*}^{*(t)}\left\vert S^{(0)}_+(\vv_{+,c_n^t}) \right\vert \right)
\end{aligned}
\end{equation}

The second inequality is true since $\max_r \langle \vw_{+,r}^{(t)}, \vv_+ \rangle \le A_{+,r^*}^{*(t)} + O(\sigma_0\ln(d))$, and for any $(+,r)\in S^{(0)}_+(\vv_+)$, $\vert\langle \vw_{+,r}^{(t)}, \vzeta_{n,p}^{(t)}\rangle\vert \le O(1/\ln^{9}(d))A_{+,r^*}^{*(t)}$. The bias value is negative (and so less than $0$). 

To further refine the bound, we recall $\left\vert S^{*(0)}_+(\vv) \right\vert / \left\vert S^{*(0)}_+(\vv') \right\vert, \left\vert S^{*(0)}_+(\vv) \right\vert / \left\vert S^{(0)}_+(\vv') \right\vert = 1 \pm O(1/\ln^5(d))$. 

Therefore, we obtain the bound
\begin{equation}
\begin{aligned}
    F_+^{(t)}(\mX_n^{(t)}) 
    \le & (1+s^{*-1/3})\sqrt{1+\iota} s^* \left(1 + O\left(\frac{1}{\ln^5(d)}\right)\right) \left( 1 + O\left( \frac{1}{\ln^5(d)} \right) \right) \\
    & \times \left(A_{+,r^*}^{*(t)}\left\vert S^{*(0)}_+(\vv_+) \right\vert + A_{+,c_n^t,r^*}^{*(t)}\left\vert S^{*(0)}_+(\vv_{+,c_n^t}) \right\vert \right)
\end{aligned}
\end{equation}

Following a similar argument, we also have the lower bound
\begin{equation}
\begin{aligned}
    & F_+^{(t)}(\mX_n^{(t)}) \\
    \ge & \sum_{p\in\calP(\mX^{(t)}; \vv_+)} \sum_{(+,r)\in S^{*(0)}_+(\vv_+)} \sigma\left(\langle \vw_{+,r}^{(t)}, \vv_+ + \vzeta_{n,p}^{(t)} \rangle + b_{+,r}^{(t)}\right) \\
    & +  \sum_{p\in\calP(\mX^{(t)}; \vv_{+,c_n^t})} \sum_{(+,r)\in S^{*(0)}_+(\vv_{+,c_n^t})} \sigma\left(\langle \vw_{+,r}^{(t)}, \vv_{+,c_n^t} + \vzeta_{n,p}^{(t)} \rangle + b_{+,r}^{(t)} \right)\\
    \ge & (1-s^{*-1/3})\sqrt{1-\iota} s^* \left(1 - O\left(\frac{1}{\ln^5(d)}\right)\right) \left( 1 - O\left( \frac{1}{\ln^5(d)} \right) \right)  \\
    & \times \left(A_{+,r^*}^{*(t)}\left\vert S^{*(0)}_+(\vv_+) \right\vert + A_{+,c_n^t,r^*}^{*(t)}\left\vert S^{*(0)}_+(\vv_{+,c_n^t}) \right\vert \right)
\end{aligned}
\end{equation}

The neurons in $ S^{*(0)}_+(\vv_+)$ have to activate, therefore they serve a key role in the lower bound, the bias bound for them is simply $-A_{+,r^*}^{*(t)}\Theta(1/\ln^5(d))$; the neurons in $S^{(0)}_+(\vv_{+,c})$ contribute at least $0$ due to the ReLU activation; the rest of the neurons do not activate. The same reasoning holds for the $S^{*(0)}_+(\vv_{+,c})$.

Knowing that neurons in $S^{*(0)}_+(\vv_+)$ cannot activate on the patches in samples belonging to the ``$-$'' class, now we may write the update expression for every $(+,r)\in S^{*(t)}_+(\vv_+)$ as (their updates are identical, same as in phase I):
\begin{equation}
\begin{aligned}
& \Delta \vw_{+,r}^{(t)} \\
= & \frac{\eta}{NP}\sum_{n=1}^N 
\mathbbm{1}\{y_n = +\}[1-\text{logit}_+^{(t)}(\mX_n^{(t)})]\sum_{p\in[P]} \sigma'(\langle \vw_{+,r}^{(t)}, \vx_{n,p}^{(t)} \rangle + b_{+,r}^{(t)}) \vx_{n,p}^{(t)} \\
= &\frac{\eta}{NP} \sum_{n=1}^N \mathbbm{1} \{ y_n = +, |\calP(\mX_n^{(t)}; \vv_{+})|>0\} \exp(-F_+^{(t)}(\mX_n^{(t)}))  \\
& \times \frac{\exp(F_-^{(t)}(\mX_n^{(t)}))}{\exp\left(F_-^{(t)}(\mX_n^{(t)})- \exp(F_+^{(t)}(\mX_n^{(t)}))\right) + 1 }\sum_{p\in\calP(\mX_n^{(t)}; \vv_{+})}  \left(\alpha_{n,p}^{(t)}\vv_{+} + \vzeta_{n,p}^{(t)}\right) \\
= & \eta  \sum_{n=1}^N  \mathbbm{1} \{ y_n = +\}\exp\Bigg\{ - (1+s^{*-1/3})\sqrt{1+\iota} s^* \left(1 + O\left(\frac{1}{\ln^5(d)}\right)\right) \\
& \times \left(A_{+,r^*}^{*(t)}\left\vert S^{*(0)}_+(\vv_+) \right\vert + A_{+,c_n^t,r^*}^{*(t)}\left\vert S^{*(0)}_+(\vv_{+,c_n^t}) \right\vert \right) \Bigg\} \\
& \times  \frac{\exp(F_-^{(t)}(\mX_n^{(t)}))}{\exp\left(F_-^{(t)}(\mX_n^{(t)})- F_+^{(t)}(\mX_n^{(t)})\right) + 1 } (1\pm s^{*-1/3})\frac{s^*}{NP} \left(\sqrt{1\pm\iota}\vv_{+} + \vzeta_{n,p}^{(t)}\right)
\label{expression: phase II, common neuron update estimate}
\end{aligned}
\end{equation}

This concludes the proof of the base case.

\textcolor{blue}{\textbf{Induction step}}. Assume the statements hold for time period $[T_0, t]$, prove for time $t+1$.

At step $t+1$, based on the induction hypothesis, we know that with probability at least $1 - O\left(\frac{mNPk_+t}{\poly(d)} \right)$, during time $\tau\in[T_0, t]$, for any $(+,r) \in S_+^{*(0)}(\vv_+)$,
\begin{equation}
\begin{aligned}
& \Delta \vw_{+,r}^{(\tau)} \\
= & \eta \sum_{n=1}^N \mathbbm{1} \{ y_n = +\} \exp\Bigg\{ - (1+s^{*-1/3})\sqrt{1+\iota} s^* \left(1 + O\left(\frac{1}{\ln^5(d)}\right)\right) \\
& \times \left(A_{+,r^*}^{*(\tau)}\left\vert S^{*(0)}_+(\vv_+) \right\vert + A_{+,c_n^t,r^*}^{*(\tau)}\left\vert S^{*(0)}_+(\vv_{+,c_n^t}) \right\vert \right) \Bigg\} \\
& \times   \frac{\exp(F_-^{(\tau)}(\mX_n^{(\tau)}))}{\exp\left(F_-^{(\tau)}(\mX_n^{(\tau)})- \exp(F_+^{(\tau)}(\mX_n^{(\tau)}))\right) + 1 } (1\pm s^{*-1/3})\frac{s^*}{NP} \left(\sqrt{1\pm\iota}\vv_{+} + \vzeta_{n,p}^{(\tau)}\right)
\end{aligned}
\end{equation}
and for the bias,
\begin{equation}
\begin{aligned}
& \Delta b_{+,r}^{(\tau)} \\
\le & - \eta \frac{1}{\ln^5(d)} \sum_{n=1}^N \mathbbm{1} \{ y_n = +\} \exp\Bigg\{ - (1+s^{*-1/3})\sqrt{1+\iota} s^* \left(1 + O\left(\frac{1}{\ln^5(d)}\right)\right) \\
& \times \left(A_{+,r^*}^{*(\tau)}\left\vert S^{*(0)}_+(\vv_+) \right\vert + A_{+,c_n^t,r^*}^{*(\tau)}\left\vert S^{*(0)}_+(\vv_{+,c_n^t}) \right\vert \right) \Bigg\}  \\
& \times  (1 - s^{*-1/3}) \frac{s^*}{NP} \left(\sqrt{1-\iota} - \frac{1}{\ln^{10}(d)}\right) \frac{\exp(F_-^{(\tau)}(\mX_n^{(\tau)}))}{\exp\left(F_-^{(\tau)}(\mX_n^{(\tau)})- \exp(F_+^{(\tau)}(\mX_n^{(\tau)}))\right) + 1 }
\end{aligned}
\end{equation}

Conditioning on the high-probability events of the induction hypothesis, 

\begin{equation}
\begin{aligned}
& \vw_{+,r}^{(t+1)} \\
= & \vw_{+,r}^{(T_0)} \\
& + \eta \sum_{\tau=T_0}^{t} \sum_{n=1}^N \mathbbm{1} \{ y_n = +\}\exp\Bigg\{ - (1+s^{*-1/3})\sqrt{1+\iota} s^* \left(1 + O\left(\frac{1}{\ln^5(d)}\right)\right) \\
& \times \left(A_{+,r^*}^{*(\tau)}\left\vert S^{*(0)}_+(\vv_+) \right\vert + A_{+,c_n^t,r^*}^{*(\tau)}\left\vert S^{*(0)}_+(\vv_{+,c_n^t}) \right\vert \right) \Bigg\}  \\
& \times \frac{\exp(F_-^{(\tau)}(\mX_n^{(\tau)}))}{\exp\left(F_-^{(\tau)}(\mX_n^{(\tau)})- \exp(F_+^{(\tau)}(\mX_n^{(\tau)}))\right) + 1 } (1\pm s^{*-1/3})\frac{s^*}{NP} \left(\sqrt{1\pm\iota}\vv_{+} + \vzeta_{n,p}^{(\tau)}\right)
\end{aligned}
\end{equation}

It follows that, with probability at least $1 - O\left(\frac{mNP}{\poly(d)} \right)$, for all $\vv_+$-dominated patch $\vx_{n,p}^{(t+1)}$,

\begin{equation}
\begin{aligned}
& \langle \vw_{+,r}^{(t+1)}, \vx_{n,p}^{(t+1)} \rangle + b_{+,r}^{(t+1)}\\
= & \langle \vw_{+,r}^{(T_0)}, \sqrt{1\pm\iota}\vv_+ + \vzeta_{n,p}^{(t+1)} \rangle + b_{+,r}^{(T_0)} \\
& + \eta \sum_{\tau=T_0}^{t} \sum_{n=1}^N \mathbbm{1} \{ y_n = +\} \exp\Bigg\{ - (1+s^{*-1/3})\sqrt{1+\iota} s^* \left(1 + O\left(\frac{1}{\ln^5(d)}\right)\right) \\
& \times \left(A_{+,r^*}^{*(\tau)}\left\vert S^{*(0)}_+(\vv_+) \right\vert + A_{+,c_n^t,r^*}^{*(\tau)}\left\vert S^{*(0)}_+(\vv_{+,c_n^t}) \right\vert \right) \Bigg\}  \\
& \times  \frac{\exp(F_-^{(\tau)}(\mX_n^{(\tau)}))}{\exp\left(F_-^{(\tau)}(\mX_n^{(\tau)})- F_+^{(\tau)}(\mX_n^{(\tau)})\right) + 1 } (1\pm s^{*-1/3})\frac{s^*}{NP} \\
& \times \langle \sqrt{1\pm\iota}\vv_{+} + \vzeta_{n,p}^{(\tau)}, \sqrt{1\pm\iota}\vv_+ + \vzeta_{n,p}^{(t+1)} \rangle + \Delta b_{+,r}^{(\tau)}\\
\ge & \; 0 \\
& + \eta \sum_{\tau=T_0}^{t} \sum_{n=1}^N \mathbbm{1} \{ y_n = +\} \exp\Bigg\{ - (1+s^{*-1/3})\sqrt{1+\iota} s^* \left(1 + O\left(\frac{1}{\ln^5(d)}\right)\right) \\
& \times \left(A_{+,r^*}^{*(\tau)}\left\vert S^{*(0)}_+(\vv_+) \right\vert + A_{+,c_n^t,r^*}^{*(\tau)}\left\vert S^{*(0)}_+(\vv_{+,c_n^t}) \right\vert \right) \Bigg\} \\
& \times  \frac{\exp(F_-^{(\tau)}(\mX_n^{(\tau)}))}{\exp\left(F_-^{(\tau)}(\mX_n^{(\tau)})- F_+^{(\tau)}(\mX_n^{(\tau)})\right) + 1 } (1\pm s^{*-1/3})\frac{s^*}{NP} \\
& \times \left(1-\iota - O\left( \frac{1}{\ln^5(d)}\right)\right) \\
> & \; 0
\end{aligned}
\end{equation}

Therefore the neurons $(+,r) \in S_+^{*(0)}(\vv_+)$ activate on the $\vv_+$-dominated patches $\vx_{n,p}^{(t+1)}$. We also know that they cannot activate on patches that are not dominated by $\vv_+$ by Theorem \ref{thm: sgd, universal nonact properties}. Following a similar derivation to the base case, we arrive at the result that, conditioning on the events of the induction hypothesis, with probability at least $1 - O\left(\frac{mNPk_+}{\poly(d)} \right)$, for all $(+,r) \in S_+^{*(0)}(\vv_+)$,
\begin{equation}
\begin{aligned}
& \Delta \vw_{+,r}^{(t+1)} \\
= & \eta  \sum_{n=1}^N \mathbbm{1} \{ y_n = +\} \exp\Bigg\{ - (1+s^{*-1/3})\sqrt{1+\iota} s^* \left(1 + O\left(\frac{1}{\ln^5(d)}\right)\right) \\
& \times \left(A_{+,r^*}^{*(t+1)}\left\vert S^{*(0)}_+(\vv_+) \right\vert + A_{+,c_n^t,r^*}^{*(t+1)}\left\vert S^{*(0)}_+(\vv_{+,c_n^t}) \right\vert \right) \Bigg\}  \\
& \times  (1\pm s^{*-1/3})\frac{s^*}{NP} \frac{\exp(F_-^{(t+1)}(\mX_n^{(t+1)}))}{\exp\left(F_-^{(t+1)}(\mX_n^{(t+1)})- F_+^{(t+1)}(\mX_n^{(t+1)})\right) + 1 } \left(\sqrt{1\pm\iota}\vv_{+} + \vzeta_{n,p}^{(t+1)}\right)
\end{aligned}
\end{equation}

Consequently, with probability at least $1 - O\left(\frac{mNP}{\poly(d)} \right)$,
\begin{equation}
\begin{aligned}
& \Delta b_{+,r}^{(t+1)}\\
\le & - \frac{1}{\ln^5(d)} \sum_{n=1}^N \mathbbm{1} \{ y_n = +\}\eta \exp\Bigg\{ - (1+s^{*-1/3})\sqrt{1+\iota} s^* \left(1 + O\left(\frac{1}{\ln^5(d)}\right)\right) \\
& \times \left(A_{+,r^*}^{*(t+1)}\left\vert S^{*(0)}_+(\vv_+) \right\vert + A_{+,c_n^t,r^*}^{*(t+1)}\left\vert S^{*(0)}_+(\vv_{+,c_n^t}) \right\vert \right) \Bigg\}  \\
& \times   \frac{\exp(F_-^{(t+1)}(\mX_n^{(t+1)}))}{\exp\left(F_-^{(t+1)}(\mX_n^{(t+1)})- F_+^{(t+1)}(\mX_n^{(t+1)})\right) + 1 } (1 - s^{*-1/3})\frac{s^*}{NP} \\
& \times \left(1-\iota - O\left( \frac{1}{\ln^9(d)}\right)\right)\\
\end{aligned}
\end{equation}

Utilizing the definition of conditional probability, we conclude that the expressions for $\Delta \vw_{+,r}^{(\tau)}$ and $\Delta b_{+,r}^{(t+1)}$ are indeed as described in the theorem during time $\tau \in [T_0, t+1]$ with probability at least $\left(1 - O\left(\frac{mNPk_+t}{\poly(d)} \right)\right) \times \left(1 - O\left(\frac{mNPk_+}{\poly(d)} \right)\right) \ge 1 - O\left(\frac{mNPk_+(t+1)}{\poly(d)} \right)$.

Moreover, based on the expression of $\Delta \vw_{+,r}^{(\tau)}$ and $\Delta b_{+,r}^{(t+1)}$, following virtually the same argument as in the base case, we can estimate the network output for any $(\mX_n^{(t+1)}, y_n=+)$:
\begin{equation}
\begin{aligned}
    F_+^{(t+1)}(\mX_n^{(t+1)}) 
    = & (1\pm s^{*-1/3})\sqrt{1\pm\iota}  \left(1 \pm O\left(\frac{1}{\ln^5(d)}\right)\right) s^* \\
    & \times \left(A_{+,r^*}^{*(t)}\left\vert S^{*(0)}_+(\vv_+) \right\vert + A_{+,c_n^t,r^*}^{*(t)}\left\vert S^{*(0)}_+(\vv_{+,c_n^t}) \right\vert \right)
\end{aligned}
\end{equation}

\end{proof}

\begin{lemma}
\label{lemma: phase II, after T1,1}
Define time $T_{1,1}$ to be the first point in time which the following identity holds on all $\mX_n^{(t)}$ belonging to the ``$+$'' class:
\begin{equation}
    \frac{\exp(F_-^{(t)}(\mX_n^{(t)}))}{\exp(F_-^{(t)}(\mX_n^{(t)}) - F_+^{(t)}(\mX_n^{(t)})) + 1} \ge 1 - O\left( \frac{1}{\ln^5(d)}\right)
\end{equation}
Then $T_{1,1} \le \poly(d)$, and for all $t \in [T_{1,1}, T_1]$, the above holds. The following also holds for this time period:
\begin{equation}
    [1 - \logit_+^{(t)}(\mX_n^{(t)})] \le O\left( \frac{1}{\ln^5(d)}\right)
\end{equation}

The same results also hold with the class signs flipped.
\end{lemma}
\begin{proof}
We first note that, the training loss $[1 - \logit_+^{(t)}(\mX_n^{(t)})]$ on samples belonging to the ``$+$'' class at any time during $t \in [T_0, T_1]$ is, asymptotically speaking, monotonically decreasing from $\frac{1}{2} - O(d^{-1})$. This can be easily proven by observing the way $s^*\left(A_{+,r^*}^{*(t)}\left\vert S^{*(0)}_+(\vv_+) \right\vert + A_{+,c,r^*}^{*(t)}\left\vert S^{*(0)}_+(\vv_{+,c}) \right\vert \right)$ monotonically increases from the proof of Lemma \ref{lemma: phase II, general}: before $F_+^{(t)}(\mX_n^{(t)}) \ge \ln\ln^5(d)$ on all $\mX_n^{(t)}$ belonging to the ``$+$'' class, there must be some samples $\mX_n^{(t)}$ on which
\begin{equation}
\begin{aligned}
    [1 - \logit_+^{(t)}(\mX_n^{(t)})] 
    = & \frac{\exp(F_-^{(t)}(\mX_n^{(t)}))}{\exp(F_-^{(t)}(\mX_n^{(t)})) + \exp(F_+^{(t)}(\mX_n^{(t)}))} \\
    \ge & \frac{1 - O(\sigma_0\ln(d)s^*d^{c_0})}{1 + O(\sigma_0\ln(d)s^*d^{c_0}) + \ln^5(d)} \\ 
    \ge & \Omega\left(\frac{1}{\ln^5(d)}\right).
\end{aligned}
\end{equation}
Therefore, by the update expressions in the proof of Lemma \ref{lemma: phase II, general}, $F_+^{(t)}(\mX_n^{(t)})$ can reach $\ln\ln^5(d)$ in time at most $O\left( \frac{NP\ln^5(d)}{\eta s^*} \right) \in \poly(d)$ (in the worst case scenario). At time $T_{1,1}$ and beyond,
\begin{equation}
\begin{aligned}
    1 - \frac{\exp(F_-^{(t)}(\mX_n^{(t)}))}{\exp(F_-^{(t)}(\mX_n^{(t)}) - F_+^{(t)}(\mX_n^{(t)})) + 1}
    \le & 1 - \frac{\exp(1 - O(\sigma_0d^{c_0}s^*))}{\exp(1 + O(\sigma_0d^{c_0}s^*))\frac{1}{\ln^5(d)} + 1}\\
    \le & O\left( \frac{1}{\ln^5(d)}\right).
\end{aligned}
\end{equation}
\end{proof}

\begin{lemma}
\label{lem: net response, final}
Denote $C = \eta \frac{s^*}{2k_+P}$, and write (for any $c\in[k_+]$)
\begin{equation}
    A_c(t) = s^*\left(A_{+,r^*}^{*(t)}\left\vert S^{*(0)}_+(\vv_+) \right\vert + A_{+,c,r^*}^{*(t)}\left\vert S^{*(0)}_+(\vv_{+,c}) \right\vert \right)
\end{equation}
(see Lemma \ref{lemma: phase II, general} for definition of $A_{\cdot}^{*(t)}$). Define $t_{c,0} = \exp(A_c(T_{1,1}))$. We write $A(t)$ and $t_0$ below for cleaner notations.

Then with probability at least $1 - o(1)$, during $t \in [T_{1,1}, T_1]$, 
\begin{equation}
    A(t) = \ln(C(t - T_{1,1}) +t_0) + E(t)
\end{equation}
where $|E(t)| \le O\left(\frac{1}{\ln^4(d)}\right) \sum_{\tau=C^{-1}t_0}^{t-T_{1,1}+C^{-1}t_0} \frac{1}{\tau} \le O\left(\frac{\ln(t) - \ln(C^{-1}t_0)}{\ln^4(d)}\right)$.

The same results also hold with the class signs flipped.
\end{lemma}
\begin{proof}
\textbf{Sidenote}: To make the writing a bit cleaner, we assume in the proof below that $C^{-1}t_0$ is an integer. The general case is easy to extend to by observing that $\left\vert \frac{1}{t - T_{1,1} + \ceil{C^{-1}t_0} } - \frac{1}{t - T_{1,1} + C^{-1}t_0 } \right\vert \le \frac{1}{(t - T_{1,1} + \ceil{C^{-1}t_0})(t - T_{1,1} + C^{-1}t_0) }$, which can be absorbed into the error term at every iteration since $\frac{1}{t - T_{1,1} + \ceil{C^{-1}t_0} } \ll \frac{1}{\ln^4(d)}$ due to $C^{-1}t_0 \ge \Omega(\sigma_0^{-1}/(\polyln(d)d^{c_0})) \gg d \gg \ln^4(d)$.

Based on result from Lemmas \ref{lemma: phase II, general} and \ref{lemma: phase II, after T1,1}, as long as $A(t)\le O\left(\ln(d)\right)$, we know during time $t\in[T_{1,1}, T_1]$ the update rule for $A(t)$ is as follows:
\begin{equation}
\begin{aligned}
A(t+1) - A(t)
= & C \exp\left\{ - (1\pm s^{*-1/3})\sqrt{1\pm\iota} \left(1 \pm O\left(\frac{1}{\ln^5(d)}\right)\right)   A(t) \right\}\\
& \times \left(1 \pm O\left(\frac{1}{\ln^5(d)}\right)\right)(1\pm s^{*-1/3}) \left(\sqrt{1\pm\iota} \pm \frac{1}{\ln^{10}(d)}\right) \\
= & C \exp\left\{ - A(t) \right\} \exp\left\{\pm O\left(\frac{1}{\ln^4(d)}\right)\right\} \left(1 \pm O\left(\frac{1}{\ln^5(d)}\right)\right) \\
= & C \exp\left\{ - A(t) \right\} \left(1 \pm \frac{C_1}{\ln^4(d)}\right) 
\end{aligned}
\end{equation}
where we write $C_1$ in place of $O(\cdot)$ for a more concrete update expression.

The base case $t = T_{1,1}$ is trivially true.

We proceed with the induction step. Assume the hypothesis true for $t \in [T_{1,1}, T]$, prove for $t + 1 = T + 1$.

Note that by Lemma \ref{lemma: harmonic series},
\begin{equation}
\begin{aligned}
A(t+1) 
= & \ln(C(t - T_{1,1}) + t_0) + E(t)  \\
& + C \exp\left\{ - \ln(C(t - T_{1,1}) + t_0) - E(t)) \right\} \left(1 \pm \frac{C_1}{\ln^4(d)} \right) \\
= & \ln(C) + \ln(t - T_{1,1} + C^{-1} t_0) + E(t) \\
& + C \frac{1}{C(t - T_{1,1}) +  t_0}\left(1 - E(t) \pm O(E(t)^2) \right)\left(1 \pm \frac{C_1}{\ln^4(d)}\right) \\
= & \ln(C) + \sum_{\tau=1}^{t-T_{1,1} + C^{-1}t_0 - 1} \frac{1}{\tau} + \frac{1}{2}\frac{1}{t-T_{1,1} + C^{-1}t_0} + \left[0, \frac{1}{8}\frac{1}{(t-T_{1,1} + C^{-1}t_0)^2}\right] \\
& + \frac{1}{t - T_{1,1} + C^{-1} t_0}\pm \frac{C_1}{\ln^4(d)}\frac{1}{t - T_{1,1} + C^{-1} t_0} \\
& + E(t) + \frac{1}{t - T_{1,1} + C^{-1} t_0} \left(- E(t) \pm O(E(t)^2) \right)\left(1 \pm \frac{C_1}{\ln^4(d)}\right)\\
= & \ln(C) + \sum_{\tau=1}^{t-T_{1,1} + C^{-1}t_0} \frac{1}{\tau} + \frac{1}{2}\frac{1}{t-T_{1,1} + C^{-1}t_0} + \left[0, \frac{1}{8}\frac{1}{(t-T_{1,1} + C^{-1}t_0)^2}\right] \\
& \pm \frac{C_1}{\ln^4(d)}\frac{1}{t - T_{1,1} + C^{-1} t_0} \\
& + E(t) + \frac{1}{t - T_{1,1} + C^{-1} t_0} \left(- E(t) \pm O(E(t)^2) \right)\left(1 \pm \frac{C_1}{\ln^4(d)}\right)\\
\end{aligned}
\end{equation}

Invoking Lemma \ref{lemma: harmonic series} again,
\begin{equation}
\begin{aligned}
A(t+1) 
= & \ln(C) + \ln(t+1 - T_{1,1} + C^{-1}t_0) \\
& - \frac{1}{2}\frac{1}{t+1-T_{1,1} + C^{-1}t_0} +\frac{1}{2}\frac{1}{t-T_{1,1} + C^{-1}t_0} \\
& + \left[- \frac{1}{8}\frac{1}{(t+1-T_{1,1} + C^{-1}t_0)^2}, 0 \right] + \left[0, \frac{1}{8}\frac{1}{(t-T_{1,1} + C^{-1}t_0)^2}\right] \\
& \pm \frac{C_1}{\ln^4(d)}\frac{1}{t - T_{1,1} + C^{-1} t_0} \\
& + E(t) + \frac{1}{t - T_{1,1} + C^{-1} t_0} \left(- E(t) \pm O(E(t)^2) \right)\left(1 \pm \frac{C_1}{\ln^4(d)}\right)\\
= & \ln(C(t+1-T_{1,1}) + t_0) \\
& +\frac{1}{2} \frac{1}{(t+1-T_{1,1} + C^{-1}t_0)(t-T_{1,1} + C^{-1}t_0)} \pm O\left( \frac{1}{(t+1-T_{1,1} + C^{-1}t_0)^2}\right) \\
& \pm \frac{C_1}{\ln^4(d)}\frac{1}{t - T_{1,1} + C^{-1} t_0} \\
& + E(t) + \frac{1}{t - T_{1,1} + C^{-1} t_0} \left(- E(t) \pm O(E(t)^2) \right)\left(1 \pm \frac{C_1}{\ln^4(d)}\right)\\
\end{aligned}
\end{equation}

To further refine the expression, first note that the error passed down from the previous step $t$ does not grow in this step (in fact it slightly decreases):
\begin{equation}
\begin{aligned}
    & \left\vert E(t) + \frac{1}{t - T_{1,1} + C^{-1} t_0} \left(- E(t) \pm O(E(t)^2) \right)\left(1 \pm \frac{C_1}{\ln^4(d)}\right) \right\vert \\
    < & |E(t)| \\
    \le & O\left(\frac{1}{\ln^4(d)}\right) \sum_{\tau=C^{-1}t_0}^{t-T_{1,1}+C^{-1}t_0} \frac{1}{\tau}.
\end{aligned}
\end{equation}

Moreover, notice that at step $t+1$, since $\frac{1}{t+1-T_{1,1} + C^{-1}t_0} \ll \frac{1}{\ln^4(d)}$, the error term $\vert E(t + 1) \vert = \vert A(t+1) - \ln(C(t+1-T_{1,1}) + t_0) \vert \le O\left(\frac{1}{\ln^4(d)}\right) \sum_{\tau=C^{-1}t_0}^{t+1-T_{1,1}+C^{-1}t_0} \frac{1}{\tau}$, which finishes the inductive step.

\end{proof}

\begin{lemma}
\label{lemma: sgd phase II, common vs finegrained ratio}
    With probability at least $1 - O\left(\frac{mNPk_+T_1}{\poly(d)} \right)$, for all $t\in[0, T_1]$, all $c\in[k_+]$,
    \begin{equation}
    \begin{aligned}
        \frac{\Delta A_{+,c,r^*}^{*(t)}}{\Delta A_{+,r^*}^{*(t)}} &= \Theta\left( \frac{1}{k_+}\right),\\
        \frac{A_{+,c,r^*}^{*(t)}}{A_{+,r^*}^{*(t)}} & = \Theta\left( \frac{1}{k_+}\right).
    \end{aligned}
    \end{equation}

    The same identity holds for the ``$-$''-classes.
\end{lemma}

\begin{proof}
    The statements in the lemma follow trivially from Theorem \ref{prop: phase 1 sgd induction} for time period $[0,T_0]$. Let us focus on the phase $[T_0, T_1]$.

    In this proof, we condition on the high-probability events of Lemma \ref{lem: net response, final} and Lemma \ref{lemma: phase II, general}.
    
    First of all, based on Lemma \ref{lem: net response, final}, we know that $s^*A_{+,r^*}^{*(t)}\left\vert S^{*(0)}_+(\vv_+) \right\vert \le O(\ln(d))$. We will make use of this fact later.

    \textbf{Base case}, $t=T_0$.

    The base case directly follows from our Theorem \ref{prop: phase 1 sgd induction}.

    \textbf{Induction step}, assume statement holds for $\tau \in [T_0, t]$, prove statement for $t+1$.

    By Lemma \ref{lemma: phase II, general}, we know that 
    \begin{equation}
    \begin{aligned}
    & \Delta A_{+,r^*}^{*(t)} \\
    = & \eta \sum_{c=1}^{k_+}\exp\Bigg\{ - (1\pm s^{*-1/3})\sqrt{1\pm\iota} s^* \left(1 \pm O\left(\frac{1}{\ln^5(d)}\right)\right) \\
    & \times \left(A_{+,r^*}^{*(t)}\left\vert S^{*(0)}_+(\vv_+) \right\vert + A_{+,c,r^*}^{*(t)}\left\vert S^{*(0)}_+(\vv_{+,c}) \right\vert \right) \Bigg\}  \\
    & \times [1/3,1] (1\pm s^{*-1/3})\frac{s^*}{2k_+P} \left(\sqrt{1\pm\iota}\pm O\left(\frac{1}{\ln^9(d)}\right)\right),
    \end{aligned}
    \end{equation}
    and for any $c\in[k_+]$,
    \begin{equation}
    \begin{aligned}
    & \Delta A_{+,c,r^*}^{*(t)} \\
    = & \eta\exp\Bigg\{ - (1\pm s^{*-1/3})\sqrt{1\pm\iota} s^* \left(1 \pm O\left(\frac{1}{\ln^5(d)}\right)\right) \\
    & \times \left(A_{+,r^*}^{*(t)}\left\vert S^{*(0)}_+(\vv_+) \right\vert + A_{+,c,r^*}^{*(t)}\left\vert S^{*(0)}_+(\vv_{+,c}) \right\vert \right) \Bigg\}  \\
    & \times [1/3,1] (1\pm s^{*-1/3})\frac{s^*}{2k_+P} \left(\sqrt{1\pm\iota}\pm O\left(\frac{1}{\ln^9(d)}\right)\right),
    \end{aligned}
    \end{equation}

    Relying on the induction hypothesis, we can reduce the above expressions to
    \begin{equation}
    \begin{aligned}
    & \Delta A_{+,r^*}^{*(t)} \\
    = & \eta \sum_{c=1}^{k_+}\exp\Bigg\{ - (1\pm s^{*-1/3})\sqrt{1\pm\iota}  \left(1 \pm O\left(\frac{1}{\ln^5(d)}\right)\right) \left( 1 \pm O\left( \frac{1}{k_+}\right) \right) s^*A_{+,r^*}^{*(t)}\left\vert S^{*(0)}_+(\vv_+) \right\vert   \Bigg\}  \\
    & \times [1/3,1] (1\pm s^{*-1/3})\frac{s^*}{2k_+P} \left(\sqrt{1\pm\iota}\pm O\left(\frac{1}{\ln^9(d)}\right)\right) \\
    = & \eta \exp\Bigg\{ - (1\pm s^{*-1/3})\sqrt{1\pm\iota}  \left(1 \pm O\left(\frac{1}{\ln^5(d)}\right)\right) \left( 1 \pm O\left( \frac{1}{k_+}\right) \right) s^*A_{+,r^*}^{*(t)}\left\vert S^{*(0)}_+(\vv_+) \right\vert   \Bigg\}  \\
    & \times \Theta(1) \times \frac{s^*}{2P} ,
    \end{aligned}
    \end{equation}
    and for any $c\in[k_+]$,
    \begin{equation}
    \begin{aligned}
    & \Delta A_{+,c,r^*}^{*(t)} \\
    = & \eta\exp\Bigg\{ - (1\pm s^{*-1/3})\sqrt{1\pm\iota}  \left(1 \pm O\left(\frac{1}{\ln^5(d)}\right)\right) \left( 1 \pm O\left( \frac{1}{k_+}\right) \right) s^* A_{+,r^*}^{*(t)}\left\vert S^{*(0)}_+(\vv_+) \right\vert  \Bigg\}  \\
    & \times \Theta(1) \times \frac{s^*}{2k_+P} .
    \end{aligned}
    \end{equation}

    By invoking the property that $s^*A_{+,r^*}^{*(t)}\left\vert S^{*(0)}_+(\vv_+) \right\vert \le O(\ln(d))$, we find that for all $c\in[k_+]$,
    \begin{equation}
    \begin{aligned}
    \frac{\Delta A_{+,c,r^*}^{*(t)}}{\Delta A_{+,r^*}^{*(t)}} 
    = & \exp\Bigg\{ \pm O\left( \frac{1}{\ln^5(d)}\right) s^* A_{+,r^*}^{*(t)}\left\vert S^{*(0)}_+(\vv_+) \right\vert  \Bigg\}  \times \Theta\left( \frac{1}{k_+}\right)  \\
    = & \left( 1 \pm O\left( \frac{1}{\ln^4(d)}\right)\right) \times \Theta\left( \frac{1}{k_+}\right)\\
    = & \Theta\left( \frac{1}{k_+}\right).
    \end{aligned}
    \end{equation}

    Therefore, we can finish our induction step:
    \begin{equation}
        \frac{A_{+,c,r^*}^{*(t+1)}}{A_{+,r^*}^{*(t+1)}} = \frac{A_{+,c,r^*}^{*(t)} + \Delta A_{+,c,r^*}^{*(t)}}{A_{+,r^*}^{*(t)} + \Delta A_{+,r^*}^{*(t)}} = \frac{A_{+,c,r^*}^{*(t)} + \Delta A_{+,c,r^*}^{*(t)}}{\Theta\left(k_+\right) \times \left( A_{+,c,r^*}^{*(t)} + \Delta A_{+,c,r^*}^{*(t)} \right)} = \Theta\left( \frac{1}{k_+}\right).
    \end{equation}
\end{proof}

\begin{lemma}
    Let $T_{\Omega(1)}$ be the first point in time such that either $s^*A_{+,r^*}^{*(t)}\left\vert S^{*(0)}_+(\vv_+) \right\vert \ge \Omega(1)$ or $s^*A_{-,r^*}^{*(t)}\left\vert S^{*(0)}_-(\vv_-) \right\vert \ge \Omega(1)$. Then for any $t < T_{\Omega(1)}$,
    \begin{equation}
        \frac{A_{-,r^*}^{*(t)}}{A_{+,r^*}^{*(t)}} = \Theta(1)
    \end{equation}
    and for any $t \in [T_{\Omega(1)}, T_1]$,
    \begin{equation}
        \frac{A_{-,r^*}^{*(t)}}{A_{+,r^*}^{*(t)}}, \frac{A_{+,r^*}^{*(t)}}{A_{-,r^*}^{*(t)}} \ge \Omega\left(\frac{1}{\ln(d)}\right).
    \end{equation}
\end{lemma}
\begin{proof}
    This lemma is a consequence of Theorem \ref{prop: phase 1 sgd induction}, Lemma\ref{lemma: phase II, general} and Lemma \ref{lem: net response, final}.

    Due to Theorem \ref{prop: phase 1 sgd induction}, we already know that $\frac{A_{-,r^*}^{*(t)}}{A_{+,r^*}^{*(t)}} = \Theta(1)$ up to time $T_0$. In addition, with Lemma \ref{lemma: phase II, general} we know that before $s^*A_{+,r^*}^{*(t)}\left\vert S^{*(0)}_+(\vv_+) \right\vert \ge \Omega(1)$, the loss term (on a $+$-class sample) $1 - \logit_+^{(t)}(\mX_n^{(t)}) = \Theta(1)$ (the same holds with the class signs flipped), in which case it is also easy to derive $\frac{A_{-,r^*}^{*(t)}}{A_{+,r^*}^{*(t)}} = \Theta(1)$ by noting that the update expressions $\Delta A_{-,r^*}^{*(t)}/ \Delta A_{+,r^*}^{*(t)} = \Theta(1)$. 

    Beyond time $T_{\Omega(1)}$, by Lemma \ref{lem: net response, final}, we know that $s^*A_{+,r^*}^{*(t)}\left\vert S^{*(0)}_+(\vv_+) \right\vert, s^*A_{-,r^*}^{*(t)}\left\vert S^{*(0)}_-(\vv_-) \right\vert \le O(\ln(d))$. With the understanding that $s^*A_{+,r^*}^{*(t)}\left\vert S^{*(0)}_+(\vv_+) \right\vert,  s^*A_{-,r^*}^{*(t)}\left\vert S^{*(0)}_-(\vv_-) \right\vert \ge \Omega(1)$ beyond $T_{\Omega(1)}$ due to the monotonicity of these functions, and the property $\left\vert \frac{|S^{*(0)}_-(\vv_-)|}{|S^{*(0)}_+(\vv_+)|} - 1 \right\vert \le O\left( \frac{1}{\ln^5(d)} \right)$ from Proposition \ref{prop: init geometry, coarse}, the rest of the lemma follows.
\end{proof}

\begin{lemma}
With probability at least $1 - O\left(\frac{mNPk_+t}{\poly(d)} \right)$, for all $t\in[0,T_1]$ and all $(+,r)\in S^{*(0)}_+(\vv_+)$,
\begin{equation}
    \frac{\Delta b_{+,r}^{(t)}}{\Delta A_{+,r}^{(t)}} = -\Theta\left( \frac{1}{\ln^5(d)}\right).
\end{equation}

The same holds with the $+$-class sign replaced by the $-$-class sign.
\end{lemma}
\begin{proof}
Choose any $(+,r)\in S^{*(0)}_+(\vv_+)$.

The statement in this lemma for time period $t\in[0,T_0]$ follows easily from Theorem \ref{prop: phase 1 sgd induction} and its proof. Let us examine the period $t\in[T_0, T_1]$.

Based on Lemma \ref{lemma: phase II, general} and its proof and Lemma \ref{lemma: sgd phase II, common vs finegrained ratio}, we know that for $t\in [T_0,T_1]$, with probability at least $1 - O\left(\frac{mNPk_+t}{\poly(d)} \right)$,
\begin{equation}
\begin{aligned}
& \Delta A_{+,r}^{(t)} \\
= & \eta  \exp\Bigg\{ - (1\pm s^{*-1/3})\sqrt{1\pm\iota}  \left(1 \pm O\left(\frac{1}{\ln^5(d)}\right)\right) \left(1 \pm O\left(\frac{1}{k_+}\right)\right) s^*A_{+,r^*}^{*(t)}\left\vert S^{*(0)}_+(\vv_+) \right\vert\Bigg\}  \\
& \times (1\pm s^{*-1/3})\frac{s^*}{NP} \left(\sqrt{1\pm\iota} \pm O\left( \frac{1}{\ln^9(d)} \right)\right) \sum_{n=1}^N  \mathbbm{1} \{ y_n = +\} \frac{\exp(F_-^{(t)}(\mX_n^{(t)}))}{\exp\left(F_-^{(t)}(\mX_n^{(t)})- F_+^{(t)}(\mX_n^{(t)})\right) + 1 } 
\end{aligned}
\end{equation}

Furthermore,
\begin{equation}
\begin{aligned}
& \Delta b_{+,r}^{(t)} \\
= & - \frac{\|\Delta \vw_{+,r}^{(t)}\|_2}{\ln^5(d)} \\
= & - \eta \frac{1}{\ln^5(d)} \exp\Bigg\{ - (1 \pm s^{*-1/3})\sqrt{1\pm\iota}  \left(1 \pm O\left(\frac{1}{\ln^5(d)}\right)\right) \left(1 \pm O\left(\frac{1}{k_+}\right)\right) s^*A_{+,r^*}^{*(t)}\left\vert S^{*(0)}_+(\vv_+) \right\vert\Bigg\}  \\
& \times (1 \pm s^{*-1/3}) \frac{s^*}{NP} \left(1\pm\iota \pm \frac{1}{\ln^{9}(d)}\right)  \sum_{n=1}^N  \mathbbm{1} \{ y_n = +\}  \frac{\exp(F_-^{(t)}(\mX_n^{(t)}))}{\exp\left(F_-^{(t)}(\mX_n^{(t)})- F_+^{(t)}(\mX_n^{(t)})\right) + 1 } \\
\end{aligned}
\end{equation}

With the understanding that $s^*A_{+,r^*}^{*(t)}\left\vert S^{*(0)}_+(\vv_+) \right\vert \le O(\ln(d))$ from Lemma \ref{lem: net response, final} and the fact that $ \frac{\exp(F_-^{(t)}(\mX_n^{(t)}))}{\exp\left(F_-^{(t)}(\mX_n^{(t)})- F_+^{(t)}(\mX_n^{(t)})\right) + 1 } = \Theta(1)$, we have

\begin{equation}
\begin{aligned}
& \frac{\Delta b_{+,r}^{(t)}}{\Delta A_{+,r}^{(t)}} \\
= & -\Theta\left( \frac{1}{\ln^5(d)}\right)\exp\Bigg\{ - \left(1 \pm O\left(\frac{1}{\ln^5(d)}\right)\right) s^*A_{+,r^*}^{*(t)}\left\vert S^{*(0)}_+(\vv_+) \right\vert\Bigg\} \\
= & -\Theta\left( \frac{1}{\ln^5(d)}\right) \left(1 \pm O\left(\frac{1}{\ln^4(d)}\right)\right)  \\
= & -\Theta\left( \frac{1}{\ln^5(d)}\right).
\end{aligned}
\end{equation}

\end{proof}

\begin{lemma}[Probability of mistake on hard samples is high]
\label{lemma: coarse, hard sample mistake prob}

For all $t\in[0,T_1]$, given a hard test sample $(\mX_{\text{hard}},y)$,
\begin{equation}
    \mathbb{P}\left[F_y^{(T)}(\mX_{\text{hard}}) \le F_{y'}^{(T)}(\mX_{\text{hard}})\right] \ge \Omega(1).
\end{equation}
\end{lemma}
\begin{proof}
We first show that at time $t=0$, the probability of the network making a mistake on hard test samples is $\Omega(1)$, then prove that for the rest of the time, i.e. $t\in(0, T_1]$, the model still makes mistake on hard test samples with probability $\Omega(1)$.

At time $t=0$, by Lemma \ref{lemma: independent gaussian vector inner product concentration}, we know that for any $r\in[m]$, with probability $\Omega(1)$,
\begin{equation}
    \langle \vw_{+,r}^{(0)}, \vzeta^*\rangle \ge \Omega(\sigma_0\sigma_{\zeta^*}\sqrt{d}) \ge \Omega(\sigma_0\polyln(d)) \gg \Omega\left(\sigma_0 \sqrt{\ln(d)} \right).
\end{equation}

Relying on concentration of the binomial random variable, with probability at least $1 - e^{-\Omega(\polyln(d))}$,
\begin{equation}
    \sum_{r=1}^m \sigma\left(\langle \vw_{+,r}^{(0)}, \vzeta^*\rangle + b_{+,r}^{(0)} \right) \ge \Omega(m\sigma_0\sigma_{\zeta^*}\sqrt{d}),
\end{equation}
which is asymptotically larger than the activation from the features, which, following from Proposition \ref{prop: init geometry, coarse}, is upper bounded by $O\left(\sigma_0\sqrt{\ln(d)}s^*d^{c_0} \right)$. The same can be said for the ``$-$'' class. In other words,
\begin{equation}
\begin{aligned}
    &F_-^{(0)}(\mX_{\text{hard}}) - F_+^{(0)}(\mX_{\text{hard}}) > 0 \\
    \iff & \Bigg\{\sum_{r=1}^m\mathbbm{1}\{\langle \vw_{-,r}^{(0)}, \vzeta^*\rangle + b_{-,r}^{(0)} > 0\}\langle \vw_{-,r}^{(0)}, \vzeta^*\rangle \\
    & - \sum_{r=1}^m\mathbbm{1}\{\langle \vw_{+,r}^{(0)}, \vzeta^*\rangle + b_{+,r}^{(0)} > 0\}\langle \vw_{+,r}^{(0)}, \vzeta^*\rangle\Bigg\}(1\pm o(1)) > 0
\end{aligned}
\end{equation}
which clearly holds with probability $\Omega(1)$.

Now consider $t\in (0, T_1]$.

During this period of time, by Theorem \ref{prop: phase 1 sgd induction} and Lemma \ref{lemma: phase II, general}, we note that for any $c\in[k_+]$ and $(+,r)\in S_{+}^{*(0)}(\vv_{+,c})$, $\Delta\vzeta_{+,r}^{(t)} \sim \calN(\vzero, \sigma_{\Delta \zeta_{+,r}}^{(t)2} \mI_d)$, with $\sigma_{\Delta \zeta_{+,r}}^{(t)} =\Theta\left(\Delta A_{+,c,r}^{(t)} \sqrt{\frac{2k_+}{s^*N}} \sigma_{\zeta}\right)$. The same can be said for $(+,r)\in S_{+}^{*(0)}(\vv_+)$, although with the $\Delta A_{+,c,r}^{(t)}\sqrt{\frac{2k_+}{s^*N}}$ factor replaced by $\Delta A_{+,r}^{(t)}\sqrt{\frac{2}{s^*N}}$. Also from the proofs of Theorem \ref{prop: phase 1 sgd induction} and Lemma \ref{lemma: phase II, general}, and using the property $\vert \calU^{(0)}_{+,r}\vert\le O(1)$ from Proposition \ref{prop: init geometry, coarse}, we know that for all neurons, the updates to the neurons also take the feature-plus-Gaussian-noise form of $\sum_{\vv'\in\calU^{(0)}_{+,r}}c^{(t)}(\vv')\vv' + \Delta \vzeta_{+,r}^{(t)}$, with $c^{(t)}(\vv')\le \left(1 + O\left( \frac{1}{\ln^5(d)}\right) \right)\Delta A_{+,c,r}^{(t)}$ if $\vv'=\vv_{+,c}$ for some $c\in[k_+]$, or $c^{(t)}(\vv')\le  \left(1 + O\left( \frac{1}{\ln^5(d)}\right) \right)\Delta A_{+,r}^{(t)} $ if $\vv'=\vv_+$ (because the $\vv'$ component of a $\vv'$-singleton neuron's update is already the maximum possible). Moreover, if $\vv_+\in\calU_{+,r}^{(0)}$, then $\sigma_{\Delta \zeta_{+,r}}^{(t)} \le O\left( \Delta A_{+,r}^{(t)}\sqrt{\frac{2}{s^*N}}\sigma_{\zeta}\right) + O\left(\Delta A_{+,c,r}^{(t)} \sqrt{\frac{2k_+}{s^*N}} \sigma_{\zeta}\right) \le O\left( \Delta A_{+,r}^{(t)}\sqrt{\frac{2}{s^*N}}\sigma_{\zeta}\right)$, otherwise, if $\calU_{+,r}^{(0)}$ only contains the fine-grained features, then $\sigma_{\Delta \zeta_{+,r}}^{(t)} \le O\left(\Delta A_{+,c,r}^{(t)} \sqrt{\frac{2k_+}{s^*N}} \sigma_{\zeta}\right)$.

With the understanding that only neurons in $S_{y}^{(0)}(\vv_y)$ and $S_{y}^{(0)}(\vv_{y,c})$ can possibly activate on the feature patches of a sample when $t \le T_1$ (coming from Theorem \ref{thm: sgd, universal nonact properties}), we have
\begin{equation}
\begin{aligned}
    F_+^{(t)}(\mX_{\text{hard}}) 
    \le & \sum_{(+,r)\in S^{(0)}_+(\vv_{+,c})}  \sum_{p\in\calP(\mX_{\text{hard}};\vv_{+,c})} \sigma\left( \langle \vw_{+,r}^{(0)} +\sum_{\tau=0}^{t-1}\Delta\vw_{+,r}^{(\tau)}, \sqrt{1\pm\iota}\vv_{+,c} + \vzeta_{p} \rangle + b_{+,r}^{(t)}\right) \\
    & + \sum_{r\in[m]} \sigma\left( \langle \vw_{+,r}^{(0)} +\sum_{\tau=0}^{t-1}\Delta\vw_{+,r}^{(\tau)}, \vzeta^* \rangle + b_{+,r}^{(t)}\right)  \\
    & + \sum_{(+,r)\in S^{(0)}_+(\vv_{-})} \sum_{p\in\calP(\mX_{\text{hard}};\vv_{-})} \sigma\left( \langle \vw_{+,r}^{(0)} +\sum_{\tau=0}^{t-1}\Delta\vw_{+,r}^{(\tau)}, \alpha_{p}^{\dagger}\vv_{-} + \vzeta_{p} \rangle + b_{+,r}^{(t)}\right)
\end{aligned}
\end{equation}

To further refine this upper bound, we first note that with probability at least $1 - O\left(\frac{mNPk_+t}{\poly(d)} \right)$, the following holds with arbitrary choice of $(+,r^*)\in S^{(0)}_+(\vv_{+,c})$:
\begin{equation}
\begin{aligned}
    &\sum_{(+,r)\in S^{(0)}_+(\vv_{+,c})} \sum_{p\in\calP(\mX_{\text{hard}};\vv_{+,c})} \langle \sum_{\tau=0}^{t-1}\Delta\vw_{+,r}^{(\tau)}, \sqrt{1\pm\iota}\vv_{+,c} + \vzeta_{p} \rangle \le O\left(s^* \left\vert S^{(0)}_+(\vv_{+,c}) \right\vert \sum_{\tau=0}^{t-1} \Delta A_{+,c,r^*}^{(\tau)}\right)
\end{aligned}
\end{equation}

Invoking Lemma \ref{lemma: sgd phase II, common vs finegrained ratio}, we obtain (for arbitrary $(+,r^*)\in S^{(0)}_+(\vv_{+})$):
\begin{equation}
\begin{aligned}
    &\sum_{(+,r)\in S^{(0)}_+(\vv_{+,c})} \sum_{p\in\calP(\mX_{\text{hard}};\vv_{+,c})} \langle \sum_{\tau=0}^{t-1}\Delta\vw_{+,r}^{(\tau)}, \sqrt{1\pm\iota}\vv_{+,c} + \vzeta_{p} \rangle 
    \le O\left(\frac{1}{k_+} s^* \left\vert S^{(0)}_+(\vv_{+,c}) \right\vert \sum_{\tau=0}^{t-1} \Delta A_{+,r^*}^{(\tau)}\right)
\end{aligned}
\end{equation}

Let us examine the term $\sum_{r\in[m]} \sigma\left( \langle \vw_{+,r}^{(0)} +\sum_{\tau=0}^{t-1}\Delta\vw_{+,r}^{(\tau)}, \vzeta^* \rangle + b_{+,r}^{(t)}\right)$ more carefully. First of all, denoting $S_+^{(0)} = \cup_{c=1}^{k_+} S^{(0)}_+(\vv_{+,c}) \cup \cup_{c=1}^{k_-}  S^{(0)}_+(\vv_{-,c}) \cup S^{(0)}_+(\vv_{+}) \cup S^{(0)}_+(\vv_{-})$, neurons $(+,r)\notin S_+^{(0)}$ cannot receive any update at all during training due to Theorem \ref{thm: sgd, universal nonact properties}. Therefore we can rewrite the term
\begin{equation}
\begin{aligned}
    &\sum_{r\in[m]} \sigma\left( \langle \vw_{+,r}^{(0)} +\sum_{\tau=0}^{t-1}\Delta\vw_{+,r}^{(\tau)}, \vzeta^* \rangle + b_{+,r}^{(t)}\right) \\
    = & \sum_{(+,r)\in S^{(0)}_+} \sigma\left( \langle \vw_{+,r}^{(0)} +\sum_{\tau=0}^{t-1}\Delta\vw_{+,r}^{(\tau)}, \vzeta^* \rangle + b_{+,r}^{(t)}\right) 
    + \sum_{(+,r)\notin S^{(0)}_+ } \sigma\left( \langle \vw_{+,r}^{(0)} , \vzeta^* \rangle + b_{+,r}^{(0)}\right)
\end{aligned}
\end{equation}

Relying on Corollary \ref{coro: coarse training, bias upper bd}, we know
\begin{equation}
    \sum_{\tau=0}^{t-1}\Delta b_{+,r}^{(\tau)} <\sum_{\tau=0}^{t-1} -\Omega\left(\frac{\polyln(d)}{\ln^5(d)}\right)\left\vert\langle \Delta \vw_{+,r}^{(\tau)}, \vzeta^* \rangle \right\vert.
\end{equation}

Therefore, we know that for $r\in[m]$, 
\begin{equation}
    \sum_{\tau=0}^{t-1} \langle \Delta\vw_{+,r}^{(\tau)}, \vzeta^* \rangle + \Delta b_{+,r}^{(\tau)} \le 0
\end{equation}

As a consequence, we can write the naive upper bound
\begin{equation}
\begin{aligned}
    &\sum_{r\in[m]} \sigma\left( \langle \vw_{+,r}^{(0)} +\sum_{\tau=0}^{t-1}\Delta\vw_{+,r}^{(\tau)}, \vzeta^* \rangle + b_{+,r}^{(t)}\right) \\
    \le & \sum_{(+,r)\in S^{(0)}_+} \sigma\left( \langle \vw_{+,r}^{(0)} , \vzeta^* \rangle + b_{+,r}^{(0)}\right) 
    + \sum_{(+,r)\notin S^{(0)}_+ } \sigma\left( \langle \vw_{+,r}^{(0)} , \vzeta^* \rangle + b_{+,r}^{(0)}\right) \\
    = & \sum_{r\in[m] } \sigma\left( \langle \vw_{+,r}^{(0)} , \vzeta^* \rangle + b_{+,r}^{(0)}\right)
\end{aligned}
\end{equation}

Additionally, due to Theorem \ref{thm: sgd, universal nonact properties} (and its proof), we know that
\begin{equation}
\begin{aligned}
    & \sum_{(+,r)\in S^{(0)}_+(\vv_{-})} \sum_{p\in\calP(\mX_{\text{hard}};\vv_{-})} \sigma\left( \langle \vw_{+,r}^{(0)} +\sum_{\tau=0}^{t-1}\Delta\vw_{+,r}^{(\tau)}, \alpha_{p}^{\dagger}\vv_{-} + \vzeta_{p} \rangle + b_{+,r}^{(t)}\right) \\
    \le & \sum_{(+,r)\in S^{(0)}_+(\vv_{-})} \sum_{p\in\calP(\mX_{\text{hard}};\vv_{-})} \sigma\left( \langle \vw_{+,r}^{(0)}, \alpha_{p}^{\dagger}\vv_{-} + \vzeta_{p} \rangle + b_{+,r}^{(0)}\right)
\end{aligned}
\end{equation}

It follows that
\begin{equation}
\begin{aligned}
    & F_+^{(t)}(\mX_{\text{hard}}) \\
    \le & O\left(\frac{1}{k_+} s^* \left\vert S^{(0)}_+(\vv_{+,c}) \right\vert \sum_{\tau=0}^{t-1} \Delta A_{+,r^*}^{(\tau)}\right) + \sum_{(+,r)\in S^{(0)}_+(\vv_{+,c})}  \sum_{p\in\calP(\mX_{\text{hard}};\vv_{+,c})} \left\vert \langle \vw_{+,r}^{(0)}, \sqrt{1\pm\iota}\vv_{+,c} + \vzeta_{p} \rangle \right\vert\\
    & + \sum_{r\in[m] } \sigma\left( \langle \vw_{+,r}^{(0)} , \vzeta^* \rangle + b_{+,r}^{(0)}\right) 
    + \sum_{(+,r)\in S^{(0)}_+(\vv_{-})} \sum_{p\in\calP(\mX_{\text{hard}};\vv_{-})} \sigma\left( \langle \vw_{+,r}^{(0)}, \alpha_{p}^{\dagger}\vv_{-} + \vzeta_{p} \rangle + b_{+,r}^{(0)}\right)
\end{aligned}
\end{equation}

On the other hand, for the ``$-$'' neurons, denoting $S_-^{(0)} = \cup_{c=1}^{k_+} S^{(0)}_-(\vv_{+,c}) \cup \cup_{c=1}^{k_-}  S^{(0)}_-(\vv_{-,c}) \cup S^{(0)}_-(\vv_{+}) \cup S^{(0)}_-(\vv_{-})$,
\begin{equation} 
\begin{aligned}
    F_-^{(t)}(\mX_{\text{hard}}) 
    \ge & \sum_{(+,r)\in S^{*(0)}_-(\vv_{-})} \sum_{p\in\calP(\mX_{\text{hard}};\vv_{-})} \sigma\left(\langle \vw_{-,r}^{(0)} + \sum_{\tau=0}^{t-1} \Delta \vw_{-,r}^{(\tau)}, \alpha_{p}^{\dagger}\vv_{-} + \vzeta_{p} \rangle + b_{+,r}^{(t)} \right)\\
    & + \sum_{(+,r)\notin S_-^{(0)}} \sigma\left(\langle \vw_{-,r}^{(0)}, \vzeta^* \rangle + b_{+,r}^{(0)} \right),
\end{aligned}
\end{equation}
note that the last line is true because neurons outside the set $S^{(0)}_-$ cannot receive any update during training with probability at least $1 - O\left(\frac{mNPk_+t}{\poly(d)} \right)$ due to Theorem \ref{thm: sgd, universal nonact properties}. Estimating the activation value of the neurons from $S^{*(0)}_-(\vv_{-})$ on the feature noise patches requires some care. We define time $t_-$ to be the first point in time such that any $(-,r^*)\in S^{*(0)}_-(\vv_{-})$ satisfies $\sum_{\tau=0}^{t_-}\Delta A_{-,r^*}^{(\tau)} \ge \sigma_0 \ln^5(d)$, and beyond this point in time, i.e. for $t \in [t_-, T_1]$, the neurons in $S^{*(0)}_-(\vv_{-})$ have to activate with high probability, since
\begin{equation}
\begin{aligned}
    \langle \vw_{-,r}^{(0)} + \sum_{\tau=0}^{t-1} \Delta \vw_{-,r}^{(\tau)}, \alpha_{p}^{\dagger}\vv_{-} + \vzeta_{p} \rangle + b_{+,r}^{(t)} 
    \ge & \left( 1 - O\left(\frac{1}{\ln^5(d)}\right) \right)\sigma_0 \ln^5(d)/\ln^4(d) - O(\sigma_0\sqrt{\ln(d)}) \\
    > & 0.
\end{aligned}
\end{equation}

Now we can proceed to prove the lemma for $t \in (0, T_1]$ by combining the above estimates for $F_+^{(t)} (\mX_{\text{hard}})$ and $F_-^{(t)}(\mX_{\text{hard}})$.

For $t \in (0, t_-]$, relying argument similar to the situation of $t = 0$ and the fact that $m - \vert S_-^{(0)} \vert = (1 - o(1))m$,
\begin{equation} 
\begin{aligned}
    & \Bigg\{\sum_{(+,r)\notin S_-^{(0)}} \mathbbm{1}\{\langle \vw_{-,r}^{(0)}, \vzeta^*\rangle + b_{-,r}^{(0)} > 0\}\langle \vw_{-,r}^{(0)}, \vzeta^*\rangle \\
    & - \sum_{r=1}^m\mathbbm{1}\{\langle \vw_{+,r}^{(0)}, \vzeta^*\rangle + b_{+,r}^{(0)} > 0\}\langle \vw_{+,r}^{(0)}, \vzeta^*\rangle\Bigg\}(1\pm o(1)) > 0 \\
    \implies & F_-^{(t)}(\mX_{\text{hard}}) - F_+^{(t)} (\mX_{\text{hard}}) > 0 \\
\end{aligned}
\end{equation}
which has to be true with probability $\Omega(1)$.

On the other hand, with $t \in (t_-, T_1]$, we have
\begin{equation} 
\begin{aligned}
    & F_-^{(t)}(\mX_{\text{hard}}) - F_+^{(t)} (\mX_{\text{hard}})\\
    \ge  & \Bigg\{\sum_{\tau=0}^{t-1}\left( 1 - O\left(\frac{1}{\ln^5(d)}\right) \right) s^{\dagger} |S^{*(0)}_-(\vv_{-})| \Delta A_{-,r^*}^{(\tau)} - O(\sigma_0\sqrt{\ln(d)}) \\
    & - O\left(\frac{1}{k_+} s^* \left\vert S^{(0)}_+(\vv_{+,c}) \right\vert \sum_{\tau=0}^{t-1} \Delta A_{+,r^*}^{(\tau)}\right) \Bigg\}\\
    & + \Bigg\{ \sum_{(+,r)\notin S_-^{(0)}} \sigma\left(\langle \vw_{-,r}^{(0)}, \vzeta^* \rangle + b_{+,r}^{(0)} \right) - \sum_{(+,r)\in S^{(0)}_+(\vv_{+,c})}  \sum_{p\in\calP(\mX_{\text{hard}};\vv_{+,c})} \left\vert \langle \vw_{+,r}^{(0)}, \sqrt{1\pm\iota}\vv_{+,c} + \vzeta_{p} \rangle \right\vert\\
    & - \sum_{r\in[m] } \sigma\left( \langle \vw_{+,r}^{(0)} , \vzeta^* \rangle + b_{+,r}^{(0)}\right) - \sum_{(+,r)\in S^{(0)}_+(\vv_{-})} \sum_{p\in\calP(\mX_{\text{hard}};\vv_{-})} \sigma\left( \langle \vw_{+,r}^{(0)}, \alpha_{p}^{\dagger}\vv_{-} + \vzeta_{p} \rangle + b_{+,r}^{(0)}\right) \Bigg\}
\end{aligned}
\end{equation}

Let us begin analyzing the first $\{\cdot\}$ bracket. 

By Proposition \ref{prop: init geometry, coarse} we know that $\left\vert S^{*(0)}_-(\vv_{-})\right\vert = (1\pm O(1/\ln^5(d)))\left\vert S^{(0)}_+(\vv_{+,c}) \right\vert$, and by Lemma \ref{lemma: sgd phase II, common vs finegrained ratio}, we know that $\Delta A_{+,r^*}^{(\tau)} \le O(\ln(d)\Delta A_{-,r^*}^{(\tau)})$, therefore, 
\begin{equation}
\begin{aligned}
    O\left(\frac{1}{k_+} s^* \left\vert S^{(0)}_+(\vv_{+,c}) \right\vert \sum_{\tau=0}^{t-1} \Delta A_{+,r^*}^{(\tau)}\right) 
    \le & O\left(\frac{\ln(d)}{k_+} s^* \left\vert S^{*(0)}_-(\vv_{-}) \right\vert \sum_{\tau=0}^{t-1} \Delta A_{-,r^*}^{(\tau)}\right) \\
    \ll & \sum_{\tau=0}^{t-1}\left( 1 - O\left(\frac{1}{\ln^5(d)}\right) \right) s^{\dagger} |S^{*(0)}_-(\vv_{-})| \Delta A_{-,r^*}^{(\tau)} - O(\sigma_0\sqrt{\ln(d)})
\end{aligned}
\end{equation}

Therefore, we obtained the simpler lower bound
\begin{equation} 
\begin{aligned}
    & F_-^{(t)}(\mX_{\text{hard}}) - F_+^{(t)} (\mX_{\text{hard}})\\
    \ge & \Bigg\{ \sum_{(+,r)\notin S_-^{(0)}} \sigma\left(\langle \vw_{-,r}^{(0)}, \vzeta^* \rangle + b_{+,r}^{(0)} \right)  - \sum_{(+,r)\in S^{(0)}_+(\vv_{+,c})}  \sum_{p\in\calP(\mX_{\text{hard}};\vv_{+,c})} \left\vert \langle \vw_{+,r}^{(0)}, \sqrt{1\pm\iota}\vv_{+,c} + \vzeta_{p} \rangle \right\vert \\
    & - \sum_{r\in[m] } \sigma\left( \langle \vw_{+,r}^{(0)} , \vzeta^* \rangle + b_{+,r}^{(0)}\right) - \sum_{(+,r)\in S^{(0)}_+(\vv_{-})} \sum_{p\in\calP(\mX_{\text{hard}};\vv_{-})} \sigma\left( \langle \vw_{+,r}^{(0)}, \alpha_{p}^{\dagger}\vv_{-} + \vzeta_{p} \rangle + b_{+,r}^{(0)}\right) \Bigg\}
\end{aligned}
\end{equation}
which is greater than $0$ with probability $\Omega(1)$ (by relying on an argument almost identical to the $t=0$ case again, and noting that $m - |S_-^{(0)}| = (1 - o(1))m$). This concludes the proof.

\end{proof}

\begin{lemma}[Probability of mistake on easy samples is low after training]
\label{lemma: coarse, easy sample mistake prob}
For $t\in[T_{1,1},T_1]$, given an easy test sample $(\mX_{\text{easy}},y)$,
\begin{equation}
    \mathbb{P}\left[F_y^{(T)}(\mX_{\text{easy}}) \le F_{y'}^{(T)}(\mX_{\text{easy}})\right] \le o(1).
\end{equation}
\end{lemma}
\begin{proof}
    Without loss of generality, assume the true label of $\mX_{\text{easy}}$ is $+1$. Assume $t \ge T_{1,1}$.
    
    Firstly, conditioning on the events of Theorem \ref{thm: sgd, universal nonact properties}, the following upper bound on $F_-^{(t)}(\mX_{\text{easy}})$ holds with probability at least $1 - O\left(\frac{m}{\poly(d)} \right)$:
    \begin{equation}
    \begin{aligned}
        F_-^{(t)}(\mX_{\text{easy}}) 
        = & \sum_{(-,r)\in S^{(0)}_-(\vv_{+})}  \sum_{p\in\calP(\mX_{\text{easy}};\vv_{+})} \sigma\left( \langle \vw_{-,r}^{(t)}, \sqrt{1\pm\iota}\vv_{+} + \vzeta_{p} \rangle + b_{-,r}^{(t)}\right) \\
        & + \sum_{(-,r)\in S^{(0)}_-(\vv_{+,c})}  \sum_{p\in\calP(\mX_{\text{easy}};\vv_{+,c})} \sigma\left( \langle \vw_{-,r}^{(t)}, \sqrt{1\pm\iota}\vv_{+,c} + \vzeta_{p} \rangle + b_{-,r}^{(t)}\right) \\
        \le & \sum_{(-,r)\in S^{(0)}_-(\vv_{+})}  \sum_{p\in\calP(\mX_{\text{easy}};\vv_{+})} \sigma\left( \langle \vw_{-,r}^{(0)}, \sqrt{1\pm\iota}\vv_{+} + \vzeta_{p} \rangle + b_{-,r}^{(0)}\right) \\
        & + \sum_{(-,r)\in S^{(0)}_-(\vv_{+,c})}  \sum_{p\in\calP(\mX_{\text{easy}};\vv_{+,c})} \sigma\left( \langle \vw_{-,r}^{(0)}, \sqrt{1\pm\iota}\vv_{+,c} + \vzeta_{p} \rangle + b_{-,r}^{(0)}\right) \\
        < & O\left(s^* d^{c_0} \sigma_0\right) \\
        \le & o(1),
    \end{aligned}
    \end{equation}
    
    and on the other hand,
    \begin{equation}
    \begin{aligned}
        F_+^{(t)}(\mX_{\text{easy}}) 
        \ge & \sum_{(+,r)\in S^{*(0)}_+(\vv_{+})}  \sum_{p\in\calP(\mX_{\text{easy}};\vv_{+})} \sigma\left( \langle \vw_{+,r}^{(t)}, \sqrt{1\pm\iota}\vv_{+} + \vzeta_{p} \rangle + b_{+,r}^{(t)}\right) \\
        & + \sum_{(+,r)\in S^{*(0)}_+(\vv_{+,c})}  \sum_{p\in\calP(\mX_{\text{easy}};\vv_{+,c})} \sigma\left( \langle \vw_{+,r}^{(t)}, \sqrt{1\pm\iota}\vv_{+,c} + \vzeta_{p} \rangle + b_{+,r}^{(t)}\right) \\
        > & \Omega(1).
    \end{aligned}
    \end{equation}

    Therefore, $F_+^{(t)}(\mX_{\text{easy}}) \gg F_-^{(t)}(\mX_{\text{easy}})$, which completes the proof.
\end{proof}

\begin{lemma} [\cite{boas1971}]
\label{lemma: harmonic series}
The partial sum of harmonic series satisfies the following identity:
\begin{equation}
    \sum_{k=1}^{n-1}\frac{1}{k} = \ln(n) + \calE - \frac{1}{2n} - \epsilon_n
\end{equation}
where $\calE$ is the Euler–Mascheroni constant (approximately 0.58), and $\epsilon_n \in [0, 1/8n^2]$.
\end{lemma}

\newpage

\section{Coarse-grained SGD, Poly-time properties}
In this section, set $T_e \in \poly(d)$.

Please note that we are performing stochastic gradient descent on easy samples only.

\begin{theorem}
\label{thm: sgd, universal nonact properties}
Fix any $t\in[0, T_e]$.
\begin{enumerate}
    \item (Non-activation invariance) For any $\tau \ge t$, with probability at least $ 1-O\left( \frac{mk_+NPt}{\poly(d)}\right)$, any feature $\vv \in \{\vv_{+,c}\}_{c=1}^{k_+} \cup \{\vv_{-,c}\}_{c=1}^{k_-}\cup\{\vv_+,\vv_-\}$, any $t'\le t$, $(+,r)\notin S_+^{(0)}(\vv)$ and $\vv$-dominated patch sample $\vx_{n,p}^{(\tau)} = \alpha_{n,p}^{(\tau)}\vv + \vzeta_{n,p}^{(\tau)}$, the following holds:
    \begin{equation}
        \sigma\left(\langle \vw_{+,r}^{(t')}, \vx_{n,p}^{(\tau)} \rangle + b_{+,r}^{(t')}\right) = 0
    \end{equation}
    
    \item (Non-activation on noise patches) For any $\tau \ge t$, with probability at least $ 1-O\left( \frac{mNPt}{\poly(d)}\right)$, for every $t'\le t$, $r\in[m]$ and noise patch $\vx_{n,p}^{(\tau)} = \vzeta_{n,p}^{(\tau)}$, the following holds: 
    \begin{equation}
        \sigma\left(\langle \vw_{+,r}^{(t')}, \vx_{n,p}^{(\tau)} \rangle + b_{+,r}^{(t')}\right) = 0
    \end{equation}

    \item (Off-diagonal nonpositive growth) For any $\tau \ge t$, with probability at least $ 1-O\left( \frac{mk_+NPt}{\poly(d)}\right)$, for any $t'\le t$, any feature $\vv \in \{\vv_{-,c}\}_{c=1}^{k_-}\cup\{\vv_-\}$, any $(+,r) \in S^{(0)}_+(\vv)$ and $\vv$-dominated patch $\vx_{n,p}^{(\tau)} = \alpha_{n,p}^{(\tau)}\vv + \vzeta_{n,p}^{(\tau)}$, $\sigma\left(\langle \vw_{+,r}^{(t')}, \vx_{n,p}^{(\tau)} \rangle + b_{+,r}^{(t')}\right) \le \sigma\left(\langle \vw_{+,r}^{(0)}, \vx_{n,p}^{(\tau)} \rangle + b_{+,r}^{(0)}\right)$.
\end{enumerate}
\end{theorem}
\begin{proof}

\textcolor{blue}{\textbf{Base case $t = 0$.}}

\textcolor{brown}{\textit{1. (Nonactivation invariance)}}

Choose any $\tau \ge 0$, $\vv^*$ from the set $\{\vv_{+,c}\}_{c=1}^{k_+} \cup \{\vv_{-,c}\}_{c=1}^{k_-}\cup\{\vv_+,\vv_-\}$. We will work with neuron sets in the ``$+$'' class in this proof; the ``$-$''-class case can be handled in the same way.
    
First, we need to show that, for every $n$ such that $\vert \calP(\mX_n^{(\tau)}; \vv^*)\vert > 0$ and $p\in\calP(\mX_n^{(\tau)}; \vv^*)$, for every $(+,r)$ neuron index, 
\begin{equation}
    \langle \vw_{+,r}^{(0)}, \vv^* \rangle < \sigma_0 \sqrt{4 + 2c_0} \sqrt{\ln(d) - \frac{1}{\ln^5(d)}} \implies \sigma\left(\langle \vw_{+,r}^{(0)}, \vx_{n,p}^{(\tau)}\rangle + b_{+,r}^{(0)} \right) = 0
\end{equation}

This is indeed true. The following holds with probability at least $1 - O\left(\frac{mNP}{\poly(d)}\right)$ for all $(+,r)\notin S_+^{(0)}(\vv)$ and all such $\vx_{n,p}^{(\tau)}$:
\begin{equation}
\begin{aligned}
    \langle \vw_{+,r}^{(0)}, \vx_{n,p}^{(\tau)}\rangle + b_{+,r}^{(0)} 
    \le & \sigma_0 \sqrt{1+\iota} \sqrt{(4+2c_0)(\ln(d) - 1/\ln^5(d))} + O\left(\frac{\sigma_0}{\ln^{9}(d)}\right) - \sqrt{4 + 2c_0}\sqrt{\ln(d) } \sigma_0\\
    = & \sigma_0 \left(\frac{(4+2c_0)(1 + \iota)(\ln(d) - 1/\ln^5(d)) - (4+2c_0)\ln(d)}{\sqrt{(4+2c_0)(\ln(d) - 1/\ln^5(d))} + \sqrt{4 + 2c_0}\sqrt{\ln(d) }}   + O\left(\frac{1}{\ln^{9}(d)}\right)\right) \\
    = & \sigma_0 \left(\frac{(4+2c_0)\iota\ln(d) - (1+\iota)/\ln^5(d)}{\sqrt{(4+2c_0)(\ln(d) - 1/\ln^5(d))} + \sqrt{4 + 2c_0}\sqrt{\ln(d) }}   + O\left(\frac{1}{\ln^{9}(d)}\right)\right) \\
    < & 0,
\end{aligned}
\end{equation}

The first equality holds by utilizing the identity $a - b = \frac{a^2 - b^2}{a + b}$. As a consequence, $\sigma(\langle \vw_{+,r}^{(0)}, \vx_{n,p}^{(\tau)}\rangle + b_{+,r}^{(0)}) = 0$.

\textcolor{brown}{\textit{2. (Non-activation on noise patches)}}
Invoking Lemma \ref{lemma: independent gaussian vector inner product concentration}, for any $\tau \ge 0$, with probability at least $1-O\left(\frac{mNP}{\poly(d)}\right)$, we have for all possible choices of $r\in[m]$ and the noise patches $\vx_{n,p}^{(\tau)} = \vzeta_{n,p}^{(\tau)}$:
\begin{equation}
    \left\vert \langle \vw_{+,r}^{(0)}, \vzeta_{n,p}^{(\tau)} \rangle \right\vert \le O(\sigma_0\sigma_{\zeta} \sqrt{d\ln(d)}) \le O\left(\frac{\sigma_0}{\ln^{9}(d)}\right) \ll b_{+,r}^{(0)}.
\end{equation}

Therefore, no neuron can activate on the noise patches at time $t=0$.

\textcolor{brown}{\textit{3. (Off-diagonal nonpositive growth)}}
This point is trivially true at $t=0$.

\vspace{4ex}
\textcolor{blue}{\textbf{Inductive step}}: we assume the induction hypothesis for $t\in [0, T]$ (with $T < T_e$ of course), and prove the statements for $t = T+1$.

\textcolor{brown}{\textit{1. (Nonactivation invariance)}}

Choose any $\vv^*$ from the set $\{\vv_{+,c}\}_{c=1}^{k_+} \cup \{\vv_{-,c}\}_{c=1}^{k_-}\cup\{\vv_+,\vv_-\}$. We will work with neuron sets in the ``$+$'' class in this proof; the ``$-$''-class case can be handled in the same way. 

We need to prove that given $\tau \ge T+1$, with probability at least $ 1-O\left( \frac{mNP(T+1)}{\poly(d)}\right)$, for every $t'\le T+1$, $(+,r)$ neuron index and $\vv^*$-dominated patch $\vx_{n,p}^{(\tau)}$,
\begin{equation}
    (+,r)\notin S_+^{(0)}(\vv^*) \implies \sigma\left(\langle \vw_{+,r}^{(t')}, \vx_{n,p}^{(\tau)}\rangle + b_{+,r}^{(t')} \right) = 0.
\end{equation}

Conditioning on the (high-probability) event of the induction hypothesis of point 1., the following is already true on all the $\vv^*$-dominated patches at time $t' \le T$:
\begin{equation}
    (+,r)\notin S_+^{(0)}(\vv^*) \implies \sigma\left(\langle \vw_{+,r}^{(t')}, \vx_{n,p}^{(T)}\rangle + b_{+,r}^{(t')} \right) = 0.
\end{equation}
In particular, $\sigma\left(\langle \vw_{+,r}^{(T)}, \vx_{n,p}^{(T)}\rangle + b_{+,r}^{(T)} \right) = 0$.

In other words, no $(+,r)\notin S_+^{(0)}(\vv^*)$ can be updated on the $\vv^*$-dominated patches at time $t=T$. Furthermore, the induction hypothesis of point 2. also states that the network cannot activate on any noise patch $\vx_{n,p}^{(T)} = \vzeta_{n,p}^{(T)}$ with probability at least $ 1-O\left( \frac{mNPT}{\poly(d)}\right)$. Therefore, the neuron update for those $(+,r)\notin S_+^{(0)}(\vv^*)$ takes the form
\begin{equation}
\begin{aligned}
    \Delta \vw_{+,r}^{(T)} 
    = & \frac{\eta}{NP} \sum_{\vv\in\calC(\vv^*)}\sum_{n=1}^N \mathbbm{1}\{|\calP(\mX_n^{(T)}; \vv)|>0\} [\mathbbm{1}\{y_n=+\}-\text{logit}_+^{(T)}(\mX_n^{(T)})] \\
    & \times  \sum_{p\in\calP(\mX_n^{(T)}; \vv)} \mathbbm{1}\{\langle \vw_{+,r}^{(T)}, \alpha_{n,p}^{(T)} \vv +\vzeta_{n,p}^{(T)}\rangle + b_{c,r}^{(T)} > 0\} \left(\alpha_{n,p}^{(T)} \vv +\vzeta_{n,p}^{(T)}\right)
\end{aligned}
\end{equation}

Now we can invoke Lemma \ref{lemma: sgd phase 1, nonact invar, t=1} and obtain that, with probability at least $ 1-O\left( \frac{mNP}{\poly(d)}\right)$, the following holds for all relevant neurons and $\vv^*$-dominated patches:
\begin{equation}
    \langle \Delta \vw_{+,r}^{(T)} , \vx_{n,p}^{(\tau)} \rangle + \Delta b_{+,r}^{(T)} < 0.
\end{equation}

In conclusion, with $\tau \ge T+1$, with probability at least $ 1-O\left( \frac{mNP}{\poly(d)}\right)$, for every $(+,r)\notin S_+^{(0)}(\vv^*)$ and relevant $(n,p)$'s,
\begin{equation}
    \langle \vw_{+,r}^{(T)} + \Delta \vw_{+,r}^{(T)}, \vx_{n,p}^{(\tau)}\rangle + b_{+,r}^{(T)} + \Delta b_{+,r}^{(T)} = \langle \vw_{+,r}^{(T+1)}, \vx_{n,p}^{(\tau)}\rangle + b_{+,r}^{(T+1)} < 0,
\end{equation}
which leads to $\langle \vw_{+,r}^{(t')}, \vx_{n,p}^{(\tau)}\rangle + b_{+,r}^{(t')} < 0$ for all $t' \le T+1$ with probability at least $ 1-O\left( \frac{mk_+NP(T+1)}{\poly(d)}\right)$ (also taking union bound over all the possible choices of $\vv^*$). This finishes the inductive step for point 1.

\textcolor{brown}{\textit{2. (Non-activation on noise patches)}}

Relying on the event of the induction hypothesis, for any $\tau \ge T$, the following holds for every $r\in[m]$ and noise patch $\vx_{n,p}^{(\tau)} = \vzeta_{n,p}^{(\tau)}$, 
\begin{equation}
    \langle \vw_{+,r}^{(T)}, \vx_{n,p}^{(\tau)} \rangle + b_{+,r}^{(T)} < 0.
\end{equation}

Conditioning on this high-probability event, this means no neuron $\vw_{+,r}^{(T)}$ can be updated on the noise patches.  Denoting the set of features $\calM=\{\vv_{+,c}\}_{c=1}^{k_+} \cup \{\vv_{-,c}\}_{c=1}^{k_-}\cup\{\vv_+,\vv_-\}$, for every $r\in[m]$, its update is reduced to
\begin{equation}
\begin{aligned}
    \Delta \vw_{+,r}^{(T)} 
    = & \frac{\eta}{NP} \sum_{\vv\in\calM}\sum_{n=1}^N \mathbbm{1}\{|\calP(\mX_n^{(T)}; \vv)|>0\} [\mathbbm{1}\{y_n=+\}-\text{logit}_+^{(T)}(\mX_n^{(T)})] \\
    & \times  \sum_{p\in\calP(\mX_n^{(T)}; \vv)} \mathbbm{1}\{\langle \vw_{+,r}^{(T)}, \alpha_{n,p}^{(T)} \vv +\vzeta_{n,p}^{(T)}\rangle + b_{c,r}^{(T)} > 0\} \left(\alpha_{n,p}^{(T)} \vv +\vzeta_{n,p}^{(T)}\right),
\end{aligned}
\end{equation}

Invoking Lemma \ref{lemma: sgd phase 1, nonact on noise}, we have that, for any  $\tau \ge T+1$, the following inequality holds with probability at least $ 1-O\left( \frac{mNP}{\poly(d)}\right)$ for every $r\in[m]$ and noise patches, 
\begin{equation}
    \langle \Delta \vw_{+,r}^{(T)} , \vx_{n,p}^{(\tau)} \rangle + \Delta b_{+,r}^{(T)} < 0.
\end{equation}

Consequently, for any $\tau \ge T+1$, the following inequality holds with probability at least $ 1-O\left( \frac{mNP}{\poly(d)}\right)$ for every $r\in[m]$ and noise patches $\vx_{n,p}^{(\tau)} = \vzeta_{n,p}^{(\tau)}$:
\begin{equation}
    \langle \vw_{+,r}^{(T)} + \Delta \vw_{+,r}^{(T)}, \vx_{n,p}^{(\tau)}\rangle + b_{+,r}^{(T)} + \Delta b_{+,r}^{(T)} = \langle \vw_{+,r}^{(T+1)}, \vx_{n,p}^{(\tau)}\rangle + b_{+,r}^{(T+1)} < 0.
\end{equation}
 This finishes the inductive step for point 2.

\textcolor{brown}{\textit{3. (Off-diagonal nonpositive growth)}}
Choose any $\vv^*\in\{\vv_-\}\cup\{\vv_{-,c}\}_{c=1}^{k_-}$.

Choose any neuron with index $(+,r)$. Similar to our proof for point 2., we know that its update, when taken inner product with a $\vv^*$-dominated patch $\vx_{n,p}^{(\tau)} = \sqrt{1\pm\iota}\vv^* + \vzeta_{n,p}^{(\tau)}$, has to take the form
\begin{equation}
\begin{aligned}
    & \langle \Delta \vw_{+,r}^{(T)} , \sqrt{1\pm\iota}\vv^* + \vzeta_{n,p}^{(\tau)} \rangle \\
    = & \frac{\eta}{NP} \sum_{\vv\in\calM}\sum_{n=1}^N \mathbbm{1}\{|\calP(\mX_n^{(T)}; \vv)|>0\} [\mathbbm{1}\{y_n=+\}-\text{logit}_+^{(T)}(\mX_n^{(T)})] \\
    & \times  \sum_{p\in\calP(\mX_n^{(T)}; \vv)} \mathbbm{1}\{\langle \vw_{+,r}^{(T)}, \alpha_{n,p}^{(T)} \vv +\vzeta_{n,p}^{(T)}\rangle + b_{+,r}^{(T)} > 0\} \langle \alpha_{n,p}^{(T)} \vv +\vzeta_{n,p}^{(T)}, \sqrt{1\pm\iota}\vv^* + \vzeta_{n,p}^{(\tau)} \rangle \\
    = & \frac{\eta}{NP} \sum_{\vv\in\calM-\{\vv^*\}}\sum_{n=1}^N \mathbbm{1}\{|\calP(\mX_n^{(T)}; \vv)|>0\} [\mathbbm{1}\{y_n=+\}-\text{logit}_+^{(T)}(\mX_n^{(T)})] \\
    & \times  \sum_{p\in\calP(\mX_n^{(T)}; \vv)} \mathbbm{1}\{\langle \vw_{+,r}^{(T)}, \alpha_{n,p}^{(T)} \vv +\vzeta_{n,p}^{(T)}\rangle + b_{+,r}^{(T)} > 0\} \left(\langle \vzeta_{n,p}^{(T)}, \sqrt{1\pm\iota}\vv^* \rangle + \langle \alpha_{n,p}^{(T)} \vv +\vzeta_{n,p}^{(T)}, \vzeta_{n,p}^{(\tau)} \rangle\right)\\
    & - \frac{\eta}{NP} \sum_{n=1}^N \mathbbm{1}\{|\calP(\mX_n^{(T)}; \vv^*)|>0\} [\text{logit}_+^{(T)}(\mX_n^{(T)})] \\
    & \times  \sum_{p\in\calP(\mX_n^{(T)}; \vv)} \mathbbm{1}\{\langle \vw_{+,r}^{(T)}, \alpha_{n,p}^{(T)} \vv +\vzeta_{n,p}^{(T)}\rangle + b_{+,r}^{(T)} > 0\} \langle \alpha_{n,p}^{(T)} \vv^* +\vzeta_{n,p}^{(T)}, \sqrt{1\pm\iota}\vv^* + \vzeta_{n,p}^{(\tau)} \rangle
\end{aligned}
\end{equation}

With probability at least $ 1-O\left( \frac{NP}{\poly(d)}\right)$, $\langle \alpha_{n,p}^{(T)} \vv^* +\vzeta_{n,p}^{(T)}, \sqrt{1\pm\iota}\vv^* + \vzeta_{n,p}^{(\tau)} \rangle > 0$, and $\langle \vzeta_{n,p}^{(T)}, \sqrt{1\pm\iota}\vv^* \rangle + \langle \alpha_{n,p}^{(T)} \vv +\vzeta_{n,p}^{(T)}, \vzeta_{n,p}^{(\tau)} \rangle < O(1/\ln^9(d))$. Therefore, 
\begin{equation}
\begin{aligned}
    \langle \Delta \vw_{+,r}^{(T)} , \vv^* \rangle
    < & \frac{\eta}{NP} \sum_{\vv\in\calM-\{\vv^*\}}\sum_{n=1}^N \mathbbm{1}\{|\calP(\mX_n^{(T)}; \vv)|>0\} [\mathbbm{1}\{y_n=+\}-\text{logit}_+^{(T)}(\mX_n^{(T)})] \\
    & \times  \sum_{p\in\calP(\mX_n^{(T)}; \vv)} \mathbbm{1}\{\langle \vw_{+,r}^{(T)}, \alpha_{n,p}^{(T)} \vv +\vzeta_{n,p}^{(T)}\rangle + b_{+,r}^{(T)} > 0\} O\left( \frac{1}{\ln^9(d)}\right)
\end{aligned}
\end{equation}
Invoking Lemma \ref{lemma: sgd phase 1, nonact on noise}, we know that
\begin{equation}
\begin{aligned}
    &\Delta b_{+,r}^{(T)} \\
    \le & -\frac{1}{\ln^5(d)} \frac{\eta}{NP}\left(\sqrt{1-\iota} - \frac{1}{\ln^9(d)}\right) \\
    & \times \Bigg( \sum_{\vv\in\calM}\sum_{n=1}^N \mathbbm{1}\{|\calP(\mX_n^{(T)}; \vv)|>0\} \left\vert\mathbbm{1}\{y_n=+\}-\text{logit}_+^{(T)}(\mX_n^{(T)})\right\vert \\
    & \times \sum_{p\in\calP(\mX_n^{(T)}; \vv)} \mathbbm{1}\{\langle \vw_{+,r}^{(T)}, \alpha_{n,p}^{(T)} \vv +\vzeta_{n,p}^{(T)}\rangle + b_{+,r}^{(T)} > 0\} \Bigg).
\end{aligned}
\end{equation}

It follows that 
\begin{equation}
\begin{aligned}
    & \langle \Delta \vw_{+,r}^{(T)} , \sqrt{1\pm\iota}\vv^* + \vzeta_{n,p}^{(\tau)}\rangle + \Delta b_{+,r}^{(T)} \\
    < & O\left( \frac{1}{\ln^9(d)}\right)\frac{\eta}{NP} \Bigg(\sum_{\vv\in\calM-\{\vv^*\}}\sum_{n=1}^N \mathbbm{1}\{|\calP(\mX_n^{(T)}; \vv)|>0\} [\mathbbm{1}\{y_n=+\}-\text{logit}_+^{(T)}(\mX_n^{(T)})] \\
    & \times  \sum_{p\in\calP(\mX_n^{(T)}; \vv)} \mathbbm{1}\{\langle \vw_{+,r}^{(T)}, \alpha_{n,p}^{(T)} \vv +\vzeta_{n,p}^{(T)}\rangle + b_{c,r}^{(T)} > 0\} \Bigg)  \\
    & - \Omega\left( \frac{1}{\ln^5(d)}\right)  \frac{\eta}{NP}  \Bigg( \sum_{\vv\in\calM}\sum_{n=1}^N \mathbbm{1}\{|\calP(\mX_n^{(T)}; \vv)|>0\} \left\vert\mathbbm{1}\{y_n=+\}-\text{logit}_+^{(T)}(\mX_n^{(T)})\right\vert \\
    & \times \sum_{p\in\calP(\mX_n^{(T)}; \vv)} \mathbbm{1}\{\langle \vw_{+,r}^{(T)}, \alpha_{n,p}^{(T)} \vv +\vzeta_{n,p}^{(T)}\rangle + b_{c,r}^{(T)} > 0\} \Bigg) \\
    < & 0.
\end{aligned}
\end{equation}

Consequently, 
\begin{equation}
\begin{aligned}
    &\sigma\left( \langle \vw_{+,r}^{(T+1)} , \sqrt{1\pm\iota}\vv^* + \vzeta_{n,p}^{(\tau)}\rangle + b_{+,r}^{(T+1)}\right) \\
    = & \sigma\left(\langle  \vw_{+,r}^{(T)} , \sqrt{1\pm\iota}\vv^* + \vzeta_{n,p}^{(\tau)}\rangle +  b_{+,r}^{(T)} + \langle \Delta \vw_{+,r}^{(T)} , \sqrt{1\pm\iota}\vv^* + \vzeta_{n,p}^{(\tau)}\rangle + \Delta b_{+,r}^{(T)} \right) \\
    \le & \sigma\left(\langle  \vw_{+,r}^{(T)} , \sqrt{1\pm\iota}\vv^* + \vzeta_{n,p}^{(\tau)}\rangle +  b_{+,r}^{(T)} \right) \\
    \le & \sigma\left(\langle  \vw_{+,r}^{(0)} , \sqrt{1\pm\iota}\vv^* + \vzeta_{n,p}^{(\tau)}\rangle +  b_{+,r}^{(0)} \right).
\end{aligned}
\end{equation}
\end{proof}

\begin{corollary}[Bias update upper bound]
\label{coro: coarse training, bias upper bd}
Choose any $T_e \le \poly(d)$. With probability at least $ 1-O\left( \frac{mk_+NPT_e}{\poly(d)}\right)$, for all $t \in [0,T_e]$, any neuron $\vw_{+,r}$, and any $\vv\in\calU_{+,r}^{(0)}$,
\begin{equation}
    \Delta b_{+,r}^{(t)} < -\Omega\left(\frac{\polyln(d)}{\ln^5(d)}\right)\left\vert\langle \Delta \vw_{+,r}^{(t)}, \vzeta^* \rangle \right\vert.
\end{equation}
    
\end{corollary}
\begin{proof}
    Conditioning on the high-probability events of Theorem \ref{thm: sgd, universal nonact properties} above, we know that for any neuron indexed $(+,r)$, at any time $t \le T_e$, its update takes the form 
    \begin{equation}
    \begin{aligned}
        \Delta \vw_{+,r}^{(t)} 
        = & \frac{\eta}{NP} \sum_{\vv\in\calU_{+,r}^{(0)}}\sum_{n=1}^N \mathbbm{1}\{|\calP(\mX_n^{(t)}; \vv)|>0\} [\mathbbm{1}\{y_n=+\}-\text{logit}_+^{(t)}(\mX_n^{(t)})] \\
        & \times  \sum_{p\in\calP(\mX_n^{(t)}; \vv)} \mathbbm{1}\{\langle \vw_{+,r}^{(t)}, \alpha_{n,p}^{(t)} \vv +\vzeta_{n,p}^{(t)}\rangle + b_{c,r}^{(t)} > 0\} \left(\alpha_{n,p}^{(t)} \vv +\vzeta_{n,p}^{(t)}\right),
    \end{aligned}
    \end{equation}

    It follows that, with probability at least $ 1-O\left( \frac{1}{\poly(d)}\right)$,
    \begin{equation}
    \begin{aligned}
        \left\vert \langle \Delta \vw_{+,r}^{(t)}, \vzeta^* \rangle \right\vert
        = & \Bigg\vert \frac{\eta}{NP} \sum_{\vv\in\calU_{+,r}^{(0)}}\sum_{n=1}^N \mathbbm{1}\{|\calP(\mX_n^{(t)}; \vv)|>0\} [\mathbbm{1}\{y_n=+\}-\text{logit}_+^{(t)}(\mX_n^{(t)})] \\
        & \times  \sum_{p\in\calP(\mX_n^{(t)}; \vv)} \mathbbm{1}\{\langle \vw_{+,r}^{(t)}, \alpha_{n,p}^{(t)} \vv +\vzeta_{n,p}^{(t)}\rangle + b_{c,r}^{(t)} > 0\} \langle \alpha_{n,p}^{(t)} \vv +\vzeta_{n,p}^{(t)}, \vzeta^* \rangle \Bigg\vert \\
        \le &  \frac{\eta}{NP} \sum_{\vv\in\calU_{+,r}^{(0)}}\sum_{n=1}^N \mathbbm{1}\{|\calP(\mX_n^{(t)}; \vv)|>0\} \left\vert\mathbbm{1}\{y_n=+\}-\text{logit}_+^{(t)}(\mX_n^{(t)})\right\vert \\
        & \times  \sum_{p\in\calP(\mX_n^{(t)}; \vv)} \mathbbm{1}\{\langle \vw_{+,r}^{(t)}, \alpha_{n,p}^{(t)} \vv +\vzeta_{n,p}^{(t)}\rangle + b_{c,r}^{(t)} > 0\} O\left(\frac{1}{\polyln(d)} \right)
    \end{aligned}
    \end{equation}

    On the other hand, 
    \begin{equation}
    \begin{aligned}
        \left\| \Delta \vw_{+,r}^{(t)} \right\|_2
        \ge & \Bigg\| \frac{\eta}{NP} \sum_{\vv\in\calU_{+,r}^{(0)}}\sum_{n=1}^N \mathbbm{1}\{|\calP(\mX_n^{(t)}; \vv)|>0\} [\mathbbm{1}\{y_n=+\}-\text{logit}_+^{(t)}(\mX_n^{(t)})] \\
        & \times  \sum_{p\in\calP(\mX_n^{(t)}; \vv)} \mathbbm{1}\{\langle \vw_{+,r}^{(t)}, \alpha_{n,p}^{(t)} \vv +\vzeta_{n,p}^{(t)}\rangle + b_{c,r}^{(t)} > 0\} \alpha_{n,p}^{(t)} \vv  \Bigg\|_2 \\
        & - \Bigg\| \frac{\eta}{NP} \sum_{\vv\in\calU_{+,r}^{(0)}}\sum_{n=1}^N \mathbbm{1}\{|\calP(\mX_n^{(t)}; \vv)|>0\} [\mathbbm{1}\{y_n=+\}-\text{logit}_+^{(t)}(\mX_n^{(t)})] \\
        & \times  \sum_{p\in\calP(\mX_n^{(t)}; \vv)} \mathbbm{1}\{\langle \vw_{+,r}^{(t)}, \alpha_{n,p}^{(t)} \vv +\vzeta_{n,p}^{(t)}\rangle + b_{c,r}^{(t)} > 0\} \vzeta_{n,p}^{(t)} \Bigg\|_2  \\
        \ge & \frac{\eta}{NP} \sum_{\vv\in\calU_{+,r}^{(0)}}\sum_{n=1}^N \mathbbm{1}\{|\calP(\mX_n^{(t)}; \vv)|>0\} \left\vert\mathbbm{1}\{y_n=+\}-\text{logit}_+^{(t)}(\mX_n^{(t)})\right\vert \\
        & \times  \sum_{p\in\calP(\mX_n^{(t)}; \vv)} \mathbbm{1}\{\langle \vw_{+,r}^{(t)}, \alpha_{n,p}^{(t)} \vv +\vzeta_{n,p}^{(t)}\rangle + b_{c,r}^{(t)} > 0\} \left(\sqrt{1 - \iota} - O\left(\frac{1}{\ln^9(d)} \right) \right)
    \end{aligned}
    \end{equation}

    Clearly,
    \begin{equation}
        \left\| \Delta \vw_{+,r}^{(t)} \right\|_2 \ge \Omega\left(\polyln(d)\left\vert \langle \Delta \vw_{+,r}^{(t)}, \vzeta^* \rangle \right\vert \right).
    \end{equation}

    The conclusion follows.
\end{proof}

\begin{lemma} [Nonactivation invariance]
\label{lemma: sgd phase 1, nonact invar, t=1}
    Let the assumptions in Theorem \ref{prop: phase 1 sgd induction} hold. 
    
    Denote the set of features $\calC(\vv^*)=\{\vv_{+,c}\}_{c=1}^{k_+} \cup \{\vv_{-,c}\}_{c=1}^{k_-}\cup\{\vv_+,\vv_-\} - \{\vv^*\}$. If the update term for neuron $\vw_{+,r}^{(t)}$ can be written as follows
    \begin{equation}
    \begin{aligned}
        \Delta \vw_{+,r}^{(t)} 
        = & \frac{\eta}{NP} \sum_{\vv\in\calC(\vv^*)}\sum_{n=1}^N \mathbbm{1}\{|\calP(\mX_n^{(t)}; \vv)|>0\} [\mathbbm{1}\{y_n=+\}-\text{logit}_+^{(t)}(\mX_n^{(t)})] \\
        & \times  \sum_{p\in\calP(\mX_n^{(t)}; \vv)} \mathbbm{1}\{\langle \vw_{+,r}^{(t)}, \alpha_{n,p}^{(t)} \vv +\vzeta_{n,p}^{(t)}\rangle + b_{c,r}^{(t)} > 0\} \left(\alpha_{n,p}^{(t)} \vv +\vzeta_{n,p}^{(t)}\right),
    \end{aligned}
    \end{equation}
    then given any $\tau > t$, the following inequality holds with probability at least $ 1-O\left( \frac{NP}{\poly(d)}\right)$ for all $\vv^*$-dominated patch $\vx_{n,p}^{(\tau)}$:
    \begin{equation}
        \langle \Delta \vw_{+,r}^{(t)} , \vx_{n,p}^{(\tau)} \rangle + \Delta b_{+,r}^{(t)} < 0
    \end{equation}
\end{lemma}
\begin{proof}

Let us fix a neuron $\vw_{+,r}$ satisfying the update expression in the Lemma statement, and fix some $\tau > t$.

Firstly, the bias update for this neuron can be upper bounded via the reverse triangle inequality:
\begin{equation}
\begin{aligned}
    \Delta b_{+,r}^{(t)} 
    = & -\frac{\left\| \Delta \vw_{+,r}^{(t)} \right\|_2}{\ln^5(d)} \\
    \le & -\frac{1}{\ln^5(d)} \frac{\eta}{NP} \Bigg\|\sum_{\vv\in\calC(\vv^*)}\sum_{n=1}^N \mathbbm{1}\{|\calP(\mX_n^{(t)}; \vv)|>0\} [\mathbbm{1}\{y_n=+\}-\text{logit}_+^{(t)}(\mX_n^{(t)})] \\
    & \times  \sum_{p\in\calP(\mX_n^{(t)}; \vv)} \mathbbm{1}\{\langle \vw_{+,r}^{(t)}, \alpha_{n,p}^{(t)} \vv +\vzeta_{n,p}^{(t)}\rangle + b_{c,r}^{(t)} > 0\} \alpha_{n,p}^{(t)} \vv \Bigg\|_2 \\
    & + \frac{1}{\ln^5(d)} \frac{\eta}{NP} \Bigg\|\sum_{\vv\in\calC(\vv^*)}\sum_{n=1}^N \mathbbm{1}\{|\calP(\mX_n^{(t)}; \vv)|>0\} [\mathbbm{1}\{y_n=+\}-\text{logit}_+^{(t)}(\mX_n^{(t)})] \\
    & \times  \sum_{p\in\calP(\mX_n^{(t)}; \vv)} \mathbbm{1}\{\langle \vw_{+,r}^{(t)}, \alpha_{n,p}^{(t)} \vv +\vzeta_{n,p}^{(t)}\rangle + b_{c,r}^{(t)} > 0\} \vzeta_{n,p}^{(t)} \Bigg\|_2 \\
\end{aligned}
\end{equation}

Let us further upper bound the two $\|\cdot\|_2$ terms separately. Firstly,
\begin{equation}
\begin{aligned}
    & \Bigg\|\sum_{\vv\in\calC(\vv^*)}\sum_{n=1}^N \mathbbm{1}\{|\calP(\mX_n^{(t)}; \vv)|>0\} [\mathbbm{1}\{y_n=+\}-\text{logit}_+^{(t)}(\mX_n^{(t)})] \\
    & \times \sum_{p\in\calP(\mX_n^{(t)}; \vv)} \mathbbm{1}\{\langle \vw_{+,r}^{(t)}, \alpha_{n,p}^{(t)} \vv +\vzeta_{n,p}^{(t)}\rangle + b_{c,r}^{(t)} > 0\} \alpha_{n,p}^{(t)} \vv \Bigg\|_2 \\
    = & \sum_{\vv\in\calC(\vv^*)} \Bigg\|\sum_{n=1}^N \mathbbm{1}\{|\calP(\mX_n^{(t)}; \vv)|>0\} [\mathbbm{1}\{y_n=+\}-\text{logit}_+^{(t)}(\mX_n^{(t)})] \\
    & \times \sum_{p\in\calP(\mX_n^{(t)}; \vv)} \mathbbm{1}\{\langle \vw_{+,r}^{(t)}, \alpha_{n,p}^{(t)} \vv +\vzeta_{n,p}^{(t)}\rangle + b_{c,r}^{(t)} > 0\} \alpha_{n,p}^{(t)} \vv \Bigg\|_2 \\
    = & \sum_{\vv\in\calC(\vv^*)} \sum_{n=1}^N \mathbbm{1}\{|\calP(\mX_n^{(t)}; \vv)|>0\} \left\vert\mathbbm{1}\{y_n=+\}-\text{logit}_+^{(t)}(\mX_n^{(t)})\right\vert \\
    & \times \sum_{p\in\calP(\mX_n^{(t)}; \vv)} \mathbbm{1}\{\langle \vw_{+,r}^{(t)}, \alpha_{n,p}^{(t)} \vv +\vzeta_{n,p}^{(t)}\rangle + b_{c,r}^{(t)} > 0\} \alpha_{n,p}^{(t)} \left\| \vv \right\|_2 \\
    \ge & \sum_{\vv\in\calC(\vv^*)} \sum_{n=1}^N \mathbbm{1}\{|\calP(\mX_n^{(t)}; \vv)|>0\} \left\vert\mathbbm{1}\{y_n=+\}-\text{logit}_+^{(t)}(\mX_n^{(t)})\right\vert \\
    & \times \sum_{p\in\calP(\mX_n^{(t)}; \vv)} \mathbbm{1}\{\langle \vw_{+,r}^{(t)}, \alpha_{n,p}^{(t)} \vv +\vzeta_{n,p}^{(t)}\rangle + b_{c,r}^{(t)} > 0\} \sqrt{1-\iota}
\end{aligned}
\end{equation}

Secondly, with probability at least $1-O\left(\frac{NP}{\poly(d)}\right)$,
\begin{equation}
\begin{aligned}
    & \Bigg\|\sum_{\vv\in\calC(\vv^*)}\sum_{n=1}^N \mathbbm{1}\{|\calP(\mX_n^{(t)}; \vv)|>0\} [\mathbbm{1}\{y_n=+\}-\text{logit}_+^{(t)}(\mX_n^{(t)})] \\
    & \times \sum_{p\in\calP(\mX_n^{(t)}; \vv)} \mathbbm{1}\{\langle \vw_{+,r}^{(t)}, \alpha_{n,p}^{(t)} \vv +\vzeta_{n,p}^{(t)}\rangle + b_{c,r}^{(t)} > 0\} \vzeta_{n,p}^{(t)}\Bigg\|_2 \\
    \le & \sum_{\vv\in\calC(\vv^*)}\sum_{n=1}^N \mathbbm{1}\{|\calP(\mX_n^{(t)}; \vv)|>0\} \left\vert\mathbbm{1}\{y_n=+\}-\text{logit}_+^{(t)}(\mX_n^{(t)})\right\vert \\
    & \times \sum_{p\in\calP(\mX_n^{(t)}; \vv)} \mathbbm{1}\{\langle \vw_{+,r}^{(t)}, \alpha_{n,p}^{(t)} \vv +\vzeta_{n,p}^{(t)}\rangle + b_{c,r}^{(t)} > 0\} \left\| \vzeta_{n,p}^{(t)}\right\|_2 \\
    \le & \sum_{\vv\in\calC(\vv^*)}\sum_{n=1}^N \mathbbm{1}\{|\calP(\mX_n^{(t)}; \vv)|>0\} \left\vert\mathbbm{1}\{y_n=+\}-\text{logit}_+^{(t)}(\mX_n^{(t)})\right\vert \\
    & \times \sum_{p\in\calP(\mX_n^{(0)}; \vv)} \mathbbm{1}\{\langle \vw_{+,r}^{(t)}, \alpha_{n,p}^{(t)} \vv +\vzeta_{n,p}^{(t)}\rangle + b_{c,r}^{(t)} > 0\} \frac{1}{\ln^9(d)}
\end{aligned}
\end{equation}

Therefore, with probability at least $1-O\left(\frac{NP}{\poly(d)}\right)$, we can bound the update to the bias as follows:
\begin{equation}
\begin{aligned}
    &\Delta b_{+,r}^{(t)} \\
    \le & -\frac{1}{\ln^5(d)} \frac{\eta}{NP}\left(\sqrt{1-\iota} - \frac{1}{\ln^9(d)}\right) \\
    & \times \Bigg( \sum_{\vv\in\calC(\vv^*)}\sum_{n=1}^N \mathbbm{1}\{|\calP(\mX_n^{(t)}; \vv)|>0\} \left\vert\mathbbm{1}\{y_n=+\}-\text{logit}_+^{(t)}(\mX_n^{(t)})\right\vert \\
    & \times \sum_{p\in\calP(\mX_n^{(t)}; \vv)} \mathbbm{1}\{\langle \vw_{+,r}^{(t)}, \alpha_{n,p}^{(t)} \vv +\vzeta_{n,p}^{(t)}\rangle + b_{c,r}^{(t)} > 0\} \Bigg)
\end{aligned}
\end{equation}

Furthermore, with probability at least $1 - e^{-\Omega(d) + O(\ln(d))} > 1 - O\left(\frac{NP}{\poly(d)}\right)$, the following holds for all $n,p$:
\begin{equation}
    \langle \alpha_{n,p}^{(t)} \vv, \vzeta_{n,p}^{(\tau)} \rangle,\; \langle \vzeta_{n,p}^{(t)}, \alpha_{n,p}^{(\tau)} \vv^*\rangle,\; \langle \vzeta_{n,p}^{(t)}, \vzeta_{n,p}^{(\tau)} \rangle < O\left( \frac{1}{\ln^9(d)} \right).
\end{equation}

Combining the above derivations, they imply that with probability at least $1-O\left(\frac{NP}{\poly(d)}\right)$, for any $\vx_{n,p}^{(\tau)}$ dominated by $\vv^*$,
\begin{equation}
\begin{aligned}
    &\langle \Delta \vw_{+,r}^{(t)} , \vx_{n,p}^{(\tau)}\rangle + \Delta b_{+,r}^{(t)} \\
    = &\langle \Delta \vw_{+,r}^{(t)} , \alpha_{n,p}^{(\tau)}\vv^* + \vzeta_{n,p}^{(\tau)} \rangle + \Delta b_{+,r}^{(t)} \\
    = & \frac{\eta}{NP} \sum_{\vv\in\calC(\vv^*)}\sum_{n=1}^N \mathbbm{1}\{|\calP(\mX_n^{(t)}; \vv)|>0\} [\mathbbm{1}\{y_n=+\}-\text{logit}_+^{(t)}(\mX_n^{(t)})] \\
    & \times \sum_{p\in\calP(\mX_n^{(t)}; \vv)} \mathbbm{1}\{\langle \vw_{+,r}^{(t)}, \alpha_{n,p}^{(t)} \vv +\vzeta_{n,p}^{(t)}\rangle + b_{c,r}^{(t)} > 0\} \langle \alpha_{n,p}^{(t)} \vv +\vzeta_{n,p}^{(t)}, \alpha_{n,p}^{(\tau)}\vv^* + \vzeta_{n,p}^{(\tau)} \rangle + \Delta b_{+,r}^{(t)} \\
    = & \frac{\eta}{NP} \sum_{\vv\in\calC(\vv^*)}\sum_{n=1}^N \mathbbm{1}\{|\calP(\mX_n^{(t)}; \vv)|>0\} [\mathbbm{1}\{y_n=+\}-\text{logit}_+^{(t)}(\mX_n^{(t)})] \\
    &\times \sum_{p\in\calP(\mX_n^{(t)}; \vv)} \mathbbm{1}\{\langle \vw_{+,r}^{(t)}, \alpha_{n,p}^{(t)} \vv +\vzeta_{n,p}^{(t)}\rangle + b_{c,r}^{(t)} > 0\} \left(\langle \alpha_{n,p}^{(t)} \vv, \vzeta_{n,p}^{(\tau)} \rangle + \langle \vzeta_{n,p}^{(t)}, \alpha_{n,p}^{(\tau)} \vv^*\rangle + \langle \vzeta_{n,p}^{(t)}, \vzeta_{n,p}^{(\tau)} \rangle \right) \\
    & + \Delta b_{+,r}^{(t)} \\
    \le & \frac{\eta}{NP} \sum_{\vv\in\calC(\vv^*)}\sum_{n=1}^N \mathbbm{1}\{|\calP(\mX_n^{(t)}; \vv)|>0\} \left\vert\mathbbm{1}\{y_n=+\}-\text{logit}_+^{(t)}(\mX_n^{(t)})\right\vert \\
    &\times \sum_{p\in\calP(\mX_n^{(t)}; \vv)} \mathbbm{1}\{\langle \vw_{+,r}^{(t)}, \alpha_{n,p}^{(t)} \vv +\vzeta_{n,p}^{(t)}\rangle + b_{c,r}^{(t)} > 0\} \times O\left(\frac{1}{\ln^9(d)} \right) + \Delta b_{+,r}^{(t)} \\
    \le & \frac{\eta}{NP} \left( O\left(\frac{1}{\ln^9(d)} \right) -\frac{1}{\ln^5(d)}\left(\sqrt{1-\iota} - \frac{1}{\ln^9(d)}\right) \right) \\
    & \times \Bigg( \sum_{\vv\in\calC(\vv_+)}\sum_{n=1}^N \mathbbm{1}\{|\calP(\mX_n^{(t)}; \vv)|>0\} \left\vert\mathbbm{1}\{y_n=+\}-\text{logit}_+^{(t)}(\mX_n^{(t)})\right\vert  \\
    & \times \sum_{p\in\calP(\mX_n^{(t)}; \vv)} \mathbbm{1}\{\langle \vw_{+,r}^{(t)}, \alpha_{n,p}^{(t)} \vv +\vzeta_{n,p}^{(t)}\rangle + b_{c,r}^{(t)} > 0\} \Bigg)\\
    < & 0.
\end{aligned}
\end{equation}

This completes the proof.
\end{proof}

\begin{lemma} [Nonactivation on noise patches]
\label{lemma: sgd phase 1, nonact on noise}
    Let the assumptions in Theorem \ref{prop: phase 1 sgd induction} hold. 
    
    Denote the set of features $\calM=\{\vv_{+,c}\}_{c=1}^{k_+} \cup \{\vv_{-,c}\}_{c=1}^{k_-}\cup\{\vv_+,\vv_-\}$. If the update term for neuron $\vw_{+,r}^{(t)}$ can be written as follows
    \begin{equation}
    \begin{aligned}
        \Delta \vw_{+,r}^{(t)} 
        = & \frac{\eta}{NP} \sum_{\vv\in\calM}\sum_{n=1}^N \mathbbm{1}\{|\calP(\mX_n^{(t)}; \vv)|>0\} [\mathbbm{1}\{y_n=+\}-\text{logit}_+^{(t)}(\mX_n^{(t)})] \\
        & \times  \sum_{p\in\calP(\mX_n^{(t)}; \vv)} \mathbbm{1}\{\langle \vw_{+,r}^{(t)}, \alpha_{n,p}^{(t)} \vv +\vzeta_{n,p}^{(t)}\rangle + b_{c,r}^{(t)} > 0\} \left(\alpha_{n,p}^{(t)} \vv +\vzeta_{n,p}^{(t)}\right),
    \end{aligned}
    \end{equation}
    then
    \begin{equation}
    \begin{aligned}
        &\Delta b_{+,r}^{(t)} \\
        \le & -\frac{1}{\ln^5(d)} \frac{\eta}{NP}\left(\sqrt{1-\iota} - \frac{1}{\ln^9(d)}\right) \\
        & \times \Bigg( \sum_{\vv\in\calM}\sum_{n=1}^N \mathbbm{1}\{|\calP(\mX_n^{(t)}; \vv)|>0\} \left\vert\mathbbm{1}\{y_n=+\}-\text{logit}_+^{(t)}(\mX_n^{(t)})\right\vert \\
        & \times \sum_{p\in\calP(\mX_n^{(t)}; \vv)} \mathbbm{1}\{\langle \vw_{+,r}^{(t)}, \alpha_{n,p}^{(t)} \vv +\vzeta_{n,p}^{(t)}\rangle + b_{c,r}^{(t)} > 0\} \Bigg).
    \end{aligned}
    \end{equation}
    Moreover, for any $\tau > t$, the following inequality holds with probability at least $ 1-O\left( \frac{NP}{\poly(d)}\right)$ for all noise patches $\vx_{n,p}^{(\tau)} = \vzeta_{n,p}^{(\tau)}$:
    \begin{equation}
        \langle \Delta \vw_{+,r}^{(t)} , \vx_{n,p}^{(\tau)} \rangle + \Delta b_{+,r}^{(t)} < 0
    \end{equation}
\end{lemma}
\begin{proof}
Similar to the proof of Lemma \ref{lemma: sgd phase 1, nonact invar, t=1}, we can estimate the update to the bias term
\begin{equation}
\begin{aligned}
    &\Delta b_{+,r}^{(t)} \\
    \le & -\frac{1}{\ln^5(d)} \frac{\eta}{NP}\left(\sqrt{1-\iota} - \frac{1}{\ln^9(d)}\right) \\
    & \times \Bigg( \sum_{\vv\in\calM}\sum_{n=1}^N \mathbbm{1}\{|\calP(\mX_n^{(t)}; \vv)|>0\} \left\vert\mathbbm{1}\{y_n=+\}-\text{logit}_+^{(t)}(\mX_n^{(t)})\right\vert \\
    & \times \sum_{p\in\calP(\mX_n^{(t)}; \vv)} \mathbbm{1}\{\langle \vw_{+,r}^{(t)}, \alpha_{n,p}^{(t)} \vv +\vzeta_{n,p}^{(t)}\rangle + b_{c,r}^{(t)} > 0\} \Bigg)
\end{aligned}
\end{equation}

Then for any $\vx_{n,p}^{(\tau)} = \vzeta_{n,p}^{(\tau)}$ with $\tau > t$, with probability at least $1-O\left(\frac{mNP}{\poly(d)}\right)$,
\begin{equation}
\begin{aligned}
    &\langle \Delta \vw_{+,r}^{(t)} , \vx_{n,p}^{(\tau)}\rangle + \Delta b_{+,r}^{(t)} \\
    = &\langle \Delta \vw_{+,r}^{(t)}, \vzeta_{n,p}^{(\tau)} \rangle + \Delta b_{+,r}^{(t)} \\
    = & \frac{\eta}{NP} \sum_{\vv\in\calM}\sum_{n=1}^N \mathbbm{1}\{|\calP(\mX_n^{(t)}; \vv)|>0\} [\mathbbm{1}\{y_n=+\}-\text{logit}_+^{(t)}(\mX_n^{(t)})] \\
    & \times \sum_{p\in\calP(\mX_n^{(t)}; \vv)} \mathbbm{1}\{\langle \vw_{+,r}^{(t)}, \alpha_{n,p}^{(t)} \vv +\vzeta_{n,p}^{(t)}\rangle + b_{c,r}^{(t)} > 0\} \langle \alpha_{n,p}^{(t)} \vv +\vzeta_{n,p}^{(t)}, \vzeta_{n,p}^{(\tau)} \rangle + \Delta b_{+,r}^{(t)} \\
    = & \frac{\eta}{NP} \sum_{\vv\in\calM}\sum_{n=1}^N \mathbbm{1}\{|\calP(\mX_n^{(t)}; \vv)|>0\} [\mathbbm{1}\{y_n=+\}-\text{logit}_+^{(t)}(\mX_n^{(t)})] \\
    &\times \sum_{p\in\calP(\mX_n^{(t)}; \vv)} \mathbbm{1}\{\langle \vw_{+,r}^{(t)}, \alpha_{n,p}^{(t)} \vv +\vzeta_{n,p}^{(t)}\rangle + b_{c,r}^{(t)} > 0\} \left(\langle \alpha_{n,p}^{(t)} \vv, \vzeta_{n,p}^{(\tau)} \rangle + \langle \vzeta_{n,p}^{(t)}, \vzeta_{n,p}^{(\tau)} \rangle \right) + \Delta b_{+,r}^{(t)} \\
    \le & \frac{\eta}{NP} \sum_{\vv\in\calM}\sum_{n=1}^N \mathbbm{1}\{|\calP(\mX_n^{(t)}; \vv)|>0\} \left\vert\mathbbm{1}\{y_n=+\}-\text{logit}_+^{(t)}(\mX_n^{(t)})\right\vert \\
    &\times \sum_{p\in\calP(\mX_n^{(t)}; \vv)} \mathbbm{1}\{\langle \vw_{+,r}^{(t)}, \alpha_{n,p}^{(t)} \vv +\vzeta_{n,p}^{(t)}\rangle + b_{c,r}^{(t)} > 0\} \times O\left(\frac{1}{\ln^9(d)} \right) + \Delta b_{+,r}^{(t)} \\
    \le & \frac{\eta}{NP} \left( O\left(\frac{1}{\ln^9(d)} \right) -\frac{1}{\ln^5(d)}\left(\sqrt{1-\iota} - \frac{1}{\ln^9(d)}\right) \right) \\
    & \times \Bigg( \sum_{\vv\in\calM}\sum_{n=1}^N \mathbbm{1}\{|\calP(\mX_n^{(t)}; \vv)|>0\} \left\vert\mathbbm{1}\{y_n=+\}-\text{logit}_+^{(t)}(\mX_n^{(t)})\right\vert \\
    & \times \sum_{p\in\calP(\mX_n^{(t)}; \vv)} \mathbbm{1}\{\langle \vw_{+,r}^{(t)}, \alpha_{n,p}^{(t)} \vv +\vzeta_{n,p}^{(t)}\rangle + b_{c,r}^{(t)} > 0\} \Bigg) \\
    < & 0.
\end{aligned}
\end{equation}

\end{proof}

\newpage

\section{Fine-grained Learning}
\label{section: appendix, finegrained}

This section treats the learning dynamics of using fine-grained labels to train the NN; the analysis will be much simpler since the technical analysis overlaps significantly with that in the previous sections.

The training procedure is exactly the same as in the coarse-grained training setting. We explicitly write them out here to avoid any possible confusion.

The learner for fine-grained classification is written as follows for $c\in[k_+]$:
\begin{equation}
    F_{+,c}(\mX) = \sum_{r=1}^{m_{+,c}} a_{+,c,r}\sum_{p=1}^P \sigma(\langle \vw_{+,c,r}, \vx_p\rangle + b_{+,c,r}), \;\; c\in[k_+]
\end{equation}
with frozen linear classifier weights $ a_{+,c,r} = 1$. Same definition applies to the $-$ classes.

The SGD dynamics induced by the training loss is now
\begin{equation}
\begin{aligned}
    \vw_{+, c,r}^{(t+1)} 
    = \vw_{+, c,r}^{(t)} + \eta \frac{1}{NP} \sum_{n=1}^N \Bigg( 
    & \mathbbm{1}\{y_n = (+,c)\}[1-\text{logit}_{+,c}^{(t)}(\mX_n^{(t)})]\sum_{p\in[P]}\sigma'(\langle \vw_{+,c,r}^{(t)}, \vx_{n,p}^{(t)} \rangle +b_{c,r}^{(t)}) \vx_{n,p}^{(t)} + \\
    & \mathbbm{1}\{y_n \neq (+,c)\} [-\text{logit}_{+,c}^{(t)}(\mX_n^{(t)}) ]\sum_{p\in[P]} \sigma'(\langle \vw_{+,c,r}^{(t)}, \vx_{n,p}^{(t)}  \rangle + b^{(t)}_{c,r}) \vx_{n,p}^{(t)}\Bigg)
\end{aligned}
\end{equation}

The bias is manually tuned according to the update rule
\begin{equation}
    b_{+, c,r}^{(t+1)} = b_{+, c,r}^{(t)} - \frac{\|\Delta \vw_{+,c,r}^{(t)}\|_2}{\ln^5(d)}
\end{equation}

We assign $m_{+,c} = \Theta(d^{1+2c_0})$ neurons to each subclass $(+,c)$. For convenience, we write $m = d m_{+,c}$.

The initialization scheme is identical to the coarse-training case, except we choose a slightly less negative $b_{c,r}^{(0)} = - \sigma_0\sqrt{2 + 2c_0}\sqrt{\ln(d)}$.

The parameter choices remain the same as before.

\subsection{Initialization geometry}
\begin{definition}
Define the following sets of interest of the hidden neurons:
\begin{enumerate}
    \item $\calU_{+,c,r}^{(0)} = \{\vv \in \calV: \langle \vw_{+,c,r}^{(0)}, \vv\rangle \ge \sigma_0 \sqrt{2 + 2c_0}\sqrt{\ln(d) - \frac{1}{\ln^5(d)}}\}$
    \item Given $\vv \in \calV$, $S^{*(0)}_{+,c}(\vv) \subseteq (+,c) \times [m_{+,c}]$ satisfies:
    \begin{enumerate}
        \item $\langle \vw_{+, c, r}^{(0)}, \vv \rangle \ge \sigma_0 \sqrt{2 + 2c_0} \sqrt{\ln(d) + \frac{1}{\ln^5(d)}}$
        \item $\forall \vv' \in \calV \text{ s.t. } \vv' \perp \vv, \, \langle \vw_{+,c,r}^{(0)}, \vv' \rangle < \sigma_0 \sqrt{2 + 2c_0} \sqrt{\ln(d) - \frac{1}{\ln^5(d)}}$ 
    \end{enumerate}
    \item Given $\vv \in \calV$, $S_{+,c}^{(0)}(\vv) \subseteq (+,c)\times [m_{+,c}]$ satisfies:
    \begin{enumerate}
        \item $\langle \vw_{+,c,r}^{(0)}, \vv \rangle \ge \sigma_0 \sqrt{2 + 2c_0} \sqrt{\ln(d) - \frac{1}{\ln^5(d)}}$
    \end{enumerate}
    \item For any $(+,c,r) \in S_{+, c, reg}^{*(0)} \subseteq (+,c) \times [m_{+,c}]$:
    \begin{enumerate}
        \item $\langle \vw_{+,c,r}^{(0)}, \vv \rangle \le \sigma_0 \sqrt{10} \sqrt{\ln(d)} \; \forall \vv\in\calV$
        \item $\left\vert \calU_{+,c,r}^{(0)} \right\vert \le O(1)$
    \end{enumerate}
\end{enumerate}

The same definitions apply to the $-$-class neurons.
\end{definition}

\begin{prop}
\label{prop: init geometry, finegrained}
At $t=0$, for all $\vv \in \calD$, the following properties are true with probability at least $1- d^{-2}$ over the randomness of the initialized kernels:
\begin{enumerate}
    \item $|S_{+,c}^{*(0)}(\vv)|, |S_{+,c}^{(0)}(\vv)| = \Theta\left(\frac{1}{\sqrt{\ln(d)}}\right) d^{c_0}$
    \item In particular, $\left\vert \frac{|S_{y}^{*(0)}(\vv)|}{|S_{y'}^{(0)}(\vv')|} - 1 \right\vert= O\left( \frac{1}{\ln^5(d)}\right)$ and $\left\vert \frac{|S_{y}^{*(0)}(\vv)|}{|S_{y'}^{*(0)}(\vv')|} - 1 \right\vert= O\left( \frac{1}{\ln^5(d)}\right)$ for any $y,y'\in\{(+,c)\}_{c=1}^{k_+}\cup \{(-,c)\}_{c=1}^{k_-}$ and common or fine-grained features $\vv, \vv'$.
    \item $S_{+,c,reg}^{(0)} = [m_{+,c}]$
\end{enumerate}
The same properties apply to the $-$-class neurons.
\end{prop}
\begin{proof}
This proof proceeds in virtually the same way as in the proof of Proposition \ref{prop: init geometry, coarse}, so we omit it here.
\end{proof}

\subsection{Poly-time properties}
\begin{theorem}
\label{thm: sgd fine-grained, universal nonact properties}
Fix any $t\in[0, T_e]$, assuming $T_e\in\poly(d)$.
\begin{enumerate}
    \item (Non-activation invariance) For any $\tau \ge t$, with probability at least $ 1-O\left( \frac{mk_+NPt}{\poly(d)}\right)$, for any feature $\vv \in \{\vv_{+,c}\}_{c=1}^{k_+} \cup \{\vv_{-,c}\}_{c=1}^{k_-}\cup\{\vv_+,\vv_-\}$, for every $t'\le t$, $(+,c,r)\notin S_{+,c}^{(0)}(\vv)$ and $\vv$-dominated patch sample $\vx_{n,p}^{(\tau)} = \alpha_{n,p}^{(\tau)}\vv + \vzeta_{n,p}^{(\tau)}$, the following holds:
    \begin{equation}
        \sigma\left(\langle \vw_{+,c,r}^{(t')}, \vx_{n,p}^{(\tau)} \rangle + b_{+,c,r}^{(t')}\right) = 0
    \end{equation}
    
    \item (Non-activation on noise patches) For any $\tau \ge t$, with probability at least $ 1-O\left( \frac{mNPt}{\poly(d)}\right)$, for every $c\in[k_+]$, $r\in[m]$ and noise patch $\vx_{n,p}^{(\tau)} = \vzeta_{n,p}^{(\tau)}$, the following holds: 
    \begin{equation}
        \sigma\left(\langle \vw_{+,c,r}^{(t)}, \vx_{n,p}^{(\tau)} \rangle + b_{+,c,r}^{(t)}\right) = 0
    \end{equation}

    \item (Off-diagonal nonpositive growth) Given fine-grained class $(+,c)$ and any $\tau \ge t$, with probability at least $ 1-O\left( \frac{mk_+NPt}{\poly(d)}\right)$, for any $t'\le t$, any feature $\vv \in \{\vv_{-,c}\}_{c=1}^{k_-}\cup\{\vv_-\}\cup \{\vv_{+,c'}\}_{c'\neq c}$, any neuron $\vw_{+,c,r}\in S^{(0)}_{+,c}(\vv)$ and any $\vv$-dominated patch $\vx_{n,p}^{(\tau)} = \alpha_{n,p}^{(\tau)}\vv + \vzeta_{n,p}^{(\tau)}$, $\sigma\left(\langle \vw_{+,c,r}^{(t')}, \vx_{n,p}^{(\tau)} \rangle + b_{+,c,r}^{(t')}\right) \le \sigma\left(\langle \vw_{+,c,r}^{(0)}, \vx_{n,p}^{(\tau)} \rangle + b_{+,c,r}^{(0)}\right)$.

\end{enumerate}
\end{theorem}
\begin{proof}
The proof of this theorem is similar to that of Theorem \ref{thm: sgd, universal nonact properties}, but with some subtle differences.

\textcolor{blue}{\textbf{Base case $t = 0$.}}

\textcolor{brown}{\textit{1. (Nonactivation invariance)}}

Choose any $\vv^*$ from the set $\{\vv_{+,c}\}_{c=1}^{k_+} \cup \{\vv_{-,c}\}_{c=1}^{k_-}\cup\{\vv_+,\vv_-\}$. We will work with neuron sets in the ``$+$'' class in this proof; the ``$-$''-class case can be handled in the same way.
    
First, given $\tau \ge 0$, we need to show that, for every $n$ such that $\vert \calP(\mX_n^{(\tau)}; \vv^*)\vert > 0$ and $p\in\calP(\mX_n^{(\tau)}; \vv^*)$, for every $(+,c,r)$ neuron index, 
\begin{equation}
    \langle \vw_{+,c,r}^{(0)}, \vv^* \rangle < \sigma_0 \sqrt{2 + 2c_0} \sqrt{\ln(d) - \frac{1}{\ln^5(d)}} \implies \sigma\left(\langle \vw_{+,c,r}^{(0)}, \vx_{n,p}^{(\tau)}\rangle + b_{+,c,r}^{(0)} \right) = 0
\end{equation}

This is indeed true. The following holds with probability at least $1 - O\left(\frac{mNP}{\poly(d)}\right)$ for all $(+,r)\notin S_+^{(0)}(\vv)$ and all such $\vx_{n,p}^{(\tau)}$:
\begin{equation}
\begin{aligned}
    \langle \vw_{+,c,r}^{(0)}, \vx_{n,p}^{(\tau)}\rangle + b_{+,c,r}^{(0)} 
    \le & \sigma_0 \sqrt{1+\iota} \sqrt{(2+2c_0)(\ln(d) - 1/\ln^5(d))} + O\left(\frac{\sigma_0}{\ln^{9}(d)}\right) - \sqrt{2 + 2c_0}\sqrt{\ln(d) } \sigma_0\\
    = & \sigma_0 \left(\frac{(2+2c_0)(1 + \iota)(\ln(d) - 1/\ln^5(d)) - (2+2c_0)\ln(d)}{\sqrt{(2+2c_0)(\ln(d) - 1/\ln^5(d))} + \sqrt{4 + 2c_0}\sqrt{\ln(d) }}   + O\left(\frac{1}{\ln^{9}(d)}\right)\right) \\
    = & \sigma_0 \left(\frac{(2+2c_0)(\iota\ln(d) - (1+\iota)/\ln^5(d))}{\sqrt{(2+2c_0)(\ln(d) - 1/\ln^5(d))} + \sqrt{2 + 2c_0}\sqrt{\ln(d) }}   + O\left(\frac{1}{\ln^{9}(d)}\right)\right) \\
    < & 0,
\end{aligned}
\end{equation}

The first equality holds by utilizing the identity $a - b = \frac{a^2 - b^2}{a + b}$. As a consequence, $\sigma(\langle \vw_{+,c,r}^{(0)}, \vx_{n,p}^{(\tau)}\rangle + b_{+,r}^{(0)}) = 0$.

\textcolor{brown}{\textit{2. (Non-activation on noise patches)}}
Invoking Lemma \ref{lemma: independent gaussian vector inner product concentration}, for any $\tau \ge 0$, with probability at least $1-O\left(\frac{mNP}{\poly(d)}\right)$, we have for all possible choices of $r\in[m]$ and the noise patches $\vx_{n,p}^{(\tau)} = \vzeta_{n,p}^{(\tau)}$:
\begin{equation}
    \left\vert \langle \vw_{+,c,r}^{(0)}, \vzeta_{n,p}^{(\tau)} \rangle \right\vert \le O(\sigma_0\sigma_{\zeta} \sqrt{d\ln(d)}) \le O\left(\frac{\sigma_0}{\ln^{9}(d)}\right) \ll b_{+,r}^{(0)}.
\end{equation}

Therefore, no neuron can activate on the noise patches at time $t=0$.

\textcolor{brown}{\textit{3. (Off-diagonal nonpositive growth)}}
This point is trivially true at $t=0$.

\vspace{4ex}
\textcolor{blue}{\textbf{Inductive step}}: we assume the induction hypothesis for $t\in [0, T]$ (with $T < T_e$ of course), and prove the statements for $t = T+1$.

\textcolor{brown}{\textit{1. (Nonactivation invariance)}}

Again, choose any $\vv^*$ from the set $\{\vv_{+,c}\}_{c=1}^{k_+} \cup \{\vv_{-,c}\}_{c=1}^{k_-}\cup\{\vv_+,\vv_-\}$.

We need to prove that given $\tau \ge T+1$, with probability at least $ 1-O\left( \frac{mk_+NP(T+1)}{\poly(d)}\right)$, for every $t'\le T+1$, $(+,c,r)$ neuron index and $\vv^*$-dominated patch $\vx_{n,p}^{(\tau)}$,
\begin{equation}
    (+,c,r)\notin S_{+,c}^{(0)}(\vv^*) \implies \sigma\left(\langle \vw_{+,c,r}^{(t')}, \vx_{n,p}^{(\tau)}\rangle + b_{+,c,r}^{(t')} \right) = 0.
\end{equation}

By the induction hypothesis of point 1., with probability at least $ 1-O\left( \frac{mk_+NPT}{\poly(d)}\right)$, the following is already true on all the $\vv^*$-dominated patches at time $t' \le T$:
\begin{equation}
    (+,c,r)\notin S_{+,c}^{(0)}(\vv^*) \implies \sigma\left(\langle \vw_{+,c,r}^{(t')}, \vx_{n,p}^{(T)}\rangle + b_{+,c,r}^{(t')} \right) = 0.
\end{equation}
In particular, $\sigma\left(\langle \vw_{+,c,r}^{(T)}, \vx_{n,p}^{(T)}\rangle + b_{+,c,r}^{(T)} \right) = 0$.

In other words, no $(+,c,r)\notin S_{+,c}^{(0)}(\vv^*)$ can be updated on the $\vv^*$-dominated patches at time $t=T$. Furthermore, the induction hypothesis of point 2. also states that the network cannot activate on any noise patch $\vx_{n,p}^{(T)} = \vzeta_{n,p}^{(T)}$ with probability at least $ 1-O\left( \frac{mNPT}{\poly(d)}\right)$. Therefore, the neuron update for those $(+,c,r)\notin S_{+,c}^{(0)}(\vv^*)$ takes the form
\begin{equation}
\begin{aligned}
    \Delta \vw_{+,c,r}^{(T)} 
    = & \frac{\eta}{NP} \sum_{\vv\in\calC(\vv^*)}\sum_{n=1}^N \mathbbm{1}\{|\calP(\mX_n^{(T)}; \vv)|>0\} [\mathbbm{1}\{y_n=(+,c)\}-\text{logit}_{+,c}^{(T)}(\mX_n^{(T)})] \\
    & \times  \sum_{p\in\calP(\mX_n^{(T)}; \vv)} \mathbbm{1}\{\langle \vw_{+,c,r}^{(T)}, \alpha_{n,p}^{(T)} \vv +\vzeta_{n,p}^{(T)}\rangle + b_{+,c,r}^{(T)} > 0\} \left(\alpha_{n,p}^{(T)} \vv +\vzeta_{n,p}^{(T)}\right)
\end{aligned}
\end{equation}

Conditioning on this high-probability event, we have
\begin{equation}
\begin{aligned}
    \Delta b_{+,c,r}^{(t)} 
    = & -\frac{\left\| \Delta \vw_{+,c,r}^{(t)} \right\|_2}{\ln^5(d)} \\
    \le & -\frac{1}{\ln^5(d)} \frac{\eta}{NP} \Bigg\|\sum_{\vv\in\calC(\vv^*)}\sum_{n=1}^N \mathbbm{1}\{|\calP(\mX_n^{(t)}; \vv)|>0\} [\mathbbm{1}\{y_n=(+,c)\}-\text{logit}_{+,c}^{(t)}(\mX_n^{(t)})] \\
    & \times  \sum_{p\in\calP(\mX_n^{(t)}; \vv)} \mathbbm{1}\{\langle \vw_{+,c,r}^{(t)}, \alpha_{n,p}^{(t)} \vv +\vzeta_{n,p}^{(t)}\rangle + b_{+,c,r}^{(t)} > 0\} \alpha_{n,p}^{(t)} \vv \Bigg\|_2 \\
    & + \frac{1}{\ln^5(d)} \frac{\eta}{NP} \Bigg\|\sum_{\vv\in\calC(\vv^*)}\sum_{n=1}^N \mathbbm{1}\{|\calP(\mX_n^{(t)}; \vv)|>0\} [\mathbbm{1}\{y_n=(+,c)\}-\text{logit}_{+,c}^{(t)}(\mX_n^{(t)})] \\
    & \times  \sum_{p\in\calP(\mX_n^{(t)}; \vv)} \mathbbm{1}\{\langle \vw_{+,c,r}^{(t)}, \alpha_{n,p}^{(t)} \vv +\vzeta_{n,p}^{(t)}\rangle + b_{+,c,r}^{(t)} > 0\} \vzeta_{n,p}^{(t)} \Bigg\|_2 \\
\end{aligned}
\end{equation}

Let us further upper bound the two $\|\cdot\|_2$ terms separately. Firstly,
\begin{equation}
\begin{aligned}
    & \Bigg\|\sum_{\vv\in\calC(\vv^*)}\sum_{n=1}^N \mathbbm{1}\{|\calP(\mX_n^{(t)}; \vv)|>0\} [\mathbbm{1}\{y_n=(+,c)\}-\text{logit}_{+,c}^{(t)}(\mX_n^{(t)})] \\
    & \times  \sum_{p\in\calP(\mX_n^{(t)}; \vv)} \mathbbm{1}\{\langle \vw_{+,c,r}^{(t)}, \alpha_{n,p}^{(t)} \vv +\vzeta_{n,p}^{(t)}\rangle + b_{+,c,r}^{(t)} > 0\} \alpha_{n,p}^{(t)}\vv \Bigg\|_2 \\
    = & \sum_{\vv\in\calC(\vv^*)}\sum_{n=1}^N \mathbbm{1}\{|\calP(\mX_n^{(t)}; \vv)|>0\} \left\vert \mathbbm{1}\{y_n=(+,c)\}-\text{logit}_{+,c}^{(t)}(\mX_n^{(t)})\right\vert \\
    & \times  \sum_{p\in\calP(\mX_n^{(t)}; \vv)} \mathbbm{1}\{\langle \vw_{+,c,r}^{(t)}, \alpha_{n,p}^{(t)} \vv +\vzeta_{n,p}^{(t)}\rangle + b_{+,c,r}^{(t)} > 0\} \alpha_{n,p}^{(t)} \left\| \vv \right\|_2 \\
    \ge & \sum_{\vv\in\calC(\vv^*)}\sum_{n=1}^N \mathbbm{1}\{|\calP(\mX_n^{(t)}; \vv)|>0\} \left\vert \mathbbm{1}\{y_n=(+,c)\}-\text{logit}_{+,c}^{(t)}(\mX_n^{(t)})\right\vert \\
    & \times  \sum_{p\in\calP(\mX_n^{(t)}; \vv)} \mathbbm{1}\{\langle \vw_{+,c,r}^{(t)}, \alpha_{n,p}^{(t)} \vv +\vzeta_{n,p}^{(t)}\rangle + b_{+,c,r}^{(t)} > 0\} \sqrt{1-\iota}
\end{aligned}
\end{equation}

For the second $\|\cdot\|_2$ term consisting purely of noise, note that since all the $\vzeta_{n,p}^{(t)}$'s are independent Gaussian random vectors, the standard deviation of the sum is in fact
\begin{equation}
\begin{aligned}
    & \Bigg\{\sum_{\vv\in\calC(\vv^*)}\sum_{n=1}^N \sum_{p\in\calP(\mX_n^{(t)}; \vv)} \mathbbm{1}\{|\calP(\mX_n^{(t)}; \vv)|>0\}   \mathbbm{1}\{\langle \vw_{+,c,r}^{(t)}, \alpha_{n,p}^{(t)} \vv +\vzeta_{n,p}^{(t)}\rangle + b_{+,c,r}^{(t)} > 0\} \\
    & \times [\mathbbm{1}\{y_n=(+,c)\}-\text{logit}_{+,c}^{(t)}(\mX_n^{(t)})]^2 \Bigg\}^{1/2} \sigma_{\zeta}.
\end{aligned}
\end{equation}
With the basic property that $\sqrt{\sum_{j}c_j^2} \le \sum_{j} |c_j|$ for any sequence of real numbers $c_1, c_2, ...$, we know this standard deviation can be upper bounded by
\begin{equation}
\begin{aligned}
    & \sum_{\vv\in\calC(\vv^*)}\sum_{n=1}^N \sum_{p\in\calP(\mX_n^{(t)}; \vv)} \mathbbm{1}\{|\calP(\mX_n^{(t)}; \vv)|>0\}   \mathbbm{1}\{\langle \vw_{+,c,r}^{(t)}, \alpha_{n,p}^{(t)} \vv +\vzeta_{n,p}^{(t)}\rangle + b_{+,c,r}^{(t)} > 0\} \\
    & \times \left\vert\mathbbm{1}\{y_n=(+,c)\}-\text{logit}_{+,c}^{(t)}(\mX_n^{(t)}) \right\vert \sigma_{\zeta}
\end{aligned}
\end{equation}

It follows that with probability at least $1-O\left(\frac{1}{\poly(d)}\right)$, 
\begin{equation}
\begin{aligned}
    & \Bigg\|\sum_{\vv\in\calC(\vv^*)}\sum_{n=1}^N \mathbbm{1}\{|\calP(\mX_n^{(t)}; \vv)|>0\} [\mathbbm{1}\{y_n=(+,c)\}-\text{logit}_{+,c}^{(t)}(\mX_n^{(t)})] \\
    & \times  \sum_{p\in\calP(\mX_n^{(t)}; \vv)} \mathbbm{1}\{\langle \vw_{+,c,r}^{(t)}, \alpha_{n,p}^{(t)} \vv +\vzeta_{n,p}^{(t)}\rangle + b_{+,c,r}^{(t)} > 0\} \vzeta_{n,p}^{(t)}\Bigg\|_2 \\
    \le & \sum_{\vv\in\calC(\vv^*)}\sum_{n=1}^N \mathbbm{1}\{|\calP(\mX_n^{(t)}; \vv)|>0\} \left\vert \mathbbm{1}\{y_n=(+,c)\}-\text{logit}_{+,c}^{(t)}(\mX_n^{(t)})\right\vert \\
    & \times  \sum_{p\in\calP(\mX_n^{(t)}; \vv)} \mathbbm{1}\{\langle \vw_{+,c,r}^{(t)}, \alpha_{n,p}^{(t)} \vv +\vzeta_{n,p}^{(t)}\rangle + b_{+,c,r}^{(t)} > 0\} \frac{1}{\ln^9(d)}
\end{aligned}
\end{equation}

Therefore, we can upper bound the bias update as follows:
\begin{equation}
\begin{aligned}
    &\Delta b_{+,c,r}^{(t)} \\
    \le & -\frac{1}{\ln^5(d)} \frac{\eta}{NP}\left(\sqrt{1-\iota} - \frac{1}{\ln^9(d)}\right) \\
    & \times \Bigg( \sum_{\vv\in\calC(\vv^*)}\sum_{n=1}^N \mathbbm{1}\{|\calP(\mX_n^{(t)}; \vv)|>0\} \left\vert \mathbbm{1}\{y_n=(+,c)\}-\text{logit}_{+,c}^{(t)}(\mX_n^{(t)})\right\vert \\
    & \times  \sum_{p\in\calP(\mX_n^{(t)}; \vv)} \mathbbm{1}\{\langle \vw_{+,c,r}^{(t)}, \alpha_{n,p}^{(t)} \vv +\vzeta_{n,p}^{(t)}\rangle + b_{+,c,r}^{(t)} > 0\} \Bigg)
\end{aligned}
\end{equation}

Furthermore, with probability at least $1 - O\left(\frac{NP}{\poly(d)}\right)$, the following holds for all $n,p$:
\begin{equation}
    \langle \alpha_{n,p}^{(t)} \vv, \vzeta_{n,p}^{(\tau)} \rangle,\; \langle \vzeta_{n,p}^{(t)}, \alpha_{n,p}^{(\tau)} \vv^*\rangle,\; \langle \vzeta_{n,p}^{(t)}, \vzeta_{n,p}^{(\tau)} \rangle < O\left( \frac{1}{\ln^9(d)} \right).
\end{equation}

Combining the above derivations, they imply that with probability at least $1-O\left(\frac{NP}{\poly(d)}\right)$, for any $\vx_{n,p}^{(\tau)}$ dominated by $\vv^*$,
\begin{equation}
\begin{aligned}
    &\langle \Delta \vw_{+,c,r}^{(t)} , \vx_{n,p}^{(\tau)}\rangle + \Delta b_{+,c,r}^{(t)} \\
    = &\langle \Delta \vw_{+,c,r}^{(t)} , \alpha_{n,p}^{(\tau)}\vv^* + \vzeta_{n,p}^{(\tau)} \rangle + \Delta b_{+,c,r}^{(t)} \\
    = & \frac{\eta}{NP} \sum_{\vv\in\calC(\vv^*)}\sum_{n=1}^N \mathbbm{1}\{|\calP(\mX_n^{(t)}; \vv)|>0\} [\mathbbm{1}\{y_n=(+,c)\}-\text{logit}_{+,c}^{(t)}(\mX_n^{(t)})] \\
    & \times  \sum_{p\in\calP(\mX_n^{(t)}; \vv)} \mathbbm{1}\{\langle \vw_{+,c,r}^{(t)}, \alpha_{n,p}^{(t)} \vv +\vzeta_{n,p}^{(t)}\rangle + b_{+,c,r}^{(t)} > 0\}  \langle \alpha_{n,p}^{(t)} \vv +\vzeta_{n,p}^{(t)}, \alpha_{n,p}^{(\tau)}\vv^* + \vzeta_{n,p}^{(\tau)} \rangle + \Delta b_{+,c,r}^{(t)} \\
    = & \frac{\eta}{NP}\sum_{\vv\in\calC(\vv^*)}\sum_{n=1}^N \mathbbm{1}\{|\calP(\mX_n^{(t)}; \vv)|>0\} [\mathbbm{1}\{y_n=(+,c)\}-\text{logit}_{+,c}^{(t)}(\mX_n^{(t)})] \\
    & \times  \sum_{p\in\calP(\mX_n^{(t)}; \vv)} \mathbbm{1}\{\langle \vw_{+,c,r}^{(t)}, \alpha_{n,p}^{(t)} \vv +\vzeta_{n,p}^{(t)}\rangle + b_{+,c,r}^{(t)} > 0\}  \left(\langle \alpha_{n,p}^{(t)} \vv, \vzeta_{n,p}^{(\tau)} \rangle + \langle \vzeta_{n,p}^{(t)}, \alpha_{n,p}^{(\tau)} \vv^*\rangle + \langle \vzeta_{n,p}^{(t)}, \vzeta_{n,p}^{(\tau)} \rangle \right) \\
    &+ \Delta b_{+,c,r}^{(t)} \\
    \le & \frac{\eta}{NP} \sum_{\vv\in\calC(\vv^*)}\sum_{n=1}^N \mathbbm{1}\{|\calP(\mX_n^{(t)}; \vv)|>0\} \left\vert \mathbbm{1}\{y_n=(+,c)\}-\text{logit}_{+,c}^{(t)}(\mX_n^{(t)})\right\vert \\
    & \times  \sum_{p\in\calP(\mX_n^{(t)}; \vv)} \mathbbm{1}\{\langle \vw_{+,c,r}^{(t)}, \alpha_{n,p}^{(t)} \vv +\vzeta_{n,p}^{(t)}\rangle + b_{+,c,r}^{(t)} > 0\} \times O\left(\frac{1}{\ln^9(d)} \right) + \Delta b_{+,c,r}^{(t)} \\
    \le & \frac{\eta}{NP} \left( O\left(\frac{1}{\ln^9(d)} \right) -\frac{1}{\ln^5(d)}\left(\sqrt{1-\iota} - \frac{1}{\ln^9(d)}\right) \right) \\
    & \times \Bigg( \sum_{\vv\in\calC(\vv^*)}\sum_{n=1}^N \mathbbm{1}\{|\calP(\mX_n^{(t)}; \vv)|>0\} \left\vert \mathbbm{1}\{y_n=(+,c)\}-\text{logit}_{+,c}^{(t)}(\mX_n^{(t)})\right\vert \\
    & \times  \sum_{p\in\calP(\mX_n^{(t)}; \vv)} \mathbbm{1}\{\langle \vw_{+,c,r}^{(t)}, \alpha_{n,p}^{(t)} \vv +\vzeta_{n,p}^{(t)}\rangle + b_{+,c,r}^{(t)} > 0\} \Bigg)\\
    < & 0.
\end{aligned}
\end{equation}

Therefore, with probability at least $ 1-O\left( \frac{mNP}{\poly(d)}\right)$, the following holds for the relevant neurons and $\vv^*$-dominated patches:
\begin{equation}
    \langle \Delta \vw_{+,c,r}^{(T)} , \vx_{n,p}^{(\tau)} \rangle + \Delta b_{+,c,r}^{(T)} < 0.
\end{equation}

In conclusion, with $\tau \ge T+1$, with probability at least $ 1-O\left( \frac{mNP}{\poly(d)}\right)$, for every $(+,c,r)\notin S_{+,c}^{(0)}(\vv^*)$ and relevant $(n,p)$'s,
\begin{equation}
    \langle \vw_{+,c,r}^{(T)} + \Delta \vw_{+,c,r}^{(T)}, \vx_{n,p}^{(\tau)}\rangle + b_{+,c,r}^{(T)} + \Delta b_{+,c,r}^{(T)} = \langle \vw_{+,c,r}^{(T+1)}, \vx_{n,p}^{(\tau)}\rangle + b_{+,c,r}^{(T+1)} < 0,
\end{equation}
which leads to $\langle \vw_{+,c,r}^{(t')}, \vx_{n,p}^{(\tau)}\rangle + b_{+,c,r}^{(t')} < 0$ for all $t' \le T+1$ with probability at least $ 1-O\left( \frac{mk_+NP(T+1)}{\poly(d)}\right)$ (also by taking union bound over all the possible choices of $\vv^*$ at time $T+1$). This finishes the inductive step for point 1.

\textcolor{brown}{\textit{2. (Non-activation on noise patches)}}

The inductive step for this part is very similar to (and even simpler than) the inductive step of point 1, so we omit the calculations here.

\textcolor{brown}{\textit{3. (Off-diagonal nonpositive growth)}}
By the induction hypothesis's high-probability event, we already have that, given any fine-grained class $(+,c)$, $\tau \ge T+1$, for any feature $\vv^* \in \{\vv_{-,c}\}_{c=1}^{k_-}\cup\{\vv_-\}\cup \{\vv_{+,c'}\}_{c'\neq c}$ and any neuron $\vw_{+,c,r}$, $\sigma\left(\langle \vw_{+,c,r}^{(T)}, \vx_{n,p}^{(\tau)} \rangle + b_{+,c,r}^{(T)}\right) \le \sigma\left(\langle \vw_{+,c,r}^{(0)}, \vx_{n,p}^{(\tau)} \rangle + b_{+,c,r}^{(0)}\right)$. We just need to show that $\langle \Delta \vw_{+,c,r}^{(t)}, \vx_{n,p}^{(\tau)} \rangle + \Delta b_{+,r}^{(T)} \le 0$ to finish the proof; the rest proceeds in a similar fashion to the induction step of point 3 in the proof of Theorem \ref{thm: sgd, universal nonact properties}.

Similar to the induction step of point 1, denoting $\calM$ to be the set of all common and fine-grained features, the update expression of any neuron $(+,c,r)$ has to be
\begin{equation}
\begin{aligned}
    \Delta \vw_{+,c,r}^{(T)} 
    = & \frac{\eta}{NP} \sum_{\vv\in\calM}\sum_{n=1}^N \mathbbm{1}\{|\calP(\mX_n^{(T)}; \vv)|>0\} [\mathbbm{1}\{y_n=(+,c)\}-\text{logit}_{+,c}^{(T)}(\mX_n^{(T)})] \\
    & \times  \sum_{p\in\calP(\mX_n^{(T)}; \vv)} \mathbbm{1}\{\langle \vw_{+,c,r}^{(T)}, \alpha_{n,p}^{(T)} \vv +\vzeta_{n,p}^{(T)}\rangle + b_{+,c,r}^{(T)} > 0\} \left(\alpha_{n,p}^{(T)} \vv +\vzeta_{n,p}^{(T)}\right)
\end{aligned}
\end{equation}

Written more explicitly,
\begin{equation}
\begin{aligned}
    \Delta \vw_{+,c,r}^{(T)} 
    = & \frac{\eta}{NP} \sum_{\vv\in\calM - \{\vv^*\}}\sum_{n=1}^N \mathbbm{1}\{|\calP(\mX_n^{(T)}; \vv)|>0\} \mathbbm{1}\{y_n=(+,c)\} [1-\text{logit}_{+,c}^{(T)}(\mX_n^{(T)})] \\
    & \times  \sum_{p\in\calP(\mX_n^{(T)}; \vv)} \mathbbm{1}\{\langle \vw_{+,c,r}^{(T)}, \alpha_{n,p}^{(T)} \vv +\vzeta_{n,p}^{(T)}\rangle + b_{+,c,r}^{(T)} > 0\} \left(\alpha_{n,p}^{(T)} \vv +\vzeta_{n,p}^{(T)}\right) \\
    & - \frac{\eta}{NP} \sum_{n=1}^N \mathbbm{1}\{y_n\neq(+,c)\} \mathbbm{1}\{|\calP(\mX_n^{(T)}; \vv^*)|>0\} [\text{logit}_{+,c}^{(T)}(\mX_n^{(T)})] \\
    & \times  \sum_{p\in\calP(\mX_n^{(T)}; \vv^*)} \mathbbm{1}\{\langle \vw_{+,c,r}^{(T)}, \alpha_{n,p}^{(T)} \vv^* +\vzeta_{n,p}^{(T)}\rangle + b_{+,c,r}^{(T)} > 0\} \left(\alpha_{n,p}^{(T)} \vv^* +\vzeta_{n,p}^{(T)}\right)
\end{aligned}
\end{equation}

It follows that with probability at least $1-O\left(\frac{mNP}{\poly(d)}\right)$, for relevant $n,p,r$, we have
\begin{equation}
\begin{aligned}
    &\langle \Delta \vw_{+,c,r}^{(T)}, \alpha_{n,p}^{(\tau)}\vv^* + \vzeta_{n,p}^{(\tau)} \rangle \\
    < & \frac{\eta}{NP} \sum_{\vv\in\calM - \{\vv^*\}}\sum_{n=1}^N \mathbbm{1}\{|\calP(\mX_n^{(T)}; \vv)|>0\}\mathbbm{1}\{y_n=(+,c)\} [1-\text{logit}_{+,c}^{(T)}(\mX_n^{(T)})] \\
    & \times  \sum_{p\in\calP(\mX_n^{(T)}; \vv)} \mathbbm{1}\{\langle \vw_{+,c,r}^{(T)}, \alpha_{n,p}^{(T)} \vv +\vzeta_{n,p}^{(T)}\rangle + b_{+,c,r}^{(T)} > 0\} O\left( \frac{1}{\ln^9(d)}\right)
\end{aligned}
\end{equation}

Furthermore, similar to the induction step of point 1, we can estimate the bias update as follows:
\begin{equation}
\begin{aligned}
    &\Delta b_{+,c,r}^{(t)} \\
    \le & -\Omega\left(\frac{1}{\ln^5(d)} \right)\frac{\eta}{NP}  \Bigg( \sum_{\vv\in\calM}\sum_{n=1}^N \mathbbm{1}\{|\calP(\mX_n^{(t)}; \vv)|>0\} \left\vert \mathbbm{1}\{y_n=(+,c)\}-\text{logit}_{+,c}^{(t)}(\mX_n^{(t)})\right\vert \\
    & \times  \sum_{p\in\calP(\mX_n^{(t)}; \vv)} \mathbbm{1}\{\langle \vw_{+,c,r}^{(t)}, \alpha_{n,p}^{(t)} \vv +\vzeta_{n,p}^{(t)}\rangle + b_{+,c,r}^{(t)} > 0\} \Bigg)
\end{aligned}
\end{equation}

It follows that, indeed, $\langle \Delta \vw_{+,c,r}^{(T)}, \vx_{n,p}^{(\tau)} \rangle + \Delta b_{+,c,r}^{(T)} \le 0$, which completes the induction step of point 3.
\end{proof}

\subsection{Training}
Choose an arbitrary constant $B \in [\Omega(1),\ln(3/2)]$.

\begin{definition}
Let $T_0(B) > 0$ be the first time that there exists some $\mX_n^{(t)}$ and $c$ such that $F_{y}^{(T_0(B))}(\mX_{n}^{(T_0(B))}) \ge B$ for any $n\in[N]$ and $y\in \{(+,c)\}_{c=1}^{k_+}\cup \{(-,c)\}_{c=1}^{k_-}$.

We write $T_0(B)$ as $T_0$ for simplicity of notation when the context is clear.
\end{definition}

\begin{lemma}
\label{lemma:finegrained sgd induction}
With probability at least $ 1-O\left( \frac{mk_+NPT_0}{\poly(d)}\right)$, the following holds for all $t \in [0, T_0)$:
\begin{enumerate}
    \item (On-diagonal common-feature neuron growth) For every $c\in[k_+]$, every $(+,c,r), (+,c,r') \in S_{+,c}^{*(0)}(\vv_+)$,
    \begin{equation}
        \vw_{+,c,r}^{(t)} - \vw_{+,c,r}^{(0)} = \vw_{+,c,r'}^{(t)} - \vw_{+,c,r'}^{(0)}
    \end{equation}

    Moreover,
    \begin{equation}
    \begin{aligned}
    \Delta \vw_{+,r}^{(t)} 
    = & [1/4, 2/3] \sqrt{1\pm\iota} \left(1 \pm s^{*-1/3}\right) \eta \frac{s^*}{2k_+P} \vv_{+}  + \Delta \vzeta^{(t)}_{+,r}
    \end{aligned}
    \end{equation}
    where $\Delta \vzeta^{(t)}_{+,c,r} \sim \calN(\vzero, \sigma_{\Delta \zeta_{+,c,r}}^{(t)2} \mI)$,  $\sigma_{\Delta\zeta_{+,c,r}}^{(t)} = \Theta(1) \times \eta \sigma_{\zeta} \frac{\sqrt{s^*}}{P\sqrt{2N}}$.

    The bias updates satisfy
    \begin{equation}
    \begin{aligned}
        \Delta b_{+,c,r}^{(t)} = -\Theta\left(\frac{\eta s^*}{k_+P\ln^5(d)}\right).
    \end{aligned}
    \end{equation}

    Furthermore, every $(+,r) \in S_{+}^{*(0)}(\vv_+)$ activates on all the $\vv_+$-dominated patches at time $t$.
    
    \item (On-diagonal finegrained-feature neuron growth) For every $c\in[k_+]$ and every $(+,c,r), (+,c,r') \in S_{+,c}^{*(0)}(\vv_{+,c})$, 
    \begin{equation}
        \vw_{+,c,r}^{(t)} - \vw_{+,c,r}^{(0)} = \vw_{+,c,r'}^{(t)} - \vw_{+,c,r'}^{(0)}
    \end{equation}

    Moreover,
    \begin{equation}
    \begin{aligned}
    \Delta \vw_{+,c,r}^{(t)} 
    = & \left(1 \pm O\left(\frac{1}{k_+} \right)\right) \sqrt{1\pm\iota} \left(1 \pm s^{*-1/3}\right) \eta \frac{s^*}{2 k_+ P} \vv_{+,c}  + \Delta\vzeta^{(t)}_{+,r}
    \end{aligned}
    \end{equation}
    where $\vzeta^{(t)}_{+,c,r} \sim \calN(\vzero, \sigma_{\Delta\zeta_{+,cr}}^{(t)2} \mI)$, and $\sigma_{\Delta\zeta_{+,r}}^{(t)} = \left(1 \pm O\left(\frac{1}{k_+} \right)\right) \left(1 \pm s^{*-1/3}\right) \eta\sigma_{\zeta} \frac{\sqrt{s^*}}{P\sqrt{2Nk_+}}$.

    The bias updates satisfy
    \begin{equation}
    \begin{aligned}
        \Delta b_{+,c,r}^{(t)} = -\Theta\left(\frac{\eta s^*}{k_+P\ln^5(d)}\right).
    \end{aligned}
    \end{equation}

    Furthermore, every $(+,c,r) \in S_{+,c}^{*(0)}(\vv_{+,c})$ activates on all the $\vv_+$-dominated patches at time $t$.
    
    \item The above results also hold with the ``$+$'' and ``$-$'' class signs flipped.
\end{enumerate}  
\end{lemma}
\begin{proof}
The proof of this theorem proceeds in a similar fashion to Theorem \ref{prop: phase 1 sgd induction}, with some variations for the common-feature neurons.

We shall prove the statements in this theorem via induction. We focus on the $+$-class neurons; $-$-class neurons' proofs are done in the same fashion.

First of all, relying on the (high-probability) event of Theorem \ref{thm: sgd fine-grained, universal nonact properties}, we know that we can simplify the update expressions for the neurons in $S_{+,c}^{*(0)}(\vv_{+,c})$ to the form
\begin{equation}
\begin{aligned}
    \Delta \vw_{+,c,r}^{(t)} 
    = & \frac{\eta}{NP} \sum_{n=1}^N \mathbbm{1}\{y_n=(+,c)\} [1 -\text{logit}_{+,c}^{(t)}(\mX_n^{(t)})] \\
    & \times  \sum_{p\in\calP(\mX_n^{(t)}; \vv_{+,c})} \mathbbm{1}\{\langle \vw_{+,c,r}^{(t)}, \alpha_{n,p}^{(t)} \vv_{+,c} +\vzeta_{n,p}^{(t)}\rangle + b_{+,c,r}^{(t)} > 0\} \left(\alpha_{n,p}^{(t)} \vv_{+,c} +\vzeta_{n,p}^{(t)}\right),
\end{aligned}
\end{equation}
and for the neurons in $S_{+,c}^{*(0)}(\vv_{+})$, the updates take the form
\begin{equation}
\begin{aligned}
    & \Delta \vw_{+,c,r}^{(t)} \\
    = & \frac{\eta}{NP} \sum_{n=1}^N \left\{\mathbbm{1}\{y_n=(+,c)\} [1 -\logit_{+,c}^{(t)}(\mX_n^{(t)})] + \sum_{c'\in[k_+]-\{c\}}  \mathbbm{1}\{y_n=(+,c')\} [-\text{logit}_{+,c}^{(t)}(\mX_n^{(t)})]\right\} \\
    & \times  \sum_{p\in\calP(\mX_n^{(t)}; \vv_+)} \mathbbm{1}\{\langle \vw_{+,c,r}^{(t)}, \alpha_{n,p}^{(t)} \vv_+ +\vzeta_{n,p}^{(t)}\rangle + b_{+,c,r}^{(t)} > 0\} \left(\alpha_{n,p}^{(t)} \vv_{+} +\vzeta_{n,p}^{(t)}\right).
\label{eq: fine training, common S^* update}
\end{aligned}
\end{equation}

By definition of $T_0$ and the fact that $B \le \ln(3/2)$, for any $n\in[N]$ and $t < T_0$, we can write down a simple upper bound of $\logit_{+,c}^{(t)}(\mX_n^{(t)})$:
\begin{equation}
\begin{aligned}
    \text{logit}_{+,c}^{(t)}(\mX_{n}^{(t)}) 
    = & \frac{\exp(F_{+,c}(\mX_{n}^{(t)}))}{\sum_{c'=1}^{k_+} \exp(F_{+,c'}(\mX_{n}^{(t)})) + \sum_{c'=1}^{k_-} \exp(F_{-,c'}(\mX_{n}^{(t)}))} \\
    \le & \frac{\frac{3}{2}}{2k_+} = \frac{3}{4k_+},
\end{aligned}
\end{equation}
and we can lower bound it as follows
\begin{equation}
\begin{aligned}
    \text{logit}_{+,c}^{(t)}(\mX_{n}^{(t)}) 
    \ge & \frac{1}{2k_+ \times \frac{3}{2}} = \frac{1}{3k_+},
\end{aligned}
\end{equation}

The inductive proof for the fine-grained neurons $S_{+,c}^{*(0)}(\vv_{+,c})$ is almost identical to that in the proof of Theorem \ref{prop: phase 1 sgd induction}. The only notable difference here is that $[1 -\text{logit}_{+,c}^{(t)}(\mX_n^{(t)})]$ has the estimate $\left(1 \pm O\left(\frac{1}{k_+} \right)\right)$.

The inductive proof of the common-feature neurons $S_{+,c}^{*(0)}(\vv_{+})$ requires more care as its update expression \ref{eq: fine training, common S^* update} is qualitatively different from the coarse-grained training case in Theorem \ref{prop: phase 1 sgd induction}, so we present the full proof here.

\textbf{Base case, $t = 0$}.

With probability at least $1-O\left(\frac{mNP}{\poly(d)}\right)$, for every $c\in[k_+]$ and every $(+,c,r)\in S_{+,c}^{*(0)}(\vv_{+})$,
\begin{equation}
\begin{aligned}
    &\langle \vw_{+,c,r}^{(0)}, \alpha_{n,p}^{(0)} \vv_{+} + \vzeta_{n,p}^{(0)} \rangle + b_{+,c,r}^{(0)} \\
    & \ge \sigma_0\left(\sqrt{(1 - \iota)(2 + 2c_0)(\ln(d) + 1/\ln^5(d))} -  \sqrt{(2 + 2c_0)\ln(d)} - O\left(\frac{1}{\ln^{9}(d)}\right) \right) \\
    & = \sigma_0\left(\frac{(1 - \iota)(2 + 2c_0)(\ln(d) + 1/\ln^5(d)) - (2 + 2c_0)\ln(d)}{\sqrt{(1 - \iota)(2 + 2c_0)(\ln(d) + 1/\ln^5(d))} +  \sqrt{(2 + 2c_0)\ln(d)}} - O\left(\frac{1}{\ln^{9}(d)}\right) \right) \\
    & = \sigma_0\left(\frac{(2+2c_0)(-\iota\ln(d) + (1-\iota)/\ln^5(d))}{\sqrt{(1 - \iota)(2 + 2c_0)(\ln(d) + 1/\ln^5(d))} +  \sqrt{(2 + 2c_0)\ln(d)}} - O\left(\frac{1}{\ln^{9}(d)}\right) \right) \\
    & > 0.
\end{aligned}
\end{equation}

This means all the $\vv_+$-singleton neurons will be updated on all the $\vv_+$-dominated patches at time $t = 0$. Therefore, we can write update expression \ref{eq: fine training, common S^* update} as follows
\begin{equation}
\begin{aligned}
    & \Delta \vw_{+,c,r}^{(0)} \\
    = & \frac{\eta}{NP} \sum_{n=1}^N \left\{\mathbbm{1}\{y_n=(+,c)\} [1 -\logit_{+,c}^{(0)}(\mX_n^{(0)})] + \sum_{c'\in[k_+]-\{c\}}  \mathbbm{1}\{y_n=(+,c')\} [-\text{logit}_{+,c}^{(0)}(\mX_n^{(0)})]\right\} \\
    & \times  \sum_{p\in\calP(\mX_n^{(0)}; \vv_+)} \left(\alpha_{n,p}^{(0)} \vv_{+} +\vzeta_{n,p}^{(0)}\right).
\end{aligned}
\end{equation}

By concentration of the binomial random variable, we know that with probability at least $1 - e^{-\Omega(\ln^2(d))}$, for all $n$,
\begin{equation}
    \left\vert \calP(\mX_n^{(0)}; \vv_+) \right\vert = \left(1 \pm s^{*-1/3}\right)s^*.
\end{equation}

Now, with the estimates we derived for $\logit_{+,c}^{(t)}(\mX_n^{(t)})$ at the beginning of the proof and the independence of all the noise vectors $\vzeta_{n,p}^{(0)}$'s, we arrive at
\begin{equation}
\begin{aligned}
\Delta \vw_{+,r}^{(0)} 
= & [1/4, 2/3] \sqrt{1\pm\iota} \left(1 \pm s^{*-1/3}\right) \eta \frac{s^*}{2k_+P} \vv_{+}  + \Delta \vzeta^{(0)}_{+,r}
\end{aligned}
\end{equation}
where $\sigma_{\Delta\zeta_{+,c,r}}^{(0)} = \Theta(1) \times \eta \sigma_{\zeta} \frac{\sqrt{s^*}}{P\sqrt{2N}}$. 

Additionally, a byproduct of the above proof steps is that all the $S_{+,c}^{*(0)}(\vv_{+})$ neurons indeed activate on all the $\vv_+$-dominated patches at $t=0$ with high probability.

Now we examine the bias update. We first estimate $\left\|\Delta \vw_{+,c,r}^{(0)} \right\|_2$. With probability at least $1 - O\left(\frac{m}{\poly(d)} \right)$ the following upper bound holds for all neurons in $S_{+,c}^{*(0)}(\vv_{+})$:
\begin{equation}
\begin{aligned}
    \left\|\Delta \vw_{+,c,r}^{(0)} \right\|_2
    \le & O\left(\eta \frac{s^*}{k_+P}\right) \|\vv_{+}\|_2  + \left\|\Delta \vzeta^{(0)}_{+,r} \right\|_2 \\
    \le & O\left(\eta \frac{s^*}{k_+P}\right) + O\left( \eta \sigma_{\zeta} \frac{\sqrt{s^*}}{P\sqrt{N}} \sqrt{d}\right) \\
    \le & O\left(\eta \frac{s^*}{k_+P}\right),
\end{aligned}
\end{equation}
and the following lower bound holds (via the reverse triangle inequality):
\begin{equation}
\begin{aligned}
    \left\|\Delta \vw_{+,c,r}^{(0)} \right\|_2
    \ge & \Omega\left(\eta \frac{s^*}{k_+P}\right) \|\vv_{+}\|_2 - \left\|\Delta \vzeta^{(0)}_{+,r} \right\|_2 \\
    \ge & \Omega\left(\eta \frac{s^*}{k_+P}\right) - O\left( \eta \sigma_{\zeta} \frac{\sqrt{s^*}}{P\sqrt{N}} \sqrt{d}\right) \\
    \ge & \Omega\left(\eta \frac{s^*}{k_+P}\right),
\end{aligned}
\end{equation}

It follows that $\left\|\Delta \vw_{+,c,r}^{(0)} \right\|_2 = \Theta\left(\eta \frac{s^*}{k_+P}\right)$, which means
\begin{equation}
\begin{aligned}
    \Delta b_{+,c,r}^{(0)}
    = & -\frac{\left\|\Delta \vw_{+,c,r}^{(0)} \right\|_2}{\ln^5(d)} = -\Theta\left(\frac{\eta s^*}{k_+P\ln^5(d)}\right).
\end{aligned}
\end{equation}

This completes the proof of the base case.

\textbf{Induction step}. Assume statements for time $[0,t]$, prove for $t+1$.

First, by the induction hypothesis, we know that neurons in $S_{+,c}^{*(0)}(\vv_{+})$ must activate on all the $\vv_+$-dominated patches at time $t$. Therefore, we can write the update expression \ref{eq: fine training, common S^* update} as follows:
\begin{equation}
\begin{aligned}
    & \Delta \vw_{+,c,r}^{(t)} \\
    = & \frac{\eta}{NP} \sum_{n=1}^N \left\{\mathbbm{1}\{y_n=(+,c)\} [1 -\logit_{+,c}^{(t)}(\mX_n^{(t)})] + \sum_{c'\in[k_+]-\{c\}}  \mathbbm{1}\{y_n=(+,c')\} [-\text{logit}_{+,c}^{(t)}(\mX_n^{(t)})]\right\} \\
    & \times  \sum_{p\in\calP(\mX_n^{(t)}; \vv_+)} \left(\alpha_{n,p}^{(t)} \vv_{+} +\vzeta_{n,p}^{(t)}\right).
\end{aligned}
\end{equation}

Following the same argument as in the base case, we have that with probability at least $1-O\left(\frac{mNP}{\poly(d)}\right)$, 
\begin{equation}
\begin{aligned}
\Delta \vw_{+,c,r}^{(t)} 
= & [1/4, 2/3] \sqrt{1\pm\iota} \left(1 \pm s^{*-1/3}\right) \eta \frac{s^*}{2k_+P} \vv_{+}  + \Delta \vzeta^{(t)}_{+,c,r},
\end{aligned}
\end{equation}
and $\sigma_{\Delta\zeta_{+,c,r}}^{(t)} = \Theta(1) \times \eta \sigma_{\zeta} \frac{\sqrt{s^*}}{P\sqrt{2N}}$.

Now we need to show that $\vw_{+,c,r}^{(t+1)}$ indeed activate on all the $\vv_+$-dominated patches at time $t+1$ with high probability.

So far, we know that for $\tau \in [0, t+1]$, 
\begin{equation}
\begin{aligned}
\Delta \vw_{+,c,r}^{(\tau)} 
= & [1/4, 2/3] \sqrt{1\pm\iota} \left(1 \pm s^{*-1/3}\right) \eta \frac{s^*}{2k_+P} \vv_{+}  + \Delta \vzeta^{(\tau)}_{+,c,r},
\end{aligned}
\end{equation}
and $\sigma_{\Delta\zeta_{+,c,r}}^{(\tau)} = \Theta(1) \times \eta \sigma_{\zeta} \frac{\sqrt{s^*}}{P\sqrt{2N}}$. It follows that
\begin{equation}
\begin{aligned}
\vw_{+,r}^{(t+1)} 
= & \vw_{+,c,r}^{(0)} + (t+1) [1/4, 2/3] \sqrt{1\pm\iota} \left(1 \pm s^{*-1/3}\right) \eta \frac{s^*}{2k_+P} \vv_{+}  + \vzeta^{(t+1)}_{+,c,r},
\end{aligned}
\end{equation}
where $\sigma_{\zeta_{+,c,r}}^{(t+1)} = \Theta(1)\times \sqrt{t+1}  \eta \sigma_{\zeta} \frac{\sqrt{s^*}}{P\sqrt{2N}}$.

The following holds with probability at least $1-O\left(\frac{mNP}{\poly(d)}\right)$ over all the $\vv_+$-dominated patches $\vx_{n,p}^{(t+1)} = \alpha_{n,p}^{(t+1)}\vv_+ + \vzeta_{n,p}^{(t+1)}$ (which are independent of $\vw_{+,r}^{(t+1)}$) and the $\vv_+$-singleton neurons:
\begin{equation}
\begin{aligned}
& \langle \vw_{+,c,r}^{(t+1)}, \alpha_{n,p}^{(t+1)}\vv_+ + \vzeta_{n,p}^{(t+1)} \rangle \\
= & \langle \vw_{+,c,r}^{(0)} \alpha_{n,p}^{(t+1)}\vv_+ + \vzeta_{n,p}^{(t+1)} \rangle + (t+1) [1/4, 2/3] (1\pm\iota) \left(1 \pm s^{*-1/3}\right) \left(1 \pm O\left( \frac{1}{\ln^9(d)} \right)\right)\eta \frac{s^*}{2k_+P} \\
& + \langle \vzeta^{(t+1)}_{+,c,r}, \alpha_{n,p}^{(t+1)}\vv_+ + \vzeta_{n,p}^{(t+1)} \rangle
\end{aligned}
\end{equation}

Note that with probability at least $1-O\left(\frac{1}{\poly(d)}\right)$,
\begin{equation}
\begin{aligned}
    \langle \vzeta^{(t+1)}_{+,c,r}, \alpha_{n,p}^{(t+1)}\vv_+\rangle \le O(1)\times \sqrt{T}  \eta \sigma_{\zeta} \frac{\sqrt{s^*}}{P\sqrt{2N}} \sqrt{d\ln(d)},
\end{aligned}
\end{equation}
and since $\sqrt{t+1} \le t+1$, $\sqrt{s^*} < s^*$, $\sigma_{\zeta}\sqrt{d\ln(d)} < \frac{1}{\ln^9(d)}$, and $N > dk_+$, we know that 
\begin{equation}
\begin{aligned}
    \langle \vzeta^{(t+1)}_{+,c,r}, \alpha_{n,p}^{(t+1)}\vv_+\rangle \le O\left(\frac{1}{d}\right)\times (t+1)  \eta \frac{s^*}{2 k_+ P}.
\end{aligned}
\end{equation}

Similarly, with probability at least $1-O\left(\frac{1}{\poly(d)}\right)$,
\begin{equation}
    \langle \vzeta^{(t+1)}_{+,c,r}, \alpha_{n,p}^{(t+1)}\vv_+ + \vzeta_{n,p}^{(t+1)} \rangle \le O(1)\times \sqrt{T}  \eta \sigma_{\zeta}^2 \frac{\sqrt{s^*}}{P\sqrt{2N}} \sqrt{d\ln(d)} \le O\left(\frac{1}{d}\right)\times (t+1)  \eta \frac{s^*}{2 k_+ P}.
\end{equation}

It follows that with probability at least $1-O\left(\frac{mNP}{\poly(d)}\right)$,
\begin{equation}
\begin{aligned}
& \langle \vw_{+,c,r}^{(t+1)}, \alpha_{n,p}^{(t+1)}\vv_+ + \vzeta_{n,p}^{(t+1)} \rangle \\
\ge & \langle \vw_{+,c,r}^{(0)}, \alpha_{n,p}^{(t+1)}\vv_+ + \vzeta_{n,p}^{(t+1)} \rangle + \frac{1}{4} (t+1) (1-\iota) \left(1 - s^{*-1/3}\right) \left(1 - O\left( \frac{1}{\ln^9(d)} \right)\right)\eta \frac{s^*}{2k_+P}.
\end{aligned}
\end{equation}

Next, let us estimate the bias updates for $\tau \in [0, t+1]$.

Estimating $\Delta b_{+,c,r}^{(t)}$ follows an almost identical argument as in the base case (with the only main difference being relying on Theorem \ref{thm: sgd fine-grained, universal nonact properties} for non-activation on non-$\vv_+$-dominated patches), so we skip its calculations. 

Therefore, $b_{+,c,r}^{(t+1)} = b_{+,c,r}^{(0)} + -\Theta\left(\frac{\eta s^* (t+1)}{k_+P\ln^5(d)}\right)$. This means
\begin{equation}
\begin{aligned}
& \langle \vw_{+,c,r}^{(t+1)}, \alpha_{n,p}^{(t+1)}\vv_+ + \vzeta_{n,p}^{(t+1)} \rangle + b_{+,c,r}^{(t+1)}\\
\ge & \langle \vw_{+,c,r}^{(0)}, \alpha_{n,p}^{(t+1)}\vv_+ + \vzeta_{n,p}^{(t+1)} \rangle + b_{+,c,r}^{(0)} \\
& + \frac{1}{4} (t+1) (1-\iota) \left(1 - s^{*-1/3}\right) \left(1 - O\left( \frac{1}{\ln^9(d)} \right)\right)\eta \frac{s^*}{2k_+P} - O\left(\frac{\eta s^* (t+1)}{k_+P\ln^5(d)}\right) \\
> & 0.
\end{aligned}
\end{equation}

This completes the inductive step.
\end{proof}

\begin{corollary}
\label{coro: finegrained sgd, feature response theta(1)}
    At time $t=T_0$, $\frac{\eta s^*}{k_+ P} \times s^* \left\vert S_{+,c}^{*(0)}(\vv_+) \right\vert, \frac{\eta s^*}{k_+ P} \times s^* \left\vert S_{+,c}^{*(0)}(\vv_{+,c}) \right\vert = \Theta(1)$.
\end{corollary}
\begin{proof}
    Directly follows from Lemma \ref{lemma:finegrained sgd induction} and Theorem \ref{thm: sgd fine-grained, universal nonact properties}.
\end{proof}

\subsection{Model error after training}
\label{sec: appendix, finegrained trainining, end error}
In this subsection, we show the model's error after fine-grained training. We also discuss that finetuning the model further increases its feature extractor's response to the true features, so it is even more robust/generalizing in downstream classification tasks.

\begin{theorem}
    Define $\widehat{F}_+(\mX) = \max_{c\in[k_+]}F_{+,c}(\mX), \, \widehat{F}_-(\mX) = \max_{c\in[k_-]}F_{-,c}(\mX)$.

    With probability at least $1 - O\left( \frac{mk_+^2NPT_0}{\poly(d)}\right)$, the following events take place:
    \begin{enumerate}
        \item (Fine-grained easy \& hard sample test accuracies are nearly perfect) Given an easy or hard fine-grained test sample $(\mX, y)$ where $y \in \{(+,c)\}_{c=1}^{k_+}\cup\{(-,c)\}_{c=1}^{k_-}$, $\mathbb{P}\left[F_y^{(T_0)}(\mX) \le \max_{y'\neq y}F_{y'}^{(T_0)}(\mX)\right] \le o(1)$.
        \item (Coarse-grained easy \& hard sample test accuracy are nearly perfect) Given an easy or hard coarse-grained test sample $(\mX, y)$ where $y \in \{+1, -1\}$, $\mathbb{P}\left[\widehat{F}_y^{(T_0)}(\mX) \le \widehat{F}_{y'}^{(T_0)}(\mX)\right] \le o(1)$.
    \end{enumerate}
\end{theorem}
\begin{proof}
    \textbf{Probability of mistake on easy samples}.
    
    Without loss of generality, assume $\mX$ is a $(+,c)$-class easy sample.
    
    Conditioning on the events of Theorem \ref{thm: sgd fine-grained, universal nonact properties} and Lemma \ref{lemma:finegrained sgd induction}, we know that for all $c'\in[k_-]$,
    \begin{equation}
        F_{-,c'}^{(T_0)} \le O(m_{+,c'}\sigma_0\sqrt{\ln(d)}) \le o(1),
    \end{equation}
    and for all $c'\in[k_+]-\{c\}$,
    \begin{equation}
    \begin{aligned}
        F_{+,c'}^{(T_0)} 
        \le & \sum_{p\in\calP(\mX;\vv_{+})}\sum_{(+,r)\in S_{+,c'}^{(0)}(\vv_+)} \sigma\left( \langle \vw_{+,r}^{(T_0)}, \alpha_{n,p}\vv_+ + \vzeta_{n,p} \rangle + b_{+,c',r}^{(T_0)} \right) + O(m_{+,c'}\sigma_0\sqrt{\ln(d)}) \\
        \le & s^* \left\vert S_{+,c'}^{(0)}(\vv_+)\right\vert \frac{2}{3}(1 + \iota) \left(1 + s^{*-1/3}\right) \left(1 + \left(\frac{1}{\ln^9(d)} \right)\right) \eta T_0 \frac{s^*}{2k_+P} 
    \end{aligned}
    \end{equation}
    moreover,
    \begin{equation}
    \begin{aligned}
        F_{+,c}^{(T_0)} 
        \ge & \sum_{p\in\calP(\mX;\vv_{+})}\sum_{(+,r)\in S_{+,c}^{*(0)}(\vv_+)} \sigma\left( \langle \vw_{+,c,r}^{(T_0)}, \alpha_{n,p}\vv_+ + \vzeta_{n,p} \rangle + b_{+,c,r}^{(T_0)} \right) \\
        & + \sum_{p\in\calP(\mX;\vv_{+,c})}\sum_{(+,r)\in S_{+,c}^{*(0)}(\vv_{+,c})} \sigma\left( \langle \vw_{+,c,r}^{(T_0)}, \alpha_{n,p}\vv_{+,c} + \vzeta_{n,p} \rangle + b_{+,c,r}^{(T_0)} \right) \\
        \ge & s^* \left\vert S_{+,c}^{*(0)}(\vv_{+})\right\vert \frac{1}{4}(1 - \iota) \left(1 - s^{*-1/3}\right) \left(1 - \left(\frac{1}{\ln^5(d)} \right)\right) \eta T_0 \frac{s^*}{2k_+P}\\
        & + s^* \left\vert S_{+,c}^{*(0)}(\vv_{+,c})\right\vert \left(1 - O\left(\frac{1}{k_+}\right) \right)(1 - \iota) \left(1 - s^{*-1/3}\right) \left(1 - \left(\frac{1}{\ln^5(d)} \right)\right) \eta T_0 \frac{s^*}{2k_+P}
    \end{aligned}
    \end{equation}

    Relying on Proposition \ref{prop: init geometry, finegrained}, we know $\left\vert S_{+,c'}^{(0)}(\vv_+)\right\vert = \left(1 \pm \left(\frac{1}{\ln^5(d)} \right)\right) \left\vert S_{+,c}^{*(0)}(\vv_{+})\right\vert$ and $\left\vert S_{+,c}^{*(0)}(\vv_{+,c})\right\vert = \left(1 \pm \left(\frac{1}{\ln^5(d)} \right)\right) \left\vert S_{+,c}^{*(0)}(\vv_{+})\right\vert$, therefore $F_{+,c}^{(T_0)}(\mX) > \max_{c'\neq c} F_{+,c'}^{(T_0)}(\mX)$ has to be true. With Corollary \ref{coro: finegrained sgd, feature response theta(1)}, we also have $F_{+,c}^{(T_0)}(\mX) \ge \Omega(1) > o(1) \ge \max_{c'\in[k_-]} F_{-,c'}^{(T_0)}(\mX)$. It follows that the probability of mistake on an easy test sample is indeed at most $o(1)$.

    \textbf{Probability of mistake on hard samples}.
     Without loss of generality, assume $\mX$ is a $(+,c)$-class hard sample.

     By Theorem \ref{thm: sgd fine-grained, universal nonact properties} (and its proof) and Lemma \ref{lemma:finegrained sgd induction}, we know that for any $c'\in[k_+]$, the neurons $\vw_{+,c',r}$ can only possibly receive update on $\vv$-dominated patches for $\vv\in\calU_{+,c',r}^{(0)}$, and the updates to the neurons take the feature-plus-Gaussian-noise form of $\sum_{\vv'\in\calU_{+,c',r}^{(0)}}c(\vv')\vv' + \Delta \vzeta_{+,c',r}^{(t)}$, with $c(\vv')\le \sqrt{1+\iota} \left(1 + s^{*-1/3}\right) \eta \frac{s^*}{2 k_+ P}$ if $\vv'$ is a fine-grained feature, or $c(\vv')\le \frac{2}{3}\sqrt{1+\iota} \left(1 + s^{*-1/3}\right) \eta \frac{s^*}{2 k_+ P} $ if $\vv'=\vv_+$ (because the $\vv'$ component of a $\vv'$-singleton neuron's update is already the maximum possible). Moreover, $\sigma_{\Delta\zeta_{+,c',r}}^{(t)} \le O\left(\eta \sigma_{\zeta} \frac{\sqrt{s^*}}{P\sqrt{2N}} \right)$.

     Relying on Theorem \ref{thm: sgd fine-grained, universal nonact properties}, Lemma \ref{lemma:finegrained sgd induction}, Corollary \ref{coro: finegrained sgd, feature response theta(1)} and previous observations, we have
     \begin{equation}
     \begin{aligned}
         F_{+,c}^{(T_0)}(\mX) 
         \ge & \sum_{p\in\calP(\mX;\vv_{+,c})}\sum_{(+,c,r)\in S_{+,c}^{*(0)}(\vv_{+,c})} \sigma\left( \langle \vw_{+,c,r}^{(T_0)}, \alpha_{n,p}\vv_{+,c} + \vzeta_{n,p} \rangle + b_{+,c,r}^{(T_0)} \right) \\
         \ge & s^* \left\vert S_{+,c}^{*(0)}(\vv_{+,c})\right\vert \left(1 - O\left(\frac{1}{k_+}\right) \right)(1 - \iota) \left(1 - s^{*-1/3}\right) \left(1 - O\left(\frac{1}{\ln^5(d)} \right)\right) \eta T_0 \frac{s^*}{2k_+P} \\
         \ge & \Omega(1),
     \end{aligned}
     \end{equation}
     and for $c'\neq c$,
     \begin{equation}
     \begin{aligned}
         F_{+,c'}^{(T_0)}(\mX) 
         \le & \sum_{r=1}^{m_{+,c'}} \sigma\left( \langle \vw_{+,c',r}^{(T_0)}, \vzeta^* \rangle + b_{+,c',r}^{(T_0)} \right) \\
         & + \sum_{p\in\calP(\mX;\vv_{+,c})}\sum_{(+,c',r)\in S_{+,c'}^{(0)}(\vv_{+,c})} \sigma\left( \langle \vw_{+,c',r}^{(T_0)}, \alpha_{n,p}\vv_{+,c} + \vzeta_{n,p} \rangle + b_{+,c',r}^{(T_0)} \right) \\
         & + \sum_{p\in\calP(\mX;\vv_{-})}\sum_{(+,c',r)\in S_{+,c'}^{(0)}(\vv_{-})} \sigma\left( \langle \vw_{+,c',r}^{(T_0)}, \alpha_{n,p}^{\dagger}\vv_{-} + \vzeta_{n,p} \rangle + b_{+,c',r}^{(T_0)} \right) \\
         \le & O(1) \times \left(\sum_{(+,c',r)\in \calU_{+,c',r}^{(0)}} \langle \sum_{\tau=0}^{T_0-1} \Delta \vw_{+,c,'r}^{(\tau)}, \vzeta^* \rangle + \sum_{r\in[m_{+,c'}]} \langle \vw_{+,c,'r}^{(0)}, \vzeta^* \rangle \right)\\
         & + \sum_{p\in\calP(\mX;\vv_{+,c})}\sum_{(+,c',r)\in S_{+,c'}^{(0)}(\vv_{+,c})} \sigma\left( \langle \vw_{+,c',r}^{(0)}, \alpha_{n,p}\vv_{+,c} + \vzeta_{n,p} \rangle + b_{+,c',r}^{(0)} \right) \\
         & + \sum_{p\in\calP(\mX;\vv_{-})}\sum_{(+,c',r)\in S_{+,c'}^{(0)}(\vv_{-})} \sigma\left( \langle \vw_{+,c',r}^{(0)}, \alpha_{n,p}^{\dagger}\vv_{-} + \vzeta_{n,p} \rangle + b_{+,c',r}^{(0)} \right) \\
         \le & O\left(\frac{1}{\polyln(d)}\right).
     \end{aligned}
     \end{equation}

     Moreover, for any $c'\in[k_-]$, similar to before,
     \begin{equation}
     \begin{aligned}
         F_{-,c'}^{(T_0)}(\mX) 
         \le & \sum_{r=1}^{m_{-,c'}} \sigma\left( \langle \vw_{-,c',r}^{(T_0)}, \vzeta^* \rangle + b_{-,c',r}^{(T_0)} \right) \\
         & + \sum_{p\in\calP(\mX;\vv_{+,c})}\sum_{(-,c',r)\in S_{-,c'}^{(0)}(\vv_{+,c})} \sigma\left( \langle \vw_{-,c',r}^{(T_0)}, \alpha_{n,p}\vv_{+,c} + \vzeta_{n,p} \rangle + b_{-,c',r}^{(T_0)} \right) \\
         & + \sum_{p\in\calP(\mX;\vv_{-})}\sum_{(-,c',r)\in S_{-,c'}^{(0)}(\vv_{-})} \sigma\left( \langle \vw_{-,c',r}^{(T_0)}, \alpha_{n,p}^{\dagger}\vv_{-} + \vzeta_{n,p} \rangle + b_{-,c',r}^{(T_0)} \right) \\
         \le & O(1) \times \left( \sum_{(-,c',r)\in \calU_{-,c',r}^{(0)}} \langle \vw_{-,c,'r}^{(T_0)}, \vzeta^* \rangle  + \sum_{r\in[m_{-,c'}]}\langle \vw_{-,c,'r}^{(0)}, \vzeta^* \rangle \right) \\
         & + \sum_{p\in\calP(\mX;\vv_{+,c})}\sum_{(-,c',r)\in S_{-,c'}^{(0)}(\vv_{+,c})} \sigma\left( \langle \vw_{-,c',r}^{(0)}, \alpha_{n,p}\vv_{+,c} + \vzeta_{n,p} \rangle + b_{-,c',r}^{(0)} \right) \\
         & + O(1) \times s^\dagger \left\vert S_{-,c'}^{(0)}(\vv_{-}) \right\vert \times \left(\iota^{\dagger}_{upper} + O(\sigma_0\ln(d)) \right) \\
         \le &  O\left( \frac{1}{\polyln(d)}\right) +  O\left( \sigma_0 \sqrt{\ln(d)}\right) + O\left( \frac{1}{\ln(d)}\right) \\
         \le & o(1).
     \end{aligned}
     \end{equation}

     Therefore, $F_{+,c}^{(T_0)}(\mX) > \max_{y\neq (+,c)} F_y^{(T_0)}(\mX)$, which means $\widehat{F}_+^{(T_0)}(\mX) > \widehat{F}_-^{(T_0)}(\mX)$ indeed.
\end{proof}
\begin{remark}

First of all, note that the feature extractor, after fine-grained training, is already well-performing, as it responds strongly ($\Omega(1)$ strength) to the true features, and very weakly to any off-diagonal features and noise. This can already help us explain the linear-probing result we saw on ImageNet21k in Appendix \ref{appendix: im21k cross-dataset}, since linear probing does not alter the the feature extractor after fine-grained pretraining (on ImageNet21k), it only retrains a new linear classifier on top of the feature extractor for classifying on the target ImageNet1k dataset.

At a high level, \textit{finetuning} $\widehat{F}$ can only further enhance the feature extractor's response to the features, therefore making the model even more robust for challenging downstream classification problems; it will not degrade the feature extractor's response to any true feature. A rigorous proof of this statement is almost a repetition of the proofs for fine-grained training, so we do not repeat them here. Intuitively speaking, we just need to note that the properties stated in Theorem \ref{thm: sgd fine-grained, universal nonact properties} will continue to hold during finetuning (as long as we stay in polynomial time), and with similar argument to those in the proof of Lemma \ref{lemma:finegrained sgd induction}, we note that the neurons responsible for detecting fine-grained features, i.e. the $S_{+,c}^{*(0)}(\vv_{+,c})$, will continue to only receive (positive) updates on the $\vv_{+,c}$-dominated patches of the following form:
\begin{equation}
\begin{aligned}
    \Delta \vw_{+,c,r}^{(t)} 
    = & \frac{\eta}{NP} \sum_{n=1}^N \mathbbm{1}\{y_n=(+,c)\} [1 -\text{logit}_{+}^{(t)}(\mX_n^{(t)})] \\
    & \times  \sum_{p\in\calP(\mX_n^{(t)}; \vv_{+,c})} \mathbbm{1}\{\langle \vw_{+,c,r}^{(t)}, \alpha_{n,p}^{(t)} \vv_{+,c} +\vzeta_{n,p}^{(t)}\rangle + b_{+,c,r}^{(t)} > 0\} \left(\alpha_{n,p}^{(t)} \vv_{+,c} +\vzeta_{n,p}^{(t)}\right),
\end{aligned}
\end{equation}
and similar update expression can be stated for the $S_{+,c}^{*(0)}(\vv_{+})$ neurons:
\begin{equation}
\begin{aligned}
    & \Delta \vw_{+,c,r}^{(t)} \\
    = & \frac{\eta}{NP} \sum_{n=1}^N \mathbbm{1}\{y_n=(+,c)\} [1 -\logit_{+}^{(t)}(\mX_n^{(t)})] \\
    & \times  \sum_{p\in\calP(\mX_n^{(t)}; \vv_+)} \mathbbm{1}\{\langle \vw_{+,c,r}^{(t)}, \alpha_{n,p}^{(t)} \vv_+ +\vzeta_{n,p}^{(t)}\rangle + b_{+,c,r}^{(t)} > 0\} \left(\alpha_{n,p}^{(t)} \vv_{+} +\vzeta_{n,p}^{(t)}\right).
\end{aligned}
\end{equation}
Indeed, these feature-detector neurons will continue growing in the direction of the features they are responsible for detecting instead of degrade in strength.
    
\end{remark}

\newpage
\section{Probability Lemmas}
\label{section: appendix, prob lemmas}
\begin{lemma}[Laurent-Massart $\chi^2$ Concentration (\cite{laurentMassartChiSquared} Lemma 1)]
\label{lemma: laurent-massart}
Let $\vg \sim \calN(\vzero, \mI_d)$. For any vector $\va \in \mathbb{R}^{d}_{\ge 0}$, any $t>0$, the following concentration inequality holds:
\begin{align}
    \mathbb{P}\left[ \sum_{i=1}^d a_i g_i^2 \ge \|\va\|_1 + 2\|\va\|_2 \sqrt{t} + 2 \|\va\|_{\infty} t \right] \le e^{-t} 
\end{align}
\end{lemma}

\begin{lemma}
\label{lemma: gaussian vector norm concentration}
Let $\vg \sim \calN(\vzero, \sigma^2 \mI_d)$. Then,
\begin{equation}
    \mathbb{P}\left[ \|\vg\|_2^2 \ge 5 \sigma^2 d \right] \le e^{-d}
\end{equation}
\end{lemma}
\begin{proof}
By Lemma \ref{lemma: laurent-massart}, setting $a_i = 1$ for all $i$ and $t = d$ yields
\begin{equation}
    \mathbb{P}\left[ \|\vg\|_2^2 \ge \sigma^2 d + 2 \sigma^2 d + 2 \sigma^2 d \right] \le e^{-d}
\end{equation}
\end{proof}

\begin{lemma} [\cite{pmlr-v162-shen22a}]
\label{lemma: independent gaussian vector inner product concentration}
Let $\vg_1 \sim \calN(\vzero, \sigma_1^2 \mI_d)$ and  $\vg_2 \sim \calN(\vzero, \sigma_2^2 \mI_d)$ be independent. Then, for any $\delta \in (0,1)$ and sufficiently large $d$, there exist constants $c_1, c_2$ such that
\begin{align}
    & \mathbb{P}\left[ \left\vert \langle \vg_1, \vg_2 \rangle \right\vert \le c_1\sigma_1\sigma_2 \sqrt{d\ln(1/\delta)}  \right] \ge 1 - \delta \\
    & \mathbb{P}\left[ \langle \vg_1, \vg_2 \rangle \ge c_2\sigma_1\sigma_2 \sqrt{d} \right] \ge \frac{1}{4}
\end{align}
\end{lemma}

\end{document}